\documentclass{article}

\usepackage[utf8]{inputenc}
\usepackage{url,hyperref,lineno,microtype,subcaption}

\usepackage{lipsum}
\usepackage{authblk}

\usepackage{multirow}

\usepackage[margin=1in]{geometry}

\usepackage{color,xcolor}

\usepackage[ruled]{algorithm2e}
\SetKwInput{KwInput}{Input}
\SetKwInput{KwOutput}{Output}
\usepackage{amssymb, amsmath, amsthm,mathtools}

\newtheorem{theorem}{Theorem}[section]
\newtheorem{proposition}[theorem]{Proposition}
\newtheorem{corollary}[theorem]{Corollary}

\newtheorem{lemma}[theorem]{Lemma}

\newtheorem{assumption}{Assumption}

\theoremstyle{remark}
\newtheorem{remark}{Remark}[section]

\newtheorem{conjecture*}{Conjecture}
\theoremstyle{plain}

\newcommand{\calL}{\mathcal{L}}
\newcommand{\calX}{\mathcal{X}}

\newcommand{\calZ}{\mathcal{Z}}
\newcommand{\calF}{\mathcal{F}}
\newcommand{\calN}{\mathcal{N}}

\newcommand{\E}{\mathbb{E}}
\newcommand{\R}{\mathbb{R}}

\usepackage{xcolor}

\begin{document}

\title{
Neural Stein critics with staged $L^2$-regularization}
\author[1]{Matthew Repasky}
\author[2]{Xiuyuan Cheng\thanks{Email: xiuyuan.cheng@duke.edu}}
\author[1]{Yao Xie}
\affil[1]{
\small
H. Milton Stewart School of Industrial and Systems Engineering, Georgia Institute of Technology}
\affil[2]{
\small
Department of Mathematics, Duke University}
\date{\vspace{-20pt}}

\maketitle

\begin{abstract}
Learning to differentiate model distributions from observed data is a fundamental problem in statistics and machine learning, and high-dimensional data remains a challenging setting for such problems. Metrics that quantify the disparity in probability distributions, such as the Stein discrepancy, play an important role in high-dimensional statistical testing. In this paper, we investigate the role of $L^2$ regularization in training a neural network Stein critic so as to distinguish between data sampled from an unknown probability distribution and a nominal model distribution. Making a connection to the Neural Tangent Kernel (NTK) theory, we develop a novel staging procedure for the weight of regularization over training time, which leverages the advantages of highly-regularized training at early times. Theoretically, we prove the approximation of the training dynamic by the kernel optimization, namely the ``lazy training'', when the $L^2$ regularization weight is large, and training on $n$ samples converge at a rate of ${O}(n^{-1/2})$ up to a log factor. The result guarantees learning the optimal critic assuming sufficient alignment with the leading eigen-modes of the zero-time NTK. The benefit of the staged $L^2$ regularization is demonstrated on simulated high dimensional data and an application to evaluating generative models of image data.
\end{abstract}

{\smaller{\bf Keywords:} 
 Stein Discrepancy, Goodness-of-fit Test, Neural Tangent Kernel, Lazy Training, Generative Models}

\section{Introduction}

Understanding the discrepancy between probability distributions is a central problem in machine learning and statistics. In training generative models, learning to minimize such a discrepancy can be used to construct a probability density model given observed data, such as in the case of generative adversarial networks (GANs) trained to minimize $f$-divergences \cite{nowozin2016f}, Wasserstein GANs \cite{arjovsky2017wasserstein}, and score matching techniques \cite{hyvarinen2005estimation,song2019generative}. Generally, GANs and other generative models require discriminative critics to distinguish between data and a distribution \cite{goodfellow2014generative}; 
such critics have the ability to {tell the location where the model and data distributions differ as well as the magnitude of departure.}
Recent developments in the training of generative models have facilitated advancements in out-of-distribution detection \cite{grathwohl2019your}, in which such models learn to predict the higher likelihood for in-distribution samples. There exists a wide array of integral probability metrics that quantify distances on probability distributions \cite{sriperumbudur2009integral}, including the Stein discrepancy \cite{gorham2015measuring}.
In particular, the computation of the Stein discrepancy only requires knowledge of the score function (the gradient of the log density) of the model distribution. Thus, it avoids the need to integrate the normalizing constant in high dimensions. This makes the Stein discrepancy useful for evaluating some deep models like energy-based models.

Implicit in the minimization of the discrepancy between a model distribution and observed data is the concept of goodness-of-fit (GoF). In the GoF problem and the closely related two-sample test problem, the goal of the analysis is to approximate and estimate the discrepancy between two probability distributions. Integral probability metrics are widely used for such problems. For example, Reproducing Kernel Hilbert space (RKHS) Maximum Mean Discrepancy (MMD) \cite{gretton2012kernel}, a kernel-based approach, is used for two-sample testing among other testing tasks. Kernels parameterized by deep neural networks have been adopted recently in \cite{liu2020learning} to improve the testing power. For GoF tests, where the task is to detect the departure of an unknown distribution of observed data from a model distribution, methods using the Stein discrepancy metrics have been developed. The Stein discrepancy is also calculated using kernel methods \cite{liu2016kernelized,liu2016stein,chwialkowski2016kernel} and more recently using deep neural network-aided techniques \cite{grathwohl2020learning}. We provide more background information related to the Stein discrepancy and its role in GoF testing in Section \ref{sec:background}.

In machine learning, a wide array of modern generative model architectures exist. Energy-based models (EBMs) are a particularly useful subset of generative models. Such models can be described by an \textit{energy function} which describes a probability density up to a normalizing constant \cite{teh2003energy}. While such models provide flexibility in representing a probability density, the normalizing constant (which requires an integration over the energy function to compute) is required to compute the likelihood of data given the model. The Stein discrepancy provides a metric for evaluating EBMs without knowledge of this normalization constant \cite{grathwohl2020learning}. Another popular class of generative models is flow-based models \cite{dinh2014nice, dinh2016density}. Flow-based modeling approaches, such as RealNVP and Glow \cite{kingma2018glow}, provide reversible and efficient transformations representing complex distributions, yielding simple log-likelihood computation. In Section \ref{sec:eval_generative_models}, we outline our approach for evaluating generative EBM models using neural Stein critics.

In this paper, we introduce a method for learning the Stein discrepancy via a novel staged regularization strategy when training neural network Stein critics. 
We consider the $L^2$-regularization of the neural Stein critic, which has been adopted in past studies of neural network Stein discrepancy \cite{hu2018stein,grathwohl2020learning}.
Another motivation to use $L^2$-regularization is due to the fact that the (population) objective of the $L^2$-regularized Stein discrepancy is equivalent to the mean-squared error between the trained network critic and the optimal one (up to an additive constant); see \eqref{eq:Llambda-is-MSE-const}.
It has also been shown previously that the Stein discrepancy evaluated at the optimal critic under $L^2$ regularization reveals the Fisher divergence \cite{hu2018stein}. 
We analyze the role of the regularization strength parameter in $L^2$ neural Stein methods, emphasizing its impact on neural network optimization, which was overlooked in previous studies.
On the practical side, our study shows the benefit of softening the impact of this regularization over the course of training, yielding critics which fit more quickly at early times, followed by stable convergence with weaker regularization. 
An example is shown in Figure \ref{fig:1d_critic}:
(\textbf{A})-(\textbf{C}) illustrate that the target critic changes in magnitude throughout training as the weight of regularization is decreased, and (\textbf{D})-(\textbf{F}) shows the rough approximation to the optimum at early times followed by more nuanced changes in the later stages of training. 

The proposed staging of $L^2$-regularization is closely connected to
the so-called ``lazy training'' phenomenon of neural networks \cite{chizat2019lazy}, which approximately solves a kernel regression problem corresponding to the Neural Tangent Kernel (NTK) \cite{jacot2018neural} at the early training times.
Theoretically, we prove the kernel learning dynamic of $L^2$ neural Stein methods when the regularization weight is large - the main observation is that the role played by the regularization weight parameter is equivalent to a scaling of the neural network function, which itself leads to kernel learning in the case of strong regularization.
This theoretical result motivates the usage of a large penalty weight in the beginning phase of training before decreasing it to a smaller value later.

For GoF problems, the trained neural Stein critic provides model comparison capabilities that assess the accuracy of a model's approximation of the true distribution, 
allowing for identifying the locations of distribution departure in observed data.
This naturally leads to applications for GoF testing and evaluation of generative models.
In summary, the contributions of the current work are as follows: 
\begin{enumerate}

\item We introduce a new method for training neural Stein critics, which incorporates a staging of the weight of the $L^2$ regularization over the process of mini-batch training. 

\item We prove the NTK kernel learning (lazy-training dynamic) of neural Stein critic training with large $L^2$ regularization weight, {providing a theoretical justification of the benefit of using strong $L^2$ regularization in the early training phase. The analysis reveals a convergence at the rate of $O(n^{-1/2})$ up to a log factor, $n$ being the sample size, when training with finite-sample empirical loss.}

\item The advantage of the proposed neural Stein method is demonstrated in experiments, exhibiting improvements over fixed-regularization neural Stein critics and the kernelized Stein discrepancy on simulated data. 
The neural Stein critic is applied to evaluating generative models of image data.
\end{enumerate}

\subsection{Related works}

The Stein discrepancy has been widely used in various problems in machine learning, including measuring the quality of sampling algorithms \cite{gorham2015measuring},
evaluating generative models by diffusion kernel Stein discrepancy 
which unifies score matching with minimum Stein discrepancy estimators \cite{barp2019minimum},
GoF testing, among others.
For GoF testing, a kernel Stein discrepancy (KSD) approach has allowed for closed-form computation of the discrepancy metric \cite{liu2016kernelized,chwialkowski2016kernel}. 
Similar metrics have been used in the GoF setting, such as the finite set Stein discrepancy (FSSD), which behaves as the KSD but can be computed in linear time \cite{jitkrittum2017linear}. 
Our work leverages neural networks,
which potentially have large expressiveness,
to parameterize the critic function space
and studies the influence of $L^2$ regularization from a training dynamic point of view.

Recent studies have developed alternatives to kernelized approaches to computing the Stein discrepancy. Using neural networks to learn Stein critic functions, \cite{hu2018stein} applied neural Stein in the training of high-quality samplers from un-normalized densities.
The neural network Stein critic has also been applied to the GoF hypothesis test settings to evaluate EBMs \cite{grathwohl2020learning}. 
These methods impose a boundedness constraint on the $L^2$ norm of the functions represented by the neural network using a regularization term added to the training objective.
The optimal critic associated with this method yields a Stein discrepancy equivalent to the Fisher divergence and provides an additional benefit in that the critic can be used as a diagnostic tool to identify regions of poor fit \cite{hu2018stein}. 
The staging scheme we introduce in this work is motivated by the observed connection between large $L^2$ penalty and lazy training. In practice, it yields an improvement upon past techniques used to learn the $L^2$-penalized critic.

Many methods have been developed to train and evaluate generative models without knowledge of the likelihood of a model. Early works used the method of Score Matching, which minimizes the difference in score function between the data and model distributions using a proxy objective \cite{hyvarinen2005estimation}. Methods building on this approach are known as \textit{score-based methods}. Score-based methods include approaches that can estimate the normalizing constant for computation of likelihoods, as in the case of Noise-Contrastive Estimation \cite{gutmann2010noise}, and can conduct score matching using deep networks with robust samplers, as in the case of Noise Conditional Score Networks with Langevin sampling \cite{song2019generative}. Our approach potentially provides a more efficient training scheme to obtain discriminative Stein discrepancy critics, leading to a metric representing the discrepancy between distributions. We also experimentally observe that the trained Stein critic function indicates the differential regions between distributions of high dimensional data.

The Stein divergence can be interpreted as a divergence to measure the discrepancy between two distributions. 
The estimation error analysis of neural network function class-based divergence measure has been studied in \cite{arora2017generalization,zhang2018discrimination,sreekumar2022neural}, where the global optimizer within the neural network function class on the empirical loss is assumed without addressing the optimization guarantee.
To incorporate the analysis of neural network training dynamics,
the current paper utilizes theoretical understandings developed by the NTK theory \cite{jacot2018neural,du2018gradient,arora2019exact}. 
Particularly, we utilize ideas surrounding the ``lazy training" phenomenon of neural networks, suggesting the loss through training decays rapidly with a relatively small change to the parameters of the network model, which results in a kernel regression optimization dynamic \cite{chizat2019lazy}. 
For neural network hypothesis testing, the NTK learning was used in computing neural network MMD for two-sample testing \cite{cheng2021neural}. 
As for when neural network training falls under the lazy training regime, \cite{chizat2019lazy} showed that it depends on a choice of scaling of the network mapping.
In this paper, we adopt a staging of the weight of $L^2$ regularization, which can utilize the NTK kernel learning regime at early periods of training, 
and also, in practice, go beyond kernel learning at later phases.

\subsection{{Notation}}

Notations in this paper are mostly standard, with a few clarifications as follows: 
We use $\partial_{\rm x}$ with subscript x to denote partial derivative, e.g., $\partial_\theta f(x,\theta)$ means $\frac{\partial }{\partial \theta} f(x,\theta)$. 
We use $\E_{x\sim p}$ to denote integral over the measure $p(x) dx$, that is,  $\E_{x \sim p} f(x) = \int_\calX f(x) p(x) dx$. 
We may omit the variable as a short-hand notation of integrals, that is,
we write $\int_\calX f $ for   $\int_\calX f(x) dx $.
The symbol $\cdot$ stands for vector-vector inner-product (also used in the divergence operator $\nabla \cdot$), 
and  $ \circ$ stands for matrix-vector multiplication.

\section{Background }\label{sec:background}

We begin by providing necessary preliminaries of the Stein Discrepancy, $L^2$ Stein critics, and GoF testing.

\subsection{Stein Discrepancy}

The recent works on Stein discrepancy in machine learning are grounded in the theory of Stein's operator, the study of which dates back to statistical literature of the 1970s~\cite{stein1972bound}. 
Let $\calX$ be a domain in $\R^d$, and we consider probability densities on $\calX$. Given a density $q$ on $\calX$,
for a sufficiently regular vector field ${\bf f}: \calX \to \R^d$,
the Stein operator $T_q$ \cite{stein1972bound,anastasiou2021stein} 
applied to ${\bf f}$ is defined as 
\begin{equation}
    T_q {\bf f} (x) := {\bf s}_q(x) \cdot {\bf f}(x) + \nabla \cdot {\bf f}(x),
    \label{eq:stein_operator}
\end{equation}
where ${\bf s}_q$ is the {\it score function} of $q$, defined as 
\begin{equation}\label{eq:def-score-function}
{\bf s}_q := {\nabla q}/{q} = \nabla \log q.
\end{equation}
Note that the divergence of ${\bf f}$, denoted by $\nabla \cdot {\bf f}(x)$, is the same as the trace of the Jacobian of ${\bf f}$ evaluated at $x$, and throughout this paper $\cdot$ stands for vector inner-product.

Given another probability density $p$ on $\calX$
and a bounded function class $\calF$ of sufficient regularity, 
 the Stein discrepancy~\cite{gorham2015measuring}  which measures the difference between $p$ and $q$ is defined as 
\begin{equation}
{\rm SD}_{\calF} (p,q) :=
\underset{{\bf f}\in\mathcal{F}}{ \sup}  \,  {\rm SD} [ {\bf f} ], 
\quad {\rm SD} [ {\bf f} ] :=   \mathbb{E}_{x\sim p}   T_q {\bf f}(x).
    \label{eq:sd_operator}
\end{equation}
In this paper, we call ${\bf f}$ the ``critic'',
and we call  ${\rm SD} [ {\bf f} ] $ the Stein discrepancy evaluated at the critic ${\bf f}$. 
We further consider the class of ${\bf f}$ that satisfy a mild boundary condition, namely $\calF_0(p):=\{ {\bf f},\, p{\bf f}|_{\partial X} = 0\}$. (When $X$ is unbounded, $|_{\partial X}$ is understood as the infinity boundary condition. When $X$ is bounded, sufficient regularity of $\partial X$ is assumed so that the divergence theorem can apply).
When $p = q$, one can verify that for any ${\bf f} \in \calF_0(p)$ and sufficiently regular, $ {\rm SD}[{\bf f}] = 0 $, which is known as the Stein's identity.
The other direction, namely for some $\calF \subset \calF_0(p)$, ${\rm SD}_{\calF} (p,q) = 0 \Rightarrow p=q$, is also established under certain conditions~\cite{stein2004use}.
This explains the name ``critic'' because the ${\bf f}$ which makes ${\rm SD} [ {\bf f} ] $ significantly large can be viewed as a test function (vector field) that indicates the difference between $p$ and $q$.

To proceed, we introduce the following assumption on the densities $p$ and $q$:

\begin{assumption}\label{assump:p-q}
The two densities $p$ and $q$ are supported and non-vanishing on $\calX$ (vanishing at the boundary of $\calX$),
differentiable on $\calX$,
and the score functions ${\bf s}_p$ and ${\bf s}_q$ are in $\calF_0(p)$.
\end{assumption}

When the function class $\mathcal{F}$ is set to be the unit ball in an RKHS, the definition \eqref{eq:sd_operator} yields the kernelized Stein discrepancy (KSD) \cite{liu2016kernelized, chwialkowski2016kernel}. A possible limitation of the KSD approach is the sampling complexity and computational scalability in high dimension. 
In this work, we consider critics of regularized $L^2$ norm to be parametrized by neural networks as previously studied in \cite{hu2018stein,grathwohl2020learning},
which have potential advantages in model expressiveness and computation.

\subsection{ $L^2$ Stein critic }

This paper focuses on 
when the critic ${\bf f}$ in Stein discrepancy \eqref{eq:sd_operator} is at least squared integrable on $(\calX, p(x) dx)$. 
Consider the critic function ${\bf f}: \calX \to \R^d$, and the $d$ coordinates of ${\bf f}$ are denoted as 
${\bf f}(x) = ( f_1(x), \cdots, f_d(x) )^T$.
We first introduce notations of the $L^2$ space of vector fields. Define the inner-product and $L^2$-norm of vector fields on $(\calX, p(x) dx)$ as the following: for $ {\bf v}, {\bf w}: \calX \to \R^d$, let
\begin{equation}
\langle {\bf v}, {\bf w} \rangle_p 
:= \int_{\calX} {\bf v}(x) \cdot {\bf w}(x)  p(x) dx,
\end{equation}
and then the $L^2$ norm is defined as
\begin{equation}
\|  {\bf v} \|_{p}^2 := \langle {\bf v}, {\bf v} \rangle_p.
\end{equation}
{We stress that, throughout this paper, subscript $_p$ in the norm means 2-norm weighted by the measure $p(x)dx$.}
We denote  all critics ${\bf f}: \calX \to \R^d$ such that $ \|  {\bf f} \|_{p }^2  < \infty$ the space of $L^2(p): = L^2(\calX, p(x) dx)$. 

The Stein discrepancy over the class of $L^2$ critics with bounded $L^2$ norm is defined, for some $r  > 0$, as
\begin{equation}\label{eq:def-SD-alpha}
{\rm SD}_{r} (p,q)=  \underset{{\bf f} \in \calF_0(p), \, \| {\bf f} \|_{p} \le r }{\sup}   {\rm SD} [ {\bf f} ],
\end{equation}
where ${\rm SD}[ {\bf f} ]$ is defined as in \eqref{eq:sd_operator}. Define
\begin{equation}\label{eq:def-fstar}
    {\bf f}^*  := {\bf s}_q - {\bf s}_p,
\end{equation}
the following lemma characterizes the solution of \eqref{eq:def-SD-alpha}, and the proof (left to Appendix \ref{app:add-proofs}) also derives a useful equality when ${\bf f} $ and $ {\bf f}^*$ are in $L^2(p) \cap \calF_0(p)$, see also e.g. \cite[Lemma 2.3]{liu2016kernelized}:
\begin{equation}\label{eq:SD-as-innerproduct}
{\rm SD} [ {\bf f} ]  
= \mathbb{E}_{x\sim p}  T_q {\bf f}(x) 
= \langle {\bf f}^*, {\bf f} \rangle_p.
\end{equation}

\begin{lemma}\label{lemma:sol-SD-L2}
For any $r> 0$, 
suppose $ {\bf s}_q$ and  ${\bf s}_p$ are in $L^2(p)$ {and $\calF_0(p)$}, 
then 
${\rm SD}_{r} (p,q) = r \|   {\bf f}^* \|_{p } < \infty$
and,
if  $ \|   {\bf f}^* \|_{p } > 0$, 
 the supremum of \eqref{eq:def-SD-alpha} is achieved at 
$ {\bf f}  =  ( { r }/{  \|   {\bf f}^* \|_{p } }) {\bf f}^* $.  
\end{lemma}

Lemma \ref{lemma:sol-SD-L2} implies that if $ \|   {\bf f}^* \|_{p } = 0$, then ${\rm SD}_{r} (p,q) = 0$.
Thus, we take the following assumption, which restricts the case where the $L^2$ Stein critic can achieve a positive discrepancy 
 when $q \neq p$.

\begin{assumption}\label{assump:p-q-L2}
The score functions ${\bf s}_p $ and ${\bf s}_q$ are in $L^2(p)$,
and when $q\neq p$, 
$ \| {\bf s}_q - {\bf s}_p  \|_{p } >0 $.
\end{assumption}

\subsection{Goodness-of-Fit tests}\label{sec:gof_background}

In GoF testing, we are presented with samples $X=\{x_i\}_i$ drawn from an unknown distribution $p$, and we wish to assess whether this sample is likely to have come from the {\it model distribution} $q$. That is, we may define the null and alternative hypotheses as follows:
 \begin{equation}
    H_0: p=q,
    \quad \quad
    H_1: p\neq q.
    \label{eq:hypotheses}
\end{equation}
A GoF test is conducted using a test statistic $\hat{T} = \hat{T}(X)$ 
 computed using observed samples in $X$.
 {After specifying a number $t_{\rm thresh}$, which is called the ``test threshold'', 
 the null hypothesis is rejected if $\hat{T}> t_{\rm thresh}$.}
There are different approaches to specifying a test threshold, 
for example, if prior knowledge of the distribution of $\hat{T}$ under $H_0$ is available, it can be used to choose $t_{\rm thresh}$.
In the experiments of this work, we compute $t_{\rm thresh}$ using a bootstrap strategy from samples from the model distribution $q$.

The selection of $t_{\rm thresh}$ is to control the Type-I error, defined as $\Pr[\hat{T}> t_{\rm thresh} ]$ under $H_0$.
{The randomness of $\Pr$ is with respect to data $X$.}
The goal is to guarantee that $\Pr[\hat{T}> t_{\rm thresh} ] \le \alpha$, which is called the the ``significance level'' (typically, $\alpha = 0.05$).
The Type-II error measures the probability that the null hypothesis is improperly accepted as true,
that is,  $\Pr[\hat{T}\leq t_{\rm thresh}]$ under $H_1$.
Finally, the ``test power'' is defined as one minus the Type-II error.
For the application to GoF testing, 
the current work develops a test statistic computed using a Stein critic parameterized and learned by a neural network, computed from a training-testing split of the dataset. More details will be introduced in Section \ref{sec:gof_test}.

\section{Method}

\subsection{Neural $L^2$ Stein critic}\label{sec:learning_stein_critic}

Replacing the $L^2$-norm constraint  in \eqref{eq:def-SD-alpha}  to be a regularization term leads to the following minimization over a certain class of ${\bf f}$ (inside $L^2(p) \cap \calF_0(p)$) as
\begin{equation}\label{eq:def-L_lambda}
    \calL_\lambda[ {\bf f}] 
    := -  {\rm SD} [ {\bf f} ] 
    + \frac{\lambda}{2} \| {\bf f} \|_{p }^2
    = \E_{x\sim p} \left(   - T_q {\bf f}(x)  + \frac{\lambda}{2}  \|{\bf f}(x)\|^2 \right),
\end{equation}
where  $\lambda > 0$ is the penalty weight of the $L^2$ regularization. 
Under Assumption \ref{assump:p-q}-\ref{assump:p-q-L2}, 
%and supposing the score functions are in $L^2(p)$, 
the equality \eqref{eq:SD-as-innerproduct} gives that 
\begin{align}
  \calL_\lambda [ {\bf f}] 
  &  = - \langle {\bf f}^*, {\bf f} \rangle_p + \frac{\lambda}{2} \langle {\bf f}, {\bf f} \rangle_p 
 =  \frac{1}{2 \lambda }  ( \|  \lambda{\bf f} - {\bf f}^*   \|_p^2 -  \| {\bf f}^* \|_p^2 ).
\label{eq:Llambda-is-MSE-const}
\end{align}
Since $ \| {\bf f}^* \|_p^2$ is a constant independent of ${\bf f} $,
\eqref{eq:Llambda-is-MSE-const} immediately gives that  
$    \calL_\lambda[ {\bf f}] $ is minimized at
\begin{equation}
    {\bf f}_\lambda^* := \frac{ {\bf f}^* }{\lambda} = \frac{1}{\lambda} ({\bf s}_q - {\bf s}_p),
    \label{eq:optimal_critic}
\end{equation}
see also Theorem 4.1 of \cite{hu2018stein}.
The expression \eqref{eq:optimal_critic}  reveals an apparent issue: if one only considers  global functional minimization, 
then the choice of $\lambda$ plays no role but contributes a scalar normalization of the optimal critic,
and consequently does not affect the learned Stein critic in practice, e.g., in computing test statistics.
A central observation of the current paper is that the choice of $\lambda$ plays a role in optimization, 
specifically, in training neural network parameterized Stein critics.

Following \cite{hu2018stein,grathwohl2020learning}, we parameterize the critic by a neural network mapping ${\bf f}(x, \theta)$ parameterized by $\theta$,
and we assume that ${\bf f}(\cdot, \theta) \in L^2(p)$ for any $\theta$ being considered. We denote this ${\bf f}(\cdot,\theta)$ the ``neural Stein critic''.
The population loss of $\theta$ follows from \eqref{eq:def-L_lambda} as
\begin{equation}\label{eq:def-L_lambda-theta}
 {L}_\lambda(\theta) = 
    \calL_\lambda[ {\bf f}(\cdot, \theta)].
    \end{equation}
The neural Stein critic is trained by minimizing the empirical version of $ {L}_\lambda(\theta)$
computed from finite i.i.d. samples $x_i\sim p$, i.e., given $n_{\rm tr}$ training samples, the empirical loss is defined as
\[
    \hat{L}_\lambda(\theta) 
    : =  \frac{1}{ n_{\rm tr}} \sum_{i=1}^{n_{\rm tr}}    
    ( -T_q{\bf f}(x_i, \theta) + \frac{\lambda}{2}  \|{\bf f}(x_i, \theta)\|^2).
\]
\begin{remark}[{$\lambda$ as scaling parameter}]
We note a ``$\lambda$-scaling'' view of the $L^2$-regularized objective $ {L}_\lambda(\theta)$,
namely, it is equivalent to a scaling of the parameterized function ${\bf f}(x,\theta)$ by $\lambda$.
Specifically, consider $ \calL_1 [ {\bf f}] $ as the ``scaling-free'' ($\lambda$-agnostic) functional minimizing objective,  
then by definition, we have that 
\begin{equation}\label{eq:lambda-is-scaling}
\calL_\lambda [ {\bf f}] 
  =  \frac{1}{\lambda} \calL_1 [  \lambda {\bf f}].
  \end{equation}
This $\lambda$-scaling view of the $L^2$ Stein discrepancy allows us to connect to the 
lazy training of neural networks \cite{chizat2019lazy}, which indicates that,
with large values of $\lambda$, the early stage of the training of the network approximates the optimization of a kernel solution. 
We detail the theoretical proof of such an approximation in Section \ref{sec:lazy_training}. 
This view motivates our proposed usage of large $\lambda$ in early training,
which will be detailed in Section \ref{subsec:staged}.
\end{remark}

We close the current subsection by introducing a few notations 
to clarify the $\lambda$-scaling indicated by  \eqref{eq:optimal_critic}.
Suppose the neural Stein critic ${\bf f}(\cdot, \theta)$ trained by minimizing  $ \hat{L}_\lambda(\theta) $
approximates the minimizer of $ {\calL}_\lambda$,
  namely {the ``optimal critic''} ${\bf f}_\lambda^*$, 
then both ${\bf f}(\cdot, \theta)$ and the value of ${\rm SD}[  {\bf f}(\cdot, \theta)] $ would scale like $1/\lambda$.
We call ${\bf f}^*$ defined in  \eqref{eq:def-fstar} the ``scaleless optimal critic function''.
The expression \eqref{eq:optimal_critic} also suggests that, if the neural Stein critic successfully approximates the optimal, we would expect
\[
\lambda {\bf f}(\cdot, \theta) \approx {\bf f}^*,
\]
which will be confirmed by the theory in Section \ref{sec:lazy_training}. 
We thus call $\lambda {\bf f}(\cdot, \theta)$ the  ``scaleless neural Stein critic''.

\begin{figure}[t]
    \centering
    \includegraphics[width=\textwidth]{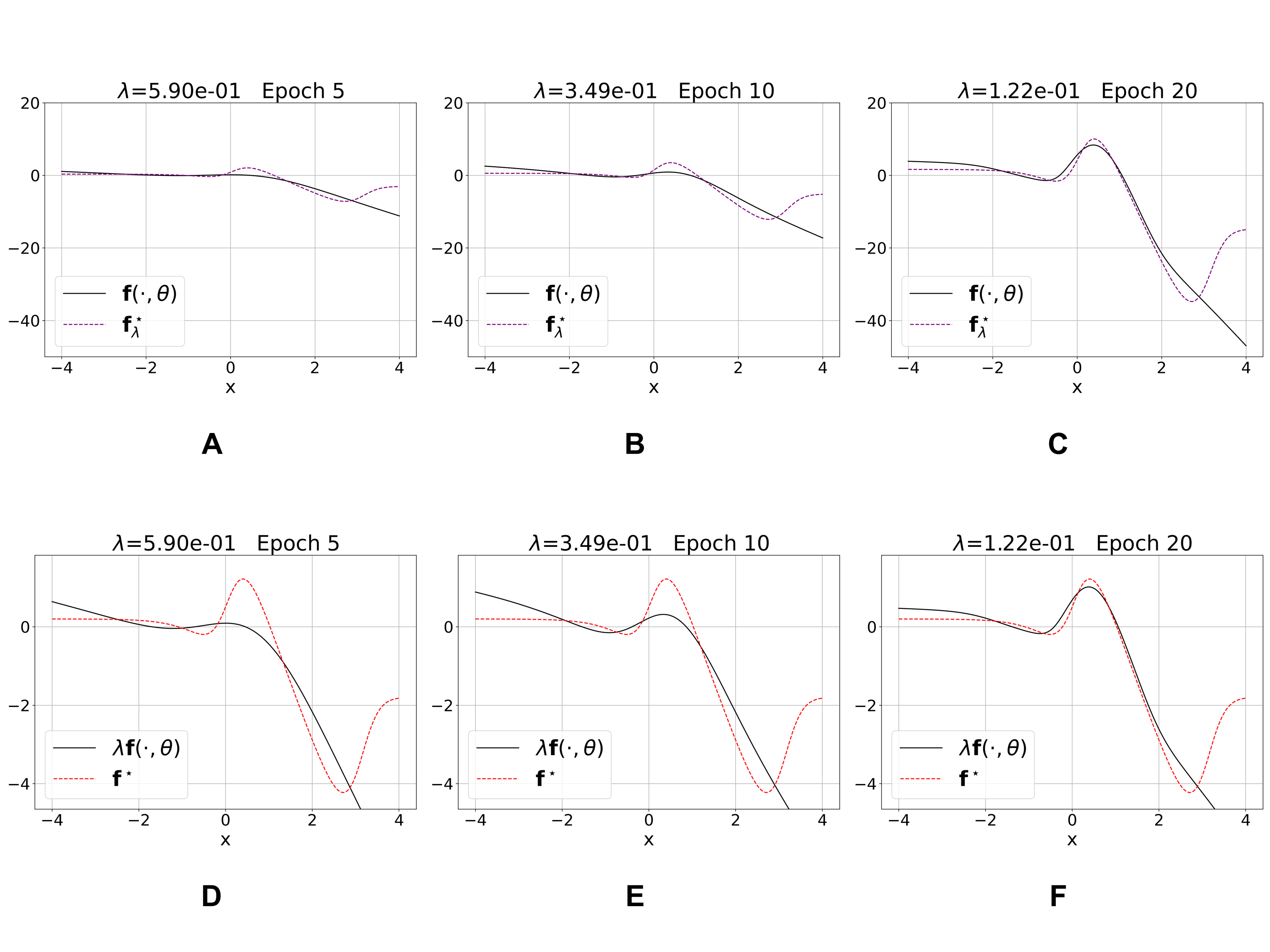}
        \vspace{-25pt}
    \caption{    
    Visualization of staged $L^2$ regularization throughout training.
    {The two distributions $p$ and $q$ are Gaussian mixtures in 1D, and the scaleless optimal critic ${\bf f}^*$ has explicit expression \eqref{eq:analyic_f_star}, see the details of the distributions and optimization in Appendix~\ref{sec:1d_experiment}.}
    In (\textbf{A})-(\textbf{C}), the {trained} neural Stein critic ${\bf f}(\cdot,\theta)$ is compared to the optimal critic ${\bf f}^*_{\lambda}$ defined in \eqref{eq:optimal_critic}. The {optimal} critic function changes scale through training as the weight of regularization is decreased. In (\textbf{D})-(\textbf{F}), the scaleless neural Stein critic {$\lambda{\bf f}(\cdot,\theta)$} is compared to the scaleless optimal critic ${\bf f}^*$ defined in \eqref{eq:def-fstar}. The {neural Stein} critic roughly approximates the {optimal critic} at early times, followed by fine-tuned changes in the later epochs of training.
    }\label{fig:1d_critic}
\end{figure}

\subsection{Staged-$\lambda$ regularization in training}\label{subsec:staged}

In practice, the neural network is optimized via the $\lambda$-dependent loss $L_\lambda$. 
We propose a staging of the regularization weight $\lambda$ that begins with a large value followed by a gradual decrease until some terminal value.
Specifically, we adopt a log-linear staged regularization scheme by which we decrease $\lambda$ via the application of a multiplicative factor $\beta < 1$ after  
a certain interval of training time measured by the number of batches $B_w$. 
In addition to $B_w$, this staging scheme uses three parameters: the initial weight $\lambda_{\rm init}$, the decay factor $\beta \in (0,1)$, and the terminal weight $\lambda_{\rm term}$ (at which point the decay ceases). 
We denote the discrete-time regularization staging with these parameters as 
$\Lambda(B_i ; \lambda_{\rm init}, \lambda_{\rm term}, \beta)$, where $B_i = i\cdot B_w$ stands for the $i$-th interval of batches, $i\in\mathbb{N}$, and $\lambda$ will be set to the value on the $i$-th interval. 
$i=0$ refers to the period before training has begun such that $\Lambda(B_0)=\lambda_{\rm init}$ and, for $i>0$,
\begin{equation}
    \Lambda\left(B_i ; \lambda_{\rm init}, \lambda_{\rm term}, \beta\right) = \text{max}\left\{\lambda_{\rm init}\cdot\beta^{i}, \lambda_{\rm term}\right\}.
    \label{eq:lam_staging}
\end{equation}
That is, when $i$ increments of $B_w$ number of batches have occurred, $\lambda_{\rm init}$ is annealed by a factor of $\beta^i$.
$\beta $ can be any positive number in $(0,1)$ and we typically choose $\beta $ to be about 0.7$\sim$ 0.9.

The beginning of the training with a large $\lambda$ is motivated by  the analysis in Section~\ref{sec:lazy_training},
which suggests that with large $\lambda$ at early times of training, 
the training of the neural Stein critic can be approximately understood from the perspective of kernel regression optimization, rapidly reaching its best approximation in $\sim  1/\lambda $ time (up to a log factor), see Theorem~\ref{thm:NTK-stein} (and Theorem \ref{thm:NTK-stein-finite-sample} for the analysis with finite-sample loss).
In many cases, the kernel learning solution may not be sufficient to approximate the optimal critic ${\bf f}^*$, and this calls for going beyond the NTK kernel learning in training neural networks, which remains a theoretically challenging problem. 
On the other hand, the proper regularization strength $\lambda$, depending on the problem (the distributions $p$ and $q$ and the sample size), may be much smaller than the initial large $\lambda$, which we set to be $\lambda_{\rm term}$. 
We empirically observed that gradually reducing the $\lambda$ in later training phases will allow the neural Stein critic to progressively fit to ${\bf f}^*$, as illustrated in Figure \ref{fig:1d_critic}. Meanwhile, we also observed that using large $\lambda$ at the beginning phase of training and then gradually tuning down to small $\lambda$ performs better than  using small and fixed $\lambda$ throughout, see Section \ref{sec:exp}. These empirical results suggest the benefit of fully exploiting the kernel learning (by using large $\lambda$) in the early training phase
and annealing to small $\lambda$ in the later phase of training.

To further study other possible staging methods, we investigated an adaptive scheme of annealing $\lambda$ over training time by monitoring the validation error, and the details are provided in Appendix \ref{app:adaptive_staging}. On simulated high dimensional Gaussian mixture data, 
the adaptive staging gives performance similar to the scheme \eqref{eq:lam_staging}, and the resulting trajectory of $\lambda$ also resembles the exponential decay as in \eqref{eq:lam_staging}, {see Figure~\ref{fig:adaptive_staging}.}
Thus we think the scheme  \eqref{eq:lam_staging} can be a typical heuristic choice of the staging, and other choices of annealing $\lambda$ are also possible.

\subsection{{Evaluation and validation metrics}}

We introduce the mean-squared error metric to quantitatively evaluate how the learned neural Stein critic ${\bf f}(x, \theta)$ approximates the optimal critic.
In our experiments on synthetic data, we can use the knowledge of ${\bf f}^*$ (which requires knowledge of ${\bf s}_p$) and compute an estimator of 
${\rm MSE}_q: =  \mathbb{E}_{x\sim q} \left\| \lambda {\bf f}(x, \theta) - {\bf f}^*(x)\right\|^2$
by the sample average on a set of $n_{\rm te}$ test samples $x_i\sim q$, that is,
\begin{equation}
    \widehat{\rm MSE}_q[{\bf f}  (\cdot, \theta) ] 
    = \frac{1}{n_{\rm te}} \sum_{i=1}^{n_{\rm te}} \left\| \lambda {\bf f} (x_i, \theta) - {\bf f}^*(x_i)\right\|^2.
\label{eq:mse}
\end{equation}
Similarly, if additional testing or validation samples from data distribution $p$ are available, we can compute an estimator of 
${\rm MSE}_p: =  \mathbb{E}_{x\sim p} \left\| \lambda {\bf f}(x, \theta) - {\bf f}^*(x)\right\|^2$ by the sample average denoted as $ \widehat{\rm MSE}_p$.

% explain what is monitor
We note that the relationship between ${\rm MSE}_p$  and $\calL_\lambda[ {\bf f} (\cdot, \theta) ]$ as in \eqref{eq:Llambda-is-MSE-const} leads to an estimator of ${\rm MSE}_p$ (up to a constant) that can be computed without the knowledge of ${\bf f}^*$ as follows.
Define 
${\rm MSE}^{(m)}_p[{\bf f}  (\cdot, \theta)  ] 
   : = 2 \lambda  \calL_\lambda[ {\bf f} (\cdot, \theta) ]
   = 2\lambda \mathbb{E}_{x\sim p} \left( -T_q {\bf f}(x, \theta) + \frac{\lambda}{2}\|{\bf f}(x, \theta)\|^2 \right)$,
and \eqref{eq:Llambda-is-MSE-const} gives that 
 $2 \lambda  \calL_\lambda[ {\bf f} (\cdot, \theta) ]  =  {\rm MSE}_p[ {\bf f} (\cdot, \theta) ]  - \| {\bf f}^* \|_p^2 $.
 Thus, we can use the sample-average estimator of ${\rm MSE}^{(m)}_p$ to estimate ${\rm MSE}_p$ up to the unknown constant $ \| {\bf f}^* \|_p^2$. That is, given $n_{\rm val}$ validation  samples $x_i \sim p$,
 the estimator can be computed as 
\begin{equation}
    \widehat{\rm MSE}_p^{(m)}[{\bf f}  (\cdot, \theta)  ] 
   : = \frac{2\lambda}{n_{\rm val}} \sum_{i=1}^{n_{\rm val}} \left( -T_q {\bf f}(x_i, \theta) + \frac{\lambda}{2}\|{\bf f}(x_i, \theta)\|^2 \right).
\label{eq:validation_mse}
\end{equation}
The superscript $^{(m)}$ stands for ``monitor'', and in experiments we will use $\widehat{\rm MSE}_p^{(m)}$ to monitor the training progress of neural Stein critic by a stand-alone validation set.

\section{Theory of lazy training} \label{sec:lazy_training}

In this section, we show the theoretical connection between choosing large $L^2$-regularization weight $\lambda$ at the beginning of training and ``lazy learning'' - referring to when the training dynamic resembles a kernel learning determined by the model initialization \cite{chizat2019lazy}.
The main result, 
Theorem \ref{thm:NTK-stein} (for population loss training)
and Theorem \ref{thm:NTK-stein-finite-sample} (for empirical loss training),
prove the kernel learning of neural Stein critic at large $\lambda$
and theoretically support the benefit of using large $\lambda$ in the beginning phase of training.

To recall the setup, we assume the score function of model distribution ${\bf s}_q$ is accessible, and the score of data distribution $ {\bf s}_p $ is unknown. We aim to train a neural network critic to infer the unknown {scaleless} optimal critic ${ \bf f}^*$  from data.
All the proofs are in Section \ref{sec:proofs} {and Appendix \ref{app:add-proofs}}.

\subsection{Evolution of network critic under gradient descent}\label{subsec:theory-GD}

Throughout this section, we derive the continuous-time optimization dynamics of gradient descent (GD).
The continuous-time GD dynamic reveals the small learning rate limit of the discrete-time GD, and the analysis may be extended to 
minibatch-based Stochastic Gradient Descent (SGD) method \cite{goodfellow2016deep} used in practice.
We start from training with the population loss, and the finite-sample analysis will be given in Section \ref{subsec:NTK-finite-sample}.

Consider a neural Stein critic ${\bf f}(\cdot, \theta)$ parameterized by $\theta$ which maps from $\calX\subset\mathbb{R}^d$ to $\mathbb{R}^d$.
 We assume $\theta \in \Theta$, which is some bounded set in $\R^{M_\Theta}$, where $M_\Theta$ is the total number of trainable parameters. 
The notation $\langle \cdot ,  \cdot \rangle_\Theta$ stands for the inner-product in $\R^{M_\Theta}$,
and $\| \cdot \|_\Theta$ the Euclidean norm. 
The subscript $_\Theta$ may be omitted when there is no confusion.

\begin{assumption}\label{assump:f(x,theta)-L2-F0}
The network function $ {\bf f}(x, \theta)$ is differentiable on $\calX \times \Theta$,
and for any $\theta \in \Theta$, ${\bf f}(\cdot, \theta) \in L^2(p) \cap \calF_0(p)$.
\end{assumption}

The boundedness of $\Theta$ is valid in most applications, and our theory will restrict $\theta$ inside a Euclidean ball in $\R^{M_\Theta}$, see more in Assumption \ref{assump:C1-C2}.
 Recall that for regularization weight $\lambda > 0$, the training objective $L_\lambda(\theta) $ is defined as in \eqref{eq:def-L_lambda-theta}.
Suppose the neural network parameter $\theta$ evolving over training time is denoted as $\theta(t)$ for $t > 0$. 
The GD dynamic is defined by {the ordinary differential equation}
\begin{equation}\label{eq:def-GD-theta}
\dot{\theta}(t) = -  {\partial_\theta} L_\lambda( \theta (t )),
\end{equation}
starting from some initial value of $\theta(0)$.
The following lemma gives the expression of \eqref{eq:def-GD-theta}.

\begin{lemma}\label{lemma:GD-theta-eqn}
For $\lambda > 0$, the GD dynamic of $\theta(t)$ of minimizing $L_\lambda(\theta)$ can be written as
\begin{equation}\label{eq:GD-theta-eqn}
\dot{\theta}(t) =  - \E_{x \sim p} \partial_\theta {\bf f}(x,\theta (t)) \cdot 
				\big(  \lambda  {\bf f}( x ,\theta(t)) - {\bf f}^*(x)  \big).
\end{equation}
\end{lemma}

Next, we derive the evolution of the network critic over time. We start by defining
\begin{equation}\label{eq:def-u(x,t)}
{\bf u}( x,t ) := {\bf f}(x , \theta(t)),
\end{equation}
and by chain rule, we have that  
\begin{equation}\label{eq:dudt-1}
 {\partial_t}{\bf u }( x, t) 
 = \langle \partial_\theta {\bf f}( x, \theta(t) ), \dot{\theta}(t) \rangle_\Theta.
\end{equation}
Combining Lemma \ref{lemma:GD-theta-eqn} and \eqref{eq:dudt-1} leads to the evolution equation of ${\bf u}(x,t)$ to be derived in the next lemma.
To proceed, we introduce the definition of the finite-time (matrix) neural tangent kernel (NTK) ${\bf K}_t(x,x') $ as
\begin{equation}\label{eq:def-NTK-t}
[ {\bf K}_t(x,x') ]_{ij } 
= \langle \partial_\theta  f_i( x, \theta(t) ), 
 \partial_\theta  f_j(x',\theta (t)) \rangle_\Theta, \quad i,j =1, \cdots, d,
\end{equation}
where $f_i$ denotes the $i$-th coordinate of ${\bf f}$. With the notation of ${\bf K}_t(x,x') $ we have the following lemma.

\begin{lemma}\label{lemma:u(x,t)-dynamic}
The dynamic of ${\bf u }( x, t) $ follows that 
\begin{equation}\label{eq:evolution-u}
{\partial_t}{\bf u }( x, t) 
 =  -
\E_{x' \sim p}
 {\bf K}_t(x,x')   \circ \left( \lambda {\bf u}( x' ,t) - {\bf f}^*(x') \right).
\end{equation}
\end{lemma}
The evolution equation \eqref{eq:evolution-u} is exact for ${\bf u }( x, t) $, but the kernel $ {\bf K}_t(x,x') $ changes with $t$ as the training progresses. Next in Section \ref{subsec:theory-Stein-NTK}, we will replace ${\bf K}_t(x,x') $ with the zero-time kernel throughout time in the evolution equation, resulting in the more analyzable kernel learning solution.

\subsection{Kernel learning with zero-time NTK}\label{subsec:theory-Stein-NTK}

Theoretical studies of NTK to analyze neural network learning
\cite{jacot2018neural,du2018gradient,arora2019exact}
typically use two approximations:
(i) zero-time approximation, namely
${\bf K}_t(x,x') \approx {\bf K}_0(x,x')$
where both ${\bf K}_t$ and ${\bf K}_0$ are finite-width NTK, 
and
(ii) infinite-width approximation, namely 
at initialization,
${\bf K}_0(x,x') \approx {\bf K}_0^{(\infty)}(x,x')$
as the widths of hidden layers increase,
where ${\bf K}_0^{(\infty)}$ is the limiting kernel at infinite width.
As has been pointed out in \cite{chizat2019lazy}, 
the reduction to kernel learning in ``lazy training'' does not necessarily require model over-parameterization
- corresponding to large width (number of neurons) in a neural network - but can be a consequence of a scaling of the network function. Here we show the same phenomenon where the scaling factor is the $L^2$ regularization parameter $\lambda$, see \eqref{eq:lambda-is-scaling}.
That is, we show the approximation (i) only and prove lazy training for finite-width neural networks.

We derive the property of the zero-time NTK kernel learning in this subsection,
and prove the validity of the approximation (i) in the next subsection.
To proceed, consider the kernel ${\bf K}_t(x,x') $ defined in \eqref{eq:def-NTK-t} at  time zero, which can be written as
\begin{equation}\label{eq:def-NTK-0}
\left[ {\bf K}_0(x,x') \right]_{ij } 
= \langle \partial_\theta  f_i( x, \theta(0) ), 
 \partial_\theta  f_j(x',\theta (0)) \rangle_\Theta, \quad i,j =1, \cdots, d.
\end{equation}
The kernel ${\bf K}_0(x,x') $ only depends on the initial network weights $\theta(0)$, which is usually random and independent from the data samples. 
We call ${\bf K}_0(x,x') $ the zero-time finite-width (matrix) NTK. 
Following the NTK analysis of neural network training, 
we will show in Section \ref{subsec:lazy-training-approx} that the evolution dynamic of the network function ${\bf u}(x, t)$ 
can be approximated by that of a kernel regression optimization 
- the so-called ``lazy-training'' dynamic - 
which is expressed by replacing the finite-time NTK with the zero-time NTK. 
For the dynamic in \eqref{eq:evolution-u}, the lazy-training dynamic counterpart is defined by the evolution of another solution $\bar{\bf u}(x,t)$ by replacing the kernel ${\bf K}_t$ with ${\bf K}_0$ in  \eqref{eq:evolution-u}
starting from the same initial value, namely, $\bar{\bf u }( x, 0) = {\bf u }( x, 0)$ and 
 \begin{equation}\label{eq:evolution-baru}
 {\partial_t} \bar{\bf u }( x, t) 
 =  - \E_{x' \sim p}
 {\bf K}_0(x,x')   \circ 	\left( \lambda \bar{\bf u}( x' ,t) - {\bf f}^*(x') \right).
\end{equation}
For simplicity, one assumes that at initialization, the network function is zero mapping, that is,
both  ${\bf u}( x ,0)$ and $\bar{\bf u}( x ,0)$ are zero. 
The argument generalizes to when the initial network function is small in magnitude, cf. Remark \ref{rk:small-u0},
which can be obtained, e.g., by initializing neural network weights with small values.

To analyze the dynamic of \eqref{eq:evolution-baru}, we introduce the eigen-decomposition of the kernel in the next lemma.

\begin{lemma}\label{lemma:eigen}
Suppose $\| \partial_\theta f_i(x, \theta(0)) \|_\Theta$ for $i=1,\cdots, d$ are squared integrable on $(\calX, p(x) dx)$.
The kernel $ {\bf K}_0(x,x')$ on $(\calX, p(x) dx)$ has  a finite collection of $M$ eigen-functions ${\bf v}_k: \calX \to \R^d$, 
$k=1,2, \cdots$, associated with positive eigenvalues, 
in the sense that 
\begin{equation}
\int_{\calX} {\bf K}_0(x,x') \circ {\bf v}_k(x') p(x') dx' = \mu_k {\bf v}_k(x),  
\end{equation}
where $\mu_1 \ge \cdots \ge \mu_M  >0$.
The eigen-functions are ortho-normal in $L^2(p)$, namely, 
$\langle {\bf v}_k, {\bf v}_l \rangle_p  =\delta_{kl}$, and for any ${\bf v}$ orthogonal to ${ span}\{  {\bf v}_1, \cdots,{\bf v}_M  \}$,
$\int_{\calX} {\bf K}_0(x,x') \circ {\bf v}(x') p(x') dx'  = 0$. 
\end{lemma}

The square integrability of $\partial_\theta f_i(x, \theta(0))$ can be guaranteed by certain boundedness condition on $\partial_\theta f$, 
see more in Assumption \ref{assump:C1-C2} below. 
The finite rank of the kernel, as shown in the proof, is due to the fact that we use a neural network of finite width. 
Below, we will assume that the optimal critic ${\bf f}^*$ can be efficiently expressed by the span of finite many leading eigen-functions ${\bf v}_k$.
To show that the kernel spectrum is expressive enough to approximate an  ${\bf f}^*$,
one may combine our analysis with  NTK approximation (ii):
the expressiveness of the limiting kernel ${\bf K}_0^{(\infty)}$ at infinite width can be theoretically characterized in certain settings,
meaning that its eigen-functions ${\bf v}^{(\infty)}_k$ collectively can span a rich functional space.
By the approximation ${\bf K}_0  \approx {\bf K}_0^{(\infty)}$ in spectrum, the expressiveness of the span of ${\bf v}_k$ can also be shown.
Such an extension is postponed here.

By Lemma \ref{lemma:eigen}, the eigen-functions $\{  {\bf v}_1, \cdots,{\bf v}_M  \}$ form an ortho-normal set with respect to the inner-product $\langle \cdot, \cdot \rangle_p$.
Thus for any integer $m \le M$,  the scaleless optimal critic ${\bf f}^* \in L^2(p)$ (by Assumption \ref{assump:p-q-L2}) can be orthogonally decomposed into two parts ${\bf f}_1^*$ and ${\bf f}_2^*$ such that 
\[
{\bf f}^* = {\bf f}_1^* + {\bf f}_2^*, \quad
{\bf f}_1^* \in \text{span}\{ {\bf v}_1, \cdots, {\bf v}_m \}, \quad {\bf f}_2^* \in \text{span}\{ {\bf v}_1, \cdots, {\bf v}_m \}^\perp.
\]
Making use of the orthogonal decomposition and by assuming that ${\bf f}^*$ has a significant projection on the eigen-space of the kernel $ {\bf K}_0(x,x')$ associated with large eigenvalues,
the  following proposition derives the optimization guarantee of $\bar{\bf u}(x,t)$.

\begin{proposition}\label{prop:NTK-stein}
Under the condition of Lemma \ref{lemma:eigen} and notations as therein,
suppose for $\delta >0$ and some integer $m \le M$,  $\mu_1 \ge \cdots \ge \mu_m \ge \delta >0$.
Let ${\bf f}^* = {\bf f}_1^* + {\bf f}_2^*$ be the  orthogonal decomposition as above.
Then, for $\lambda > 0$, starting from $\bar{\bf u}( x ,0) = 0$, for all $t > 0$,
\begin{equation}\label{eq:bound-NTK}
\| \lambda \bar{\bf u}(  \cdot ,t)  - {\bf f}^* \|_{p }  \le   e^{-t \lambda \delta } \| {\bf f}_1^* \|_{p }  +  \| {\bf f}_2^* \|_{p } .
\end{equation}
In particular, if for some $ 0 < \epsilon < 1$,
$\| {\bf f}_2^* \|_{p }  \le \epsilon \| {\bf f}^* \|_{p }  $, 
then  we have
\begin{equation}\label{eq:2eps-bound}
\| \lambda \bar{\bf u}(  \cdot ,t)  - {\bf f}^* \|_{p }  \le 2 \epsilon  \| {\bf f}^* \|_{p } ,
\quad \text{when} \quad  t  \ge  \frac{1}{\lambda} \frac{\log (1/\epsilon)}{\delta}.
\end{equation}
\end{proposition}

\subsection{Approximation by lazy-training dynamic}\label{subsec:lazy-training-approx}

The network function $ {\bf f}(x, \theta)$ maps from $\calX \times \Theta$ to $\R^d$,
and  $\partial_\theta  {\bf f }  (x, \theta) $ is a $d$-by-$M_\Theta$ matrix.
We denote by $\| \cdot \|$ the vector 2-norm and the matrix operator norm. 
We denote by $B_r$ the open Euclidean ball of radius $r$ (centered at the origin) in $\R^{M_\Theta}$.

\begin{assumption}\label{assump:C1-C2}
Under Assumption \ref{assump:f(x,theta)-L2-F0},
there are positive constants $r$ and $L_1, L_2$ such that $B_{r} \subset \Theta$ and

(C1) For any $\theta \in B_{r}$,  $ \sup_{x \in \calX }   \| \partial_\theta  {\bf f }  (x, \theta) \| \le L_1$.

(C2) 
For any $\theta_1, \theta_2 \in B_r$, 
$ \sup_{x \in \calX} \| \partial_\theta  {\bf f }  (x, \theta_1) -  \partial_\theta  {\bf f }  (x, \theta_2) \| \le L_2 \|\theta_1 - \theta_2\|$.
\end{assumption}

\begin{proposition}\label{prop:NTK-approx}
Under Assumptions \ref{assump:p-q}-\ref{assump:C1-C2}, 
suppose $\theta(0) \in B_{r/2}$, and ${\bf u}( x ,0)  = \bar{\bf u}( x ,0) = 0$,
then for any $\lambda > 0$, 

(i) For any $t\ge 0$ such that $B_{ r/2+ \sqrt{ {t}/{ (2\lambda)} } \| {\bf f}^*\|_p} \subset \Theta$,
\begin{equation}\label{eq:thetat-theta0}
\|\theta(t) - \theta(0)\|
\le \sqrt{ \frac{t}{ 2\lambda} } \| {\bf f}^*\|_p.
\end{equation}

(ii) For any    $ t \le  \frac{1}{2}( \frac{r}{ \| {\bf f}^*\|_p } )^2 \lambda $
(recall that $ \| {\bf f}^*\|_p>0$ by Assumption \ref{assump:p-q-L2}),
\begin{equation}\label{eq:y-bary}
\|  \lambda {\bf u}( \cdot ,t)  - \lambda \bar{\bf u}(  \cdot ,t)  \|_{p } 
 \le   \frac{4\sqrt{2}}{3} L_1 L_2 \sqrt{\lambda} t^{3/2}
 \| {\bf f}^* \|_{p }^2.
 \end{equation}

\end{proposition}

While the idea of proving Proposition \ref{prop:NTK-approx} largely follows that of Theorem 2.2 in \cite{chizat2019lazy},
 we adopt a slightly improved analysis.
Specifically, when we choose $t \sim 1/\lambda$, both bounds in \eqref{eq:thetat-theta0} and \eqref{eq:y-bary} reduce to $O(1/\lambda)$, which echoes Theorem 2.2 in \cite{chizat2019lazy}.
(Note that our normalization of objective multiplies another factor of $\lambda$, and thus our time $t$ corresponds to $T = t \lambda$ for time $T$ in \cite{chizat2019lazy}.)
Here we would like to derive the approximation up to time $t \sim \log (1/\epsilon) /\lambda$, corresponding to time $T \sim \log (1/\epsilon) $ in \cite{chizat2019lazy} instead of $O(1)$ time.
Technically, our analysis also improves the bounding constant in Theorem 2.2 of \cite{chizat2019lazy} by removing a factor of $(e^{C T}-1)$,
which will become a factor of $\epsilon^{-C}$ when $T \sim \log(1/\epsilon)$.
Thus our improvement is important to apply to the case when $\epsilon$ is small. 
The improvement is by the special property of  mean-squared loss, which is equivalent to 
the neural Stein minimizing loss 
$ \calL_\lambda  $ up to a constant, cf. \eqref{eq:Llambda-is-MSE-const}.
We include a proof in Section \ref{sec:proofs} for completeness.

We are ready to derive the main theorem of this section by combining Propositions \ref{prop:NTK-stein} and \ref{prop:NTK-approx}.

\begin{theorem}\label{thm:NTK-stein}
Under Assumptions \ref{assump:p-q}-\ref{assump:C1-C2}, 
let the decomposition of  ${\bf f}^*$ into ${\bf f}_1^*$ and ${\bf f}_2^*$ be as in Proposition \eqref{prop:NTK-stein} and satisfy the conditions therein and for some $ 0 < \epsilon <1$,
$\| {\bf f}_2^* \|_{p }  \le \epsilon \| {\bf f}^* \|_{p }  $.
Suppose $\theta(0) \in B_{r/2}$ and ${\bf u}( x ,0)  = 0$,
then when  $ \lambda >  ( \frac{2 \log (1/\epsilon) }{ \delta} )^{1/2} \frac{ \| {\bf f}^*\|_p }{r}$,
for 
\begin{equation}\label{eq:cond-t0-thm}  
t = \frac{t_0}{\lambda }  \frac{ \log(1/\epsilon)}{\delta }, 
\quad 1 \le t_0 \le ( \frac{r}{ \| {\bf f}^*\|_p } )^2  \frac{\delta}{2\log(1/\epsilon)} \lambda^2,
\end{equation}
we have 
\begin{equation}\label{eq:2eps-bound-thm}
\| \lambda {\bf u}(  \cdot ,t)  - {\bf f}^* \|_{p }  
\le \left( 2 \epsilon +   \frac{ C_1  }{\lambda}  \left( \frac{t_0 \log (1/\epsilon) }{\delta} \right)^{3/2}  \right) \| {\bf f}^* \|_{p },
\end{equation}
where $C_1: = \frac{4\sqrt{2}}{3}  L_1 L_2  \| {\bf f}^* \|_{p } $ is a constant 
determined by ${\bf f}^*$ and Assumption \ref{assump:C1-C2}(C1)(C2).
\end{theorem}

The needed upper bound of $t$, which $\sim \lambda$ is technical (to ensure that $\theta(t)$ stays inside $B_r$).
The theorem suggests that when the {scaleless} optimal critic $ {\bf f}^* $ can be represented using the leading eigen-modes of the zero-time NTK up to an $O(\epsilon)$ residual,
training the neural Stein critic for $\sim \log(1/\epsilon)/\lambda$ time
achieves an approximation of $ {\bf f}^*  $ with a relative error of $O(\epsilon, 1/\lambda)$
up to a factor involving $\log(1/\epsilon)$.
The theoretical bound in Theorem \ref{thm:NTK-stein} does not depend on data dimension $d$ or the domain $\calX$ explicitly, however, such dependence are indirect through the constants $L_1$ and $L_2$ in (C1)(C2).
The same applies to the finite-sample analysis in Theorem \ref{thm:NTK-stein-finite-sample} with respect to the constants in Assumptions \ref{assump:C1-C2} and \ref{assump:C3-C5}.

\begin{remark}[Small initialization]\label{rk:small-u0}
The result extends to when the initial network function ${\bf u}(x,0)$ is non-zero by satisfying $ \|  \lambda {\bf u}(x,0) \|_{p } \le \epsilon \| {\bf f}^* \|_{p }$.
By considering the evolution of $\bar{\bf u}(x,t)$ starting from $\bar{\bf u}(x,0) = {\bf u}(x,0)$,
one can extend Proposition \ref{prop:NTK-approx} where the bounds in \eqref{eq:thetat-theta0} and \eqref{eq:y-bary}
are multiplied by $O(1)$ constant factors
(due to that $\| \lambda {\bf u}(x,0) -  {\bf f}^* \|_{p } \le (1+\epsilon) \| {\bf f}^* \|_{p } \le 2 \| {\bf f}^* \|_{p }$ 
when using the argument in \eqref{eq:bound-thetadot-square-int}).
Proposition \ref{prop:NTK-stein} also extends 
by considering the evolution equation \eqref{eq:parital-w} from ${\bf w} (x, 0) = \lambda {\bf u }( x, t) -  {\bf f}^*(x)$,
and then $\lambda {\bf u }( x, t) $ contributes to another $\epsilon \|  {\bf f}^* \|_{p }$ in the bound \eqref{eq:2eps-bound}. 
\end{remark}

\subsection{{Training with finite samples}}\label{subsec:NTK-finite-sample}

In this subsection, we extend the analysis to training using empirical training loss 
$\hat{L}_\lambda(\theta) $
with $n$ training samples, which can be written as
\begin{equation}\label{eq:def-hatL_lambda}
 \hat{L}_\lambda(\theta) = 
    \E_{x\sim \hat{p}} \left(   - T_q {\bf f}(x, \theta)   + \frac{\lambda}{2}  \| {\bf f}(x, \theta) \|^2 \right),
\end{equation}
where $ \E_{x\sim \hat{p}} $ denotes the sample average over i.i.d. samples $ x_i \sim p$, 
i.e., $\E_{x\sim \hat{p}} g(x) = \frac{1}{n} \sum_{i=1}^n g(x_i)$. 
Again, using continuous-time evolution, the GD dynamic of $\hat{\theta}(t)$ is defined by
\begin{equation}\label{eq:def-GD-theta-n}
\dot{ \hat{\theta}}(t) = - {\partial_\theta} \hat{L}_\lambda( \hat{\theta} (t )),
\quad \hat{\theta}(0) = \theta(0),
\end{equation}
where $\theta(0)$ is some random initialization of the parameters. 
We define
\[
\hat{\bf u}( x,t ) := {\bf f}(x , \hat{\theta}(t)),
\]
assume zero-initialization $\hat{\bf u}( x,0 ) = 0$, 
and similarly as in \eqref{eq:dudt-1} we have
\begin{equation}\label{eq:partialt_hatu_1}
{\partial_t} \hat{\bf u }( x, t) 
 = \langle \partial_\theta {\bf f}( x, \hat{\theta}(t) ), \dot{ \hat{\theta}}(t) \rangle_\Theta.
\end{equation}
As the counterpart of \eqref{eq:def-NTK-t},
we introduce the finite-time empirical NTK  as 
\begin{equation*}
[ \hat{\bf K}_t(x,x') ]_{ij } 
:= \langle \partial_\theta  f_i( x,  \hat{\theta}(t) ), 
 \partial_\theta  f_j(x', \hat{\theta} (t)) \rangle_\Theta, \quad i,j =1, \cdots, d,
\end{equation*}
which we also denote as
\begin{equation}\label{eq:def-NTK-t-n}
\hat{\bf K}_t(x,x')  
= 
\langle \partial_\theta {\bf f}( x, \hat{\theta}(t) ),   
 \partial_\theta  {\bf f}(x', \hat{\theta}(t)) 
\rangle_\Theta 
=   \partial_\theta  {\bf f}( x,  \hat{\theta}(t) ) 
   	 \partial_\theta  {\bf f}(x', \hat{\theta} (t))^T,
\quad \partial_\theta  {\bf f}( x,  \theta ) \in \R^{d \times M_\Theta}.
\end{equation}
Since $ \hat{\theta}(0) = \theta(0)$, we have that 
\begin{equation}\label{eq:hatNTK-t=0}
\hat{\bf K}_0(x,x')  = {\bf K}_0(x,x') = 
 \partial_\theta  {\bf f}( x,  {\theta}(0) ) 
   	 \partial_\theta  {\bf f}(x', {\theta} (0))^T,
\end{equation}
which is a kernel matrix independent from training data, and this fact is important for our analysis. 
Using the definition of $\hat{\bf K}_t(x,x') $, 
the dynamic \eqref{eq:partialt_hatu_1} has the following equivalent form.
%, which is the counterpart of Lemma \ref{lemma:u(x,t)-dynamic}.

\begin{lemma}\label{lemma:hatu-dynamic}
The dynamic of $\hat{\bf u }( x, t) $ follows that 
\begin{equation}\label{eq:evolution-hatu}
{\partial_t} \hat{\bf u }( x, t) 
 =  - \E_{x' \sim \hat{p}} \left( 
 \hat{\bf K}_t(x,x')   \circ \left( \lambda  \hat{\bf u}( x' ,t) -  {\bf s}_q(x')  \right)
 -\nabla_{x'}  \cdot  \hat{\bf K}_t(x,x')  \right).
\end{equation}
\end{lemma}

Our analysis will compare the kernel $ \hat{\bf K}_t(x,x')$ to ${\bf K}_0(x,x')$, 
which allows to compare $\hat{\bf u}(x,t)$ to $\bar{\bf u}(x,t)$ where we will also need to control the error by replacing $\E_{x' \sim {p}}$ with $\E_{x' \sim \hat{p}}$. 
The kernel comparison relies on showing that $\| \hat{\theta}(t) - \theta(0) \|$ is small, which we derive in the next lemma after introducing additional technical assumptions on  the score functions ${\bf s}_p$, ${\bf s}_q$, the function ${\bf f}(x,\theta)$ and its derivatives. 
For a set $B \in \R^d$, we denote by $\bar{B}$ the closure of the set.

\begin{assumption}\label{assump:C3-C5} 
%\old{Suppose ${\bf s}_p$ and ${\bf s}_q$ are bounded on $\calX$,}\xc{revise the proof accourdingly.}
Suppose $ {\bf f}(x, \theta)$ is $C^2$ on $\calX \times \Theta$.
For the $B_r$ as in  Assumption \ref{assump:C1-C2},

(C3) 
There is  $L_3 >0$ such that,
for any $\theta \in B_{r}$,  $ \sup_{x \in \calX }   \| \nabla_x \cdot \partial_\theta  {\bf f }  (x, \theta) \| \le L_3$.

(C4) 
There is $L_4 > 0$ such that,
for any $\theta_1, \theta_2 \in B_r$, 
$ \sup_{x \in \calX} \| \nabla_x \cdot \partial_\theta  {\bf f }  (x, \theta_1) -  \nabla_x \cdot \partial_\theta  {\bf f }  (x, \theta_2) \| \le L_4 \|\theta_1 - \theta_2\|$.

(C5) There are positive constants $b_{(0)}$, $b_{(1)}$ and $b_p$ such that, for any $\theta \in B_r$, 
%\[
$\sup_{x \in \calX} \| {\bf f}(x,\theta) \| \le b_{(0)}$, %\quad
$\sup_{x \in \calX} | \nabla_x \cdot {\bf f}(x,\theta) | \le b_{(1)}$, %\quad
and 
$\sup_{x \in \calX} | {\bf f}(x,\theta) \cdot {\bf s}_p(x)| \le  b_{(0)} b_p$. 
%\]
There are positive constants $C_\calF$ and $b_\calF$ and a positive $\gamma \le 1/2$, such that, w.p. $\ge 1-2n^{-10}$, 
\begin{equation}\label{eq:bound-C5}
\max\{
\sup_{\theta \in \overline{B}_r} | (\E_{x \sim \hat{p}} - \E_{x \sim {p}}) \nabla_x \cdot {\bf f}(x, \theta)  |, \,
\sup_{\theta \in \overline{B}_r} | (\E_{x \sim \hat{p}} - \E_{x \sim {p}}) {\bf f}(x, \theta)  \cdot {\bf s}_p(x) | 
\}
\le \frac{C_\calF}{n^\gamma} + b_\calF \sqrt{\frac{\log n }{n}}.
\end{equation}

(C6) There is a constant $b_q > 0$ such that $\|  {\bf s}_q \|_p \le b_q$ and, for any $\theta \in B_r$, 
%\[
$\sup_{x \in \calX} | {\bf f}(x,\theta) \cdot {\bf s}_q(x)| \le  b_{(0)} b_q$, %\quad
and 
$\sup_{x \in \calX} \| \partial_\theta {\bf f}(x,\theta)^T {\bf s}_q(x)\| \le  L_1 b_q$. 
The random variables  $\|{\bf s}_q(x)\|^2 $ and $\|{\bf f}^*(x)\|^2 $ with $x \sim p$ are sub-exponential.
Specifically, the constant $b_q$ and another constant $b_2 >0 $ satisfy that, 
when $n$ is large (s.t. $\sqrt{{\log n }/{n}}$ is less than a constant possibly depending on $b_2$ and $b_q$), 
w.p. $\ge 1- n^{-10}$, $(\E_{x \sim \hat{p}} - \E_{x \sim {p}}) \|{\bf s}_q(x)\|^2
\le \sqrt{20} b_q^2 \sqrt{{\log n }/{n}}$;
w.p. $\ge 1- n^{-10}$, $(\E_{x \sim \hat{p}} - \E_{x \sim {p}}) \|{\bf f}^*(x)\|^2
\le \sqrt{20} b_2 \sqrt{{\log n }/{n}}$.
\end{assumption}

 In the below, we adopt big-O notation to facilitate exposition and $\tilde{O}$ stands for the involvement of a log factor. 
The constant dependence can be tracked in the proof.
We derive a non-asymptotic result which holds at a sufficiently large finite sample size $n$.

\begin{remark}[Uniform-law and rates]
The condition (C5) gives the standard uniform-law bounds which can be derived by Rademacher complexity of the relevant function classes over $\theta \in \overline{B}_r$, see e.g. \cite{wainwright2019high}, where we treat $ {\bf s}_p$ as a fixed bounded function on $\calX$.
The $C_\calF n^{-\gamma}$ term corresponds to the Rademacher complexity which is bounded by the covering complexity of the function class,
and the exponent $\gamma$ is determined by the covering number bound that usually involves the dimensionality of the domain and the regularity of the function class. The constant $b_\calF$ corresponds to the boundedness of the functions, namely $ b_{(1)}$ and $b_{(0)} b_p$.
We note that while (C5) gives an overall $\tilde{O}(n^{- {\gamma }})$ bound, it will only be used in the middle-step analysis (Lemma \ref{lemma:hatthetat-theta0}) and our final finite-sample bound in Theorem \ref{thm:NTK-stein-finite-sample} achieves the parametric rate of $\tilde{O}(n^{-1/2})$.
\end{remark}

\begin{remark}[Boundedness condition]
In (C5)(C6), the uniform boundedness of ${\bf f}(x,\theta) \cdot {\bf s}_p(x)$,  ${\bf f}(x,\theta) \cdot {\bf s}_q(x)$,  and $\| \partial_\theta {\bf f}(x,\theta)^T {\bf s}_q(x)\|$ on $\calX$ can be fulfilled if ${\bf f}(x,\theta)$ and $\partial_\theta {\bf f}(x,\theta)$ vanish sufficiently fast when approaching $\partial \calX$, and thus allowing the score functions ${\bf s}_p$ and ${\bf s}_q$ to be potentially unbounded in $L^\infty(\calX)$.
The score functions still need to be in $L^2(p)$ and the tails cannot be too heavy to guarantee that $\|{\bf s}_q(x)\|^2 $ and  $\|{\bf f}^*(x)\|^2 $ are sub-exponential. The Bernstein-type concentration bound of in (C6) follows from the standard property of sub-exponential random variable, see e.g. \cite[Chapter 2]{wainwright2019high}.
\end{remark}

\begin{lemma}\label{lemma:hatthetat-theta0}
Under Assumptions \ref{assump:p-q}-\ref{assump:C3-C5}, 
suppose $\theta(0) \in B_{r/2}$, and $\hat{\bf u}( x ,0)   = 0$. 
Then for any $\lambda > 0$, 
there is positive integer $n_\lambda$ s.t. when $n > n_\lambda$, 
under a good event which happens  w.p. $\ge 1-3n^{-10}$, 
for any $t \ge 0$ s.t. $  \|{\bf f}^*\|_p \sqrt{{t}/{\lambda}} \le {r}/{2}$, 
$\| \hat{\theta} (t) - {\theta} (0) \| \le  \sqrt{ {t}/{ \lambda}  } \|{\bf f}^*\|_p$
and  $ \hat{\theta} (t) \in B_r$.
\end{lemma}

Based on the lemma, we derive the comparison of $ \hat{\bf u}(x,t)$ with $ \bar{\bf u}(x,t)$ in the following proposition.

%\xc{$n > \max\{ n_5, n_\lambda, n_{\lambda,t} \}$  and under the intersection (of the good event in Lemma \ref{lemma:hatthetat-theta0} and) $E_a \cap E_b \cap E_c \cap E_5$.
%$n_\lambda$ defined in \eqref{eq:def-n-lambda}: by $\lambda$ and $ \|{\bf f}^*\|_p$, when $ \|{\bf f}^*\|_p$ is small needs enough samples to resolve it.
%$n_5$ defined in \eqref{eq:def-n5},  depends on $\log M_\Theta$.
%$n_{\lambda,t}$ defined in \eqref{eq:def-n-lambda-t}, depends on $\lambda t$ and $\log M_\Theta$.}

\begin{proposition}\label{prop:hatu-baru}
Under the same assumption as in Lemma \ref{lemma:hatthetat-theta0},
$\hat{\bf u}( x ,0)   = \bar{\bf u}( x ,0)  =0$. For $\lambda > 0$ and any $ t \le  ( \frac{r}{ 2\| {\bf f}^*\|_p } )^2 \lambda $,
when $n$ is sufficiently large (depending on $\lambda, t, \|{\bf f}^*\|_p, \log M_\Theta$), 
under the intersection of the good event in Lemma \ref{lemma:hatthetat-theta0}
and another good event which happens w.p.$\ge 1-4 n^{-10}$,
\begin{equation}\label{eq:bound-prop-hatu-baru}
\|  \lambda \hat{\bf u}( \cdot ,t)  - \lambda \bar{\bf u}( \cdot ,t)  \|_{p } 
 \le    C_2 
(   1 +   \lambda t  )\left(   \lambda^{1/2} t^{3/2} \|{\bf f}^*\|_p + (1+\lambda t )  \lambda t  \sqrt{\frac{ \log n + \log M_\Theta}{n}}  \right) 
 \end{equation}
 where 
 %$C_2$ is an $O(1)$ constant determined by $L_1, L_2, L_3, L_4$, $\| {\bf s}_p\|_\infty$ and $\| {\bf s}_q\|_\infty$.
 $C_2$ is an $O(1)$ constant bounded by a multiple of $(1+ b_q +  \| {\bf f}^*\|_p)$
 and the constant factor depends on $L_1, L_2, L_3, L_4$.
\end{proposition}

\begin{remark}[Largeness of $n$ and the relation to $ \|{\bf f}^*\|_p$]
\label{rk:n-C-prop-finite-sample}
The largeness requirement of $n$ depends on $\lambda, t, \|{\bf f}^*\|_p, \log M_\Theta$, 
and specifically, $n > \max\{ n_\lambda, n_5, n_{\lambda,t} \}$ which are defined in \eqref{eq:def-n-lambda}\eqref{eq:def-n5}\eqref{eq:def-n-lambda-t} respectively:
$n_\lambda$ depends on $\lambda$, $ \|{\bf f}^*\|_p$ and constants $b_2$, $b_q$, $C_\calF$, $b_\calF$, and when $ \|{\bf f}^*\|_p$ is small it calls for larger $n$, specifically $ \| {\bf f}^*\|_p^2 > \lambda n^{-\gamma }$ as the leading term, up to constant;
$n_5$  depends on $\log M_\Theta$;
%$n_6$  depends on $b_2$,$b_q$;
$n_{\lambda,t}$ depends on $\lambda t$ and $\log M_\Theta$.
The construction of the constant $C_2$ can be found in the proof.
The first term in the r.h.s. of \eqref{eq:bound-prop-hatu-baru} is the analog of \eqref{eq:y-bary} in Proposition \ref{prop:NTK-approx}(ii).
When $\| {\bf f}^*\|_p$ is small, $C_2$ stays bounded (as long as $b_q$ is bounded)
and the first term in \eqref{eq:bound-prop-hatu-baru}  is proportional to $\| {\bf f}^*\|_p$.
The second term of the order  $\tilde{O}(n^{-1/2})$ is due to the finite-sample training,
and when $\| {\bf f}^*\|_p$ is small, it suggests that larger $n$ is needed so as to make the second term balance with the first term.
\end{remark}

The main theorem for finite-sample lazy training follows by combining Propositions \ref{prop:NTK-stein} and \ref{prop:hatu-baru}.
\begin{theorem}\label{thm:NTK-stein-finite-sample}
Under Assumptions \ref{assump:p-q}-\ref{assump:C3-C5}, 
suppose the decomposition of  ${\bf f}^*$ satisfies the same condition as in Theorem \ref{thm:NTK-stein} with respect to $ 0 < \epsilon <1$,
 $\theta(0) \in B_{r/2}$, and $\hat{\bf u}( x ,0) =0$. 
For $  \lambda  > 2 (\frac{ \log(1/\epsilon)}{\delta })^{1/2} \frac{ \| {\bf f}^*\|_p }{r}  $, let
\begin{equation}\label{eq:cond-t0-thm-finite-sample}  
t = \frac{ t_0}{\lambda }  \frac{ \log(1/\epsilon)}{\delta }, 
\quad 
 1 \le {t_0}   \le  ( \frac{r}{ \| {\bf f}^*\|_p } )^2 \frac{\delta }{ 4 \log(1/\epsilon)} \lambda^2.
\end{equation}
Then, 
when $n$ is sufficiently large (depending on $\lambda, \|{\bf f}^*\|_p, \log M_\Theta$ and $t_0 \frac{ \log (1/\epsilon) }{\delta}$), 
w.p.$\ge 1-7 n^{-10}$,
\begin{equation}\label{eq:2eps-bound-thm-finite-sample}
\| \lambda \hat{\bf u}(  \cdot ,t)  - {\bf f}^* \|_{p }  
\le \left( 2 \epsilon +   \frac{ C_2 }{\lambda}  
	%(1+\frac{t_0 \log (1/\epsilon) }{\delta}) \left( \frac{t_0 \log (1/\epsilon) }{\delta} \right)^{3/2} 
	\kappa_1
	 \right) \| {\bf f}^* \|_{p }
+ C_2 
	%(1+\lambda t)^2(\lambda t)   
	\kappa_2
	\sqrt{\frac{ \log n + \log M_\Theta}{n}}, 
\end{equation}
where $C_2$ is as in Proposition \ref{prop:hatu-baru} and
$\kappa_1$, $\kappa_2$ are constant factors involving powers of $t_0 \frac{ \log (1/\epsilon) }{\delta}$.
\end{theorem}

The result in Theorems \ref{thm:NTK-stein} and \ref{thm:NTK-stein-finite-sample} suggests that by using a larger $\lambda$ at the beginning of the training process, the training can achieve the NTK kernel learning solution more rapidly.
In practice, staying with large $\lambda$ too long would lead to the worsening of the model, 
and we propose the gradual annealing scheme of $\lambda$ as described in Section \ref{subsec:staged}
so as to combine the benefits of both large $\lambda$ in the beginning and small $\lambda$ in the later phases of training. 
The theoretical benefit of using large $\lambda$ at the beginning phase of training  is supported  by experiments in Section \ref{sec:exp}.

\section{{Applications to testing and model evaluation}}

\subsection{Goodness-of-Fit (GoF) testing}\label{sec:gof_test}

In a GoF test, we are given $n_{\rm sample}$ data samples $x_i \sim p$ and a model distribution $q$, 
and assume we can sample from $q$ as well as access its score function ${\bf s}_q$.
To apply the neural Stein test, we conduct a training-test split of the samples, where the two splits have $n_{\rm tr}$ and $n_{\rm GoF}$ samples respectively, and $n_{\rm sample} = n_{\rm tr} + n_{\rm GoF}$.
We first train a neural Stein critic ${\bf f}(x, \theta)$ from the training split $\{ x_i^{\rm tr} \}_{i=1}^{n_{\rm tr}}$,
and then we compute the following test statistic on the test split $\{ x_i \}_{i=1}^{n_{\rm GoF}}$
\begin{equation}
    \hat{T} = \frac{1}{n_{\rm GoF}} \sum_{i=1}^{n_{\rm GoF}} T_q {\bf f}(x_i, \theta),
    \label{eq:test_stat}
\end{equation}
which can be viewed as a sample-average estimator of ${\rm SD}[ {\bf f} (\cdot, \theta) ]$ as defined in \eqref{eq:sd_operator}. 

To assess the null hypothesis as in \eqref{eq:hypotheses}, we will adopt a bootstrap strategy to compute the test threshold $t_{\rm thresh}$ by drawing samples from $q$, to be detailed in Section \ref{sec:bootstrap}. 
We also derive GoF test consistency analysis in Section \ref{sec:consistency_analysis}.

\subsubsection{Bootstrap strategy to compute test threshold}\label{sec:bootstrap}

The bootstrap strategy draws independent samples $y_i \sim q$ to simulate the distribution of the test statistic under  $H_0$.
We denote the test statistic as  $\hat{T}_{\rm null}$, which can be computed from a set of samples $\{y_i\}_{i=1}^{n_{\rm GoF}}$ as $\hat{T}_{\rm null} = \frac{1}{n_{\rm GoF}} \sum_{i=1}^{n_{\rm GoF}} T_q {\bf f}(y_i,\theta)$.
We will set $t_{\rm thresh}$ as the $(1-\alpha)$ quantile of the distribution of $\hat{T}_{\rm null}$.
To simulate the distribution, one can compute $ \hat{T}_{\rm null} $ in  $n_{\rm boot}$ independent replicas, and then set $t_{\rm thresh}$ to be the quantile of the empirical distribution. 
This means that the $\hat{T}_{\rm null}$ in each replica is computed from a ``fresh'' set of  $n_{\rm GoF}$ samples from $q$. 
We call these independent copies of $\hat{T}_{\rm null}$ the ``fresh null statistics''. 
Note that this would require evaluating the trained neural network on $n_{\rm boot} n_{\rm GoF}$ many samples $y_i$ in total,
which can be significant since $n_{\rm boot} $ is usually a few hundred (we use $n_{\rm boot}=500$ in all experiments). 
To accelerate computation, 
we propose an ``efficient'' bootstrap procedure, which begins by drawing $n_{\rm pool}$ samples from $q$, $n_{\rm pool}=r_{\rm pool}\cdot n_{\rm GoF}$.
The trained network is evaluated on the  $n_{\rm pool}$ samples of $y_i$ to obtain the values of $ T_q {\bf f}(y_i, \theta) $,
and then we compute $n_{\rm boot} $ many times of the $n_{\rm GoF}$-sample  average by drawing from the pool with replacement.
We call the set of values of $\hat{T}_{\rm null}$ computed this way the ``efficient null statistics''. 
 We observed that setting $r_{\rm pool}=50$ is usually sufficient to render a null statistic distribution that resembles that of the fresh statistics. 
With $n_{\rm boot}=500$, this yields about ten times  speedup in the computation of the bootstrap. 

To illustrate the validity of the bootstrap strategy, we apply the method to a critic trained on 1D Gaussian mixture data as described in Appendix~\ref{sec:1d_experiment}. 
In this case, a partially trained critic, displayed in Figure~\ref{fig:GoF_consistency}(\textbf{A}), is used to compute the test statistics. 
We set $n_{\rm GoF}=100$, and to better illustrate the empirical distribution of the test statistics we use $n_{\rm boot}=10,000$ replicas. 
Three  sets of 10,000 test statistics are computed: 
1) using $n_{\rm GoF}$ fresh samples from $p$ to compute each $\hat{T}$, 
2) using $n_{\rm GoF}$ fresh samples from $q$ to compute $\hat{T}_{\rm null}$, 
and 3) efficient bootstrap null statistics computed from a pre-generated pool with $r_{\rm pool}=50$. 
The empirical distributions are visualized by histograms  in Figure~\ref{fig:GoF_consistency}(\textbf{B}). 
The plot shows a clear disparity between the distribution of test statistics under $H_1$ and the two distributions of test statistics under $H_0$,
and  the validity of the efficient bootstrap procedure is demonstrated by the similarity of the histograms of the fresh and efficient null statistics.

\begin{figure}[t]
    \centering
    \includegraphics[width=\textwidth]{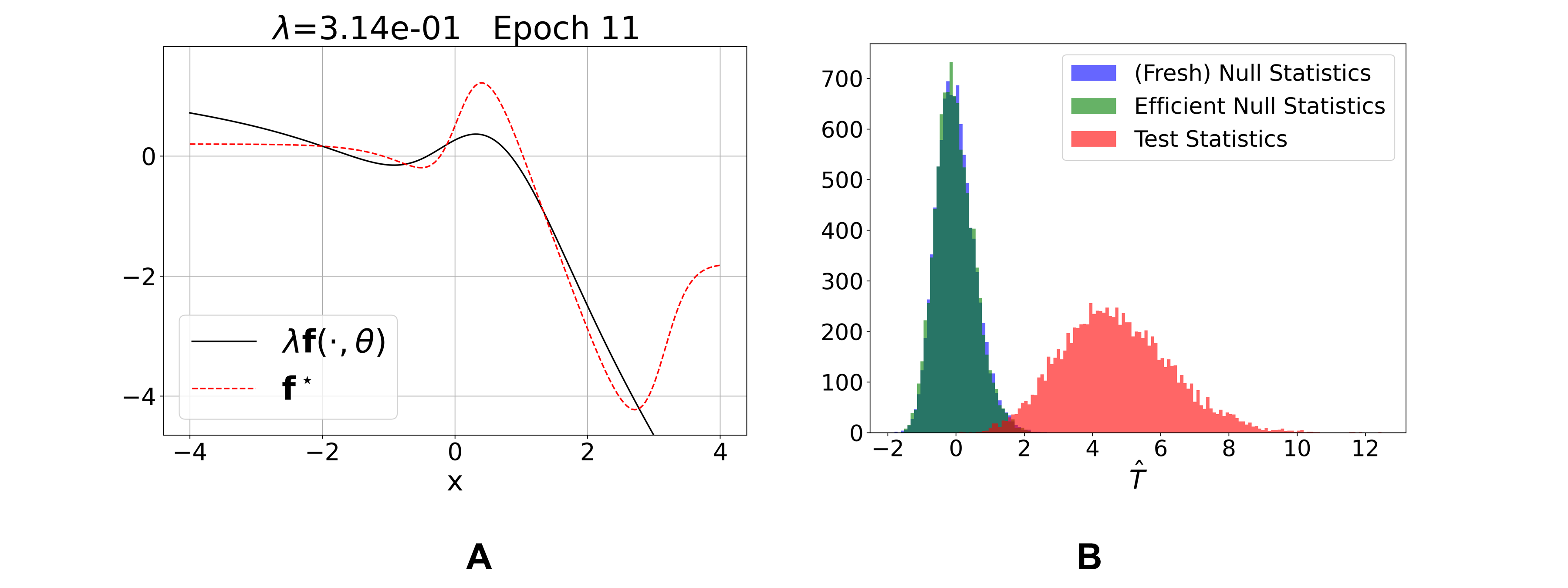}
    \vspace{-20pt}
    \caption{
    In (\textbf{A}), a scaleless neural Stein critic function $\lambda{\bf f}(\cdot,\theta)$ which is trained for a few epochs is compared to the scaleless optimal critic ${\bf f}^*$. 
    This trained neural critic may be further improved by more training time, as shown in Figure \ref{fig:1d_critic}.  
   (\textbf{B}) The trained critic, though not fully fitting the ${\bf f}^*$, still provides a GoF test with significant power:  
   the distribution of the $\hat{T}_{\rm null}$ under $H_0$ is centered around zero,
   while the distribution of $\hat{T}$ under $H_1$ has a non-zero mean around 5, which is significantly beyond the right 95\% quantile of the $\hat{T}_{\rm null}$ distribution. 
   The distribution of $\hat{T}_{\rm null}$ (blue) and that approximately computed from the efficient bootstrap (green) are close. 
 See Section~\ref{sec:bootstrap} for the details of the test statistic computation.
The details of the 1D Gaussian mixture distributions are given in Appendix~\ref{sec:1d_experiment}.
    }\label{fig:GoF_consistency}
\end{figure}

\subsubsection{{Theoretical guarantee of test power}}\label{sec:consistency_analysis}

Let the trained neural Stein critic be $\hat{\bf f}(x):= {\bf f}(x, \hat{\theta}(t))$.
We first observe the asymptotic normality of the test statistics $\hat{T}$ and $\hat{T}_{\rm null}$:
by definition, they are sample averages over i.i.d. test samples $x_i \sim p$ ($y_i \sim q$) independent from the training split.
Thus, conditioning on training from $n_{\rm tr}$ samples,
 the test statistic is an independent sum averaged over i.i.d. random variables
\[
\xi_i: = T_q \hat{\bf f}(x_i) = {\bf s}_q \cdot \hat{\bf f}(x_i) + \nabla \cdot \hat{\bf f}(x_i), \quad i=1, \cdots, n_{\rm GoF}.
\] 
By Assumption \ref{assump:C3-C5}(C5)(C6) (and that $\hat{\theta}(t) \in B_r$ by Lemma \ref{lemma:hatthetat-theta0}), 
$\xi_i$ are uniformly bounded (either evaluated on $x_i$ or $y_i$ in $\calX$).
Thus, by Central Limit Theorem,  as $n_{\rm GoF} \to \infty$, 

\begin{itemize}

\item
$\sqrt{n_{\rm GoF}} \hat{T}$ converges in distribution to $\calN ( \mu_1, \sigma_1^2)$
where $\mu_1 = \E_{x \sim p} T_q \hat{\bf f}(x)$, and $ \sigma_1^2 = {\rm Var}_{x \sim p} ( T_q \hat{\bf f}(x) )$,

\item
$\sqrt{n_{\rm GoF}} \hat{T}_{\rm null}$ converges in distribution to $\calN ( \mu_0, \sigma_0^2)$
where $\mu_0 = \E_{y \sim q} T_q \hat{\bf f}(y) = 0$, and $ \sigma_0^2 = {\rm Var}_{y \sim q} ( T_q \hat{\bf f}(y) )$.
\end{itemize}
The closeness to normal density is typically observed when $n_{\rm GoF}$ is a few hundred, see Figure~\ref{fig:GoF_consistency}(\textbf{B}) where $n_{\rm GoF}=100$.
The uniform boundedness of $\xi_i$ implies that $ \sigma_1^2$ and $ \sigma_0^2$ are at most $O(1)$ constants. 
This indicates that as long as the training obtains a critic $\hat{\bf f}(x)$ that makes $\mu_1 = \E_{x \sim p} T_q \hat{\bf f}(x) > 0$, 
then the GoF test can successfully reject the null when $n_{\rm GoF}$ is large enough.

We also derive the finite-sample test power guarantee of the neural Stein GoF test in the following corollary by incorporating the learning guarantee in Section \ref{sec:lazy_training},
and the proof is left to Section \ref{sec:proofs}.
%Theorem \ref{thm:NTK-stein-finite-sample}.

\begin{corollary}\label{cor:GoF-test-power}
Consider a significance level $\alpha$ and target Type-II error $\beta$, $0 < \alpha, \beta < 1$.
Suppose the conditions in Theorem \ref{thm:NTK-stein-finite-sample} are satisfied 
with $\epsilon <1/2$, $\lambda > 0$ as required and $ 2\epsilon + C_2 \kappa_1 \lambda^{-1} < 0.8$,
and $n_{\rm tr}$ is large enough for the good event (call it $E_{\rm tr}$) to hold, 
$ C_2 \kappa_2
	\sqrt{{( \log n_{\rm tr} + \log M_\Theta)}/{n_{\rm tr}}} < 0.1 \| {\bf f}^*\|_p$,
and  $7n_{\rm tr}^{-10} < 0.1 \min\{\alpha,  \beta\}$. 
Then the GoF test using $\hat{\bf f}(x)= {\bf f}(x, \hat{\theta}(t))$ achieves a significance level $\alpha$ and a test power at least $1- \beta$,
 if 
\[
\sqrt{2} (b_{(0)} b_q + b_{(1)} )
\left( 1+ \sqrt{\log \frac{1}{\alpha}} + \sqrt{\log \frac{1}{\beta}} \right) 
\lambda n_{\rm GoF}^{-1/2} 
<  0.1 \| {\bf f}^*\|_p^2.
\]
In particular, the test power $\to 1$ as $n_{\rm tr},  n_{\rm GoF} \to \infty$.
\end{corollary}

\subsection{Evaluation of EBM generative models}\label{sec:eval_generative_models}

%\xc{stress that here it is less about test power, but more about the interpretation of the  critic itself}

We consider the application of the Stein discrepancy in evaluating generative models. 
The model evaluation problem is to detect how the model density $q$ (by a given generative model) is different from the unknown data density $p$. While this can be formulated as a GoF testing, in machine learning applications it is of interest where and how much the two densities differ rather than rejecting or accepting the null hypothesis. To this end, we note that the trained Stein critic can be used as an indicator to reveal where $p$ and $q$ locally differ.

In the case of EBMs, the model probability density  takes the form as 
$q(x) = {\exp(-E_\phi(x))}/{Z}$,
where  $Z$ is the normalizing constant 
and the real-valued energy function $E_\phi(x)$ is parameterized by $\phi$, which is another model such as a neural network (so the set of parameters $\phi$ differs from the neural Stein critic parametrization $\theta$). 
The score function of $q$, therefore, equals the gradient of the energy function, i.e.,
\begin{equation}
    {\bf s}_q = -\nabla E_\phi(x),
    \label{eq:ebm_score}
\end{equation}
which can be computed from the parameterized form of $E_\phi(x)$. In the case that the energy function is represented by a deep generative neural network, the gradient \eqref{eq:ebm_score} can be computed by back-propagation, which is compatible with the auto-differentiation implementation of widely-used deep network platforms. In practice, the evaluation metric of a trained EBM can be computed on a holdout validation dataset.

Below we give the expression of the score function for Gaussian-Bernoulli Restricted Boltzmann Machines (RBMs)~\cite{grathwohl2020learning}, 
which is a specific type of EBM. 
The energy of an RBM with latent Bernoulli variable $h$ is defined as
$    E(x, h|B,b,c) = -\frac{1}{2}x^\text{T}Bh - b^\text{T}x - c^\text{T}h + \frac{1}{2} \|x\|^2$.
Therefore, the score function has the expression 
  $  {\bf s}_q(x) = b - x + B\cdot {\rm tanh}(B^\text{T}x + c)$.
In Section \ref{sec:mnist}, we evaluate Gaussian-Bernoulli RBMs using a Stein discrepancy test computed via neural Stein critic functions, 
and we also compare {neural Stein critics trained} using different regularization strategies.

\section{Experiment}\label{sec:exp}

In this section, we present numerical experiments applying the proposed neural Stein method on differentiating a data distribution $p$ (from which we have access to a set of data samples) and a model distribution $q$ (of which the score function is assumed to be known). We compare the proposed neural Stein method with staged regularization
to that with fixed $L^2$ regularization, as well as to a kernel Stein method previously developed in the literature.

In Sections~\ref{sec:gm_sim} and \ref{sec:compare_to_ksd}, we consider a set of Gaussian mixture models for both $p$ and $q$. 
In Section~\ref{sec:mnist}, the data are sampled from the MNIST handwritten digits dataset \cite{lecun2010mnist}, and the model distribution is a Gaussian-Bernoulli RBM neural network model. Codes to reproduce the results in this section can be found at the following repository: \url{https://github.com/mrepasky3/Staged_L2_Neural_Stein_Critics}.

\subsection{Gaussian mixture data}\label{sec:gm_sim}

\subsubsection{Simulated datasets}\label{sec:gm_sim_data}

We compare the performance of a variety of neural Stein critics in the scenario in which the distributions are bimodal, $d-$dimensional Gaussian mixtures, following an example studied in \cite{liu2020learning}. The model $q$ has equally-weighted components with means $\mu_1=\mathbf{0}_{d}$ and $\mu_2=0.5\times\mathbf{1}_{d}$, both having identity covariance. The data distribution $p$ has the following form:
\begin{equation}
    p= \frac{1}{2} \mathcal{N}\left(\mu_1,\begin{bmatrix}
    1 & \rho_1 & \mathbf{0}_{d-2}^\text{T} \\
    \rho_1 & 1 & \mathbf{0}_{d-2}^\text{T} \\
    \mathbf{0}_{d-2} & \mathbf{0}_{d-2} & \mathbf{I}_{d-2}
    \end{bmatrix}\right) + \frac{1}{2} \mathcal{N}\left(\mu_2, \begin{bmatrix}
    \omega^2 & \omega\rho_2 & \mathbf{0}_{d-2}^\text{T} \\
    \omega\rho_2 & 1 & \mathbf{0}_{d-2}^\text{T} \\
    \mathbf{0}_{d-2} & \mathbf{0}_{d-2} & \mathbf{I}_{d-2}
    \end{bmatrix} \right),
    \label{eq:data_distribution_gm}
\end{equation}
where $\rho_1 = -\rho_2$ represents the covariance shift with respect to the model distribution. We also introduce a parameter $\omega$, which scales the covariance matrix of the second component of the mixture. We examine this scenario in three settings of increasingly high dimensions. In each setting of this section, we fix the parameters $\rho_1=0.5$ and $\omega=0.8$. 

\subsubsection{Experimental setup}\label{sec:gm_sim_setup}

\paragraph{Neural network training.}
% net architecture
All neural Stein critics trained  are two-hidden-layer MLPs with Swish activation~\cite{ramachandran2017searching}, where each hidden layer includes 512 hidden units. 
% weight initialization
The weights of the linear layers of the models are initialized using the standard PyTorch weight initialization, and the biases of the linear layers are initialized as zero. 
% SGD
The Adam optimizer (using the default momentum parameters by PyTorch) was used for network optimization. 
% training sample size, learning rate, epochs
The critics are trained using $n_{\rm tr}=$2,000 samples from the data distribution $p$; the learning rate is fixed at $10^{-3}$, and the batch size is 200 samples. 
All models are trained for 60 epochs. 
To compute the divergence $\nabla \cdot {\bf f}(x,\theta)$ in dimension greater than two, we use Hutchinson's unbiased estimator of the trace Jacobian \cite{hutchinson1989stochastic} following \cite{grathwohl2020learning}.
% staging
{In these experiments, the staging of $\lambda$ according to the staged regularization strategies occurs every $B_w=10$ batches, which is equivalent to the end of every epoch.} For each choice of regularization strategy, we train 10 neural Stein critic network replicas.
We compare the neural Stein critics trained using fixed-$\lambda$ regularization vs. staged regularization schemes.

\paragraph{Computation of MSE.}
Having knowledge of the score functions of both $p$ and $q$, we are able to compute the $\widehat{\rm MSE}_q$ \eqref{eq:mse} using {$n_{\rm te}=20,000$} samples from the model $q$. Over the course of training, we also computed {$\widehat{\rm MSE}^{(m)}_p$ } \eqref{eq:validation_mse} using $n_{\rm val}=1,000$ samples from the data distribution $p$. This {monitor} is used to select the ``best" model over the course of training, where the model with the lowest {$\widehat{\rm MSE}^{(m)}_p$ } value is selected. 

\paragraph{{Computation of test power}.}
After a trained neural critic is obtained, we perform the GoF hypothesis testing as described in Section~\ref{sec:gof_test}, including computing the test threshold $t_{\rm thresh}$ by the efficient bootstrap strategy.
We set the significance level $\alpha = 0.05$,
the number of bootstrap $n_{\rm boot}=500$,
and the efficient bootstrap ratio number $r_{\rm pool}=50$.
We use $n_{\rm GoF}$ specific for each data distribution example, see below.

By definition, the test power is the probability at which the null hypothesis $H_0:p=q$ is correctly rejected.
For each trained neural critic from a given training split of $p$ samples, we estimate the test power of this critic empirically by conducting $n_{\rm run}$ times of the GoF tests 
(including independent realizations of the test split of $p$ samples and computing the $t_{\rm thresh}$ by bootstrap) 
and counting the frequency when the null hypothesis is rejected. 
Because the training also contributes to the randomness of the quality of the GoF test, we repeat the procedure for $n_{\rm replica}$ replicas (including training the model on an independent realization of the training split in each replica), and report the mean and standard deviation of the estimated test power from the replicas. 
We use $n_{\rm run} =500$ and  $n_{\rm replica} = 10$  in our experiments.

\paragraph{Setup per example.}
We introduce the specifics of the experimental setup for Gaussian mixtures in 2D, 10D, and 25D below.  
 
\begin{itemize}
\item[(a)] {\it 2-dimensional mixture.} The selected fixed-$\lambda$ regularization strategies in 2 dimensions are $\lambda\in\{1\times 10^{-3}, 1\times 10^{-2}, 1\times 10^{-1}, 1\times 10^0\}$. We compare these fixed schemes to a few staged regularization strategies. Using the notation of Equation \eqref{eq:lam_staging}, these are the $\Lambda(1\times 10^0, 5\times 10^{-2}, 0.95)$ and $\Lambda(1\times 10^0, 5\times 10^{-2}, 0.90)$ staged regularization schemes. In 2D, the number of test samples for the GoF testing is $n_{\rm GoF}=75$.

\item[(b)] {\it 10-dimensional mixture.} For this setting, we examine fixed $\lambda\in\{2.5\times 10^{-4}, 1\times 10^{-3}, 4\times 10^{-3}, 1.6\times 10^{-2}, 6.4\times 10^{-2}, 2.56\times 10^{-1}, 1.024\times 10^0\}$. We analyze the staging schemes $\Lambda(5\times 10^{-1}, 1\times 10^{-3}, 0.80)$ and $\Lambda(5\times 10^{-1}, 1\times 10^{-3}, 0.85)$. For the GoF power analysis, we use $n_{\rm GoF}=200$ test samples from the data distribution $p$.

\item[(c)] {\it 25-dimensional mixture.} The fixed regularization weights are selected as $\lambda\in\{2.5\times 10^{-4}, 1\times 10^{-3}, 4\times 10^{-3}, 1.6\times 10^{-2}, 6.4\times 10^{-2}\}$. The staging schemes analyzed in this case are $\Lambda(4\times 10^{-1}, 5\times 10^{-4}, 0.80)$, $\Lambda(4\times 10^{-1}, 5\times 10^{-4}, 0.85)$, and $\Lambda(4\times 10^{-1}, 5\times 10^{-4}, 0.90)$. We use $n_{\rm GoF}=500$ test samples from the data distribution $p$ for the GoF tests in 25D.
\end{itemize}

\begin{figure}[t]
    \centering
    \includegraphics[width=\textwidth]{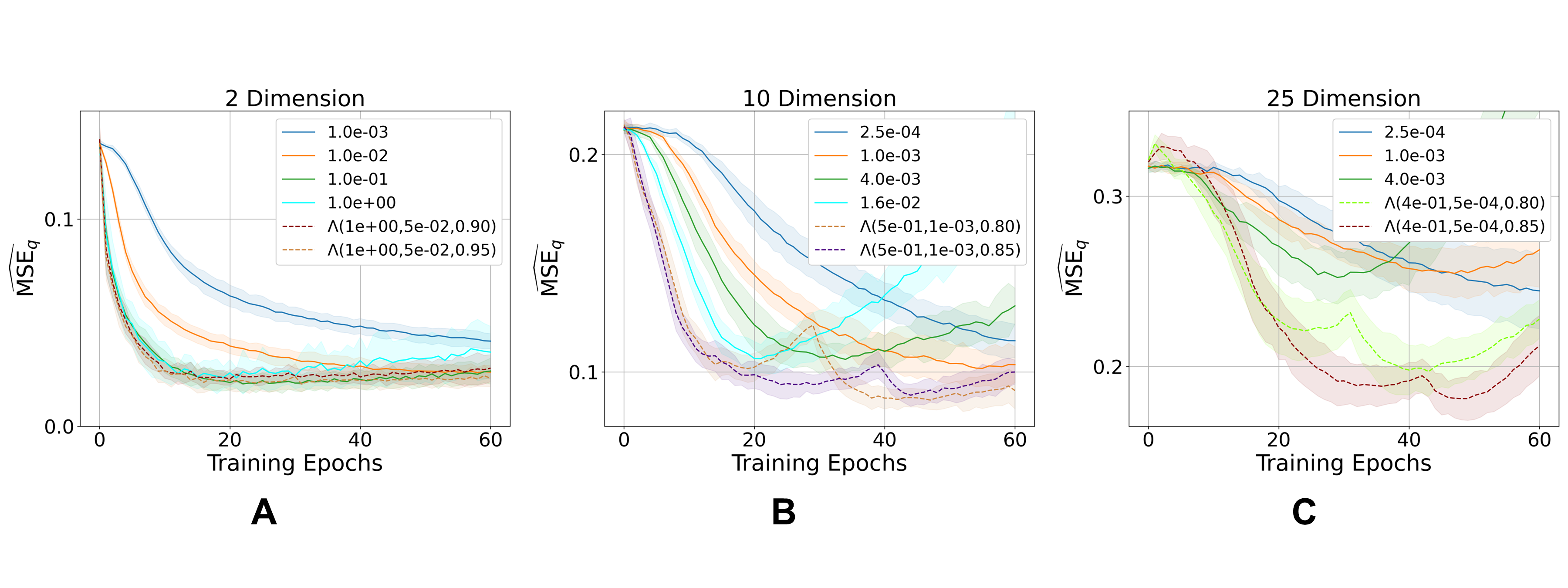}
    \vspace{-30pt}
    \caption{
    $\widehat{\rm MSE}_q$ results for the Gaussian mixture model settings in 2 dimensions (\textbf{A}), in 10 dimensions (\textbf{B}), and in 25 dimensions (\textbf{C}). In each case, 10 models are trained for 60 epochs using each regularization strategy; the mean and standard deviation are visualized over these 10 models. The legends indicate the value of $\lambda$ corresponding to each curve. We use $\Lambda(1\times10^{0},5\times10^{-2},\cdot)$ staging in 2D, $\Lambda(5\times10^{-1},1\times10^{-3},\cdot)$ in 10D, and $\Lambda(4\times10^{-1},5\times10^{-4},\cdot)$ in 25D. The advantage of staged regularization becomes more pronounced with increasing dimension.}\label{fig:GM_MSE}
\end{figure}

\begin{figure}[t]
    \centering
    \includegraphics[width=\textwidth]{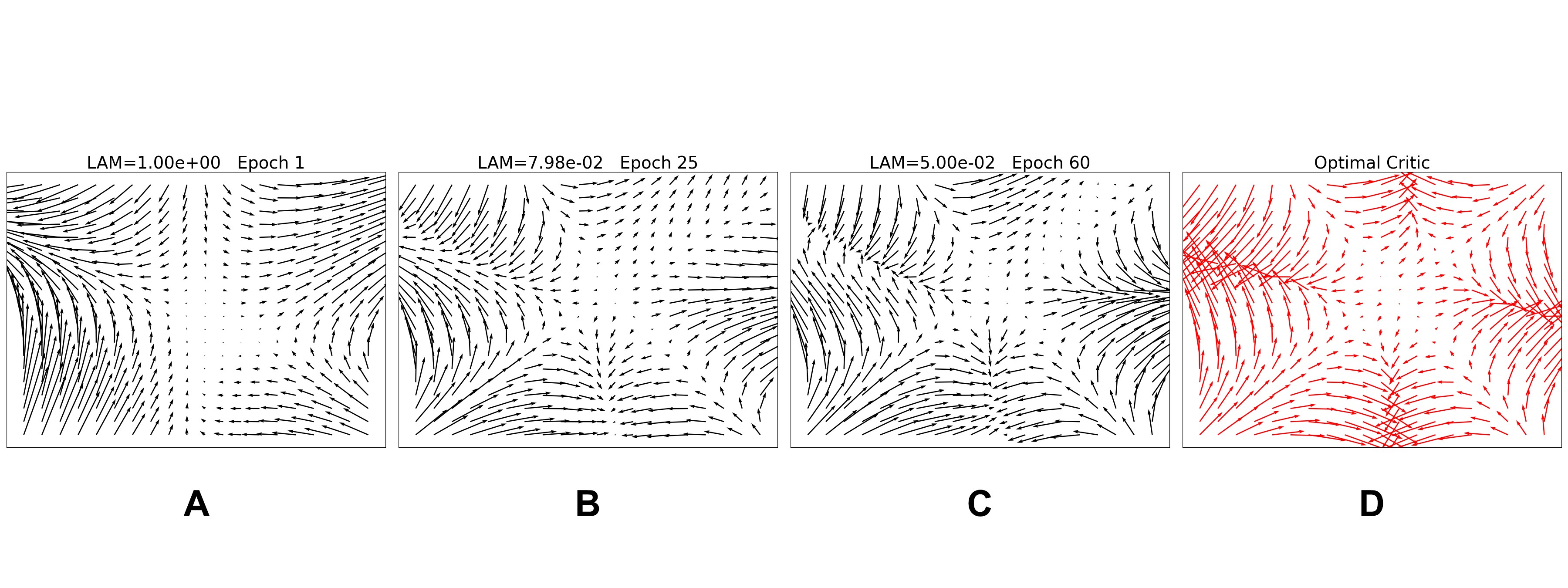}
    \vspace{-30pt}
    \caption{
    A visualization of the scaleless {neural} Stein critic {$\lambda{\bf f}(\cdot,\theta)$} throughout training, where we use the $\Lambda(1\times 10^0, 5\times 10^{-2}, 0.90)$ staging strategy on the 2-dimensional Gaussian mixture data. In (\textbf{A}), the critic quickly approximates some regions in the first epoch, whereas in (\textbf{B}) and (\textbf{C}) the critic makes smaller adjustments throughout the domain, creating an approximation of the scaleless optimal critic function {${\bf f}^*$} \eqref{eq:def-fstar} pictured in (\textbf{D}).}\label{fig:2d_critic}
\end{figure}

\begin{figure}[t]
    \centering
    \includegraphics[width=\textwidth]{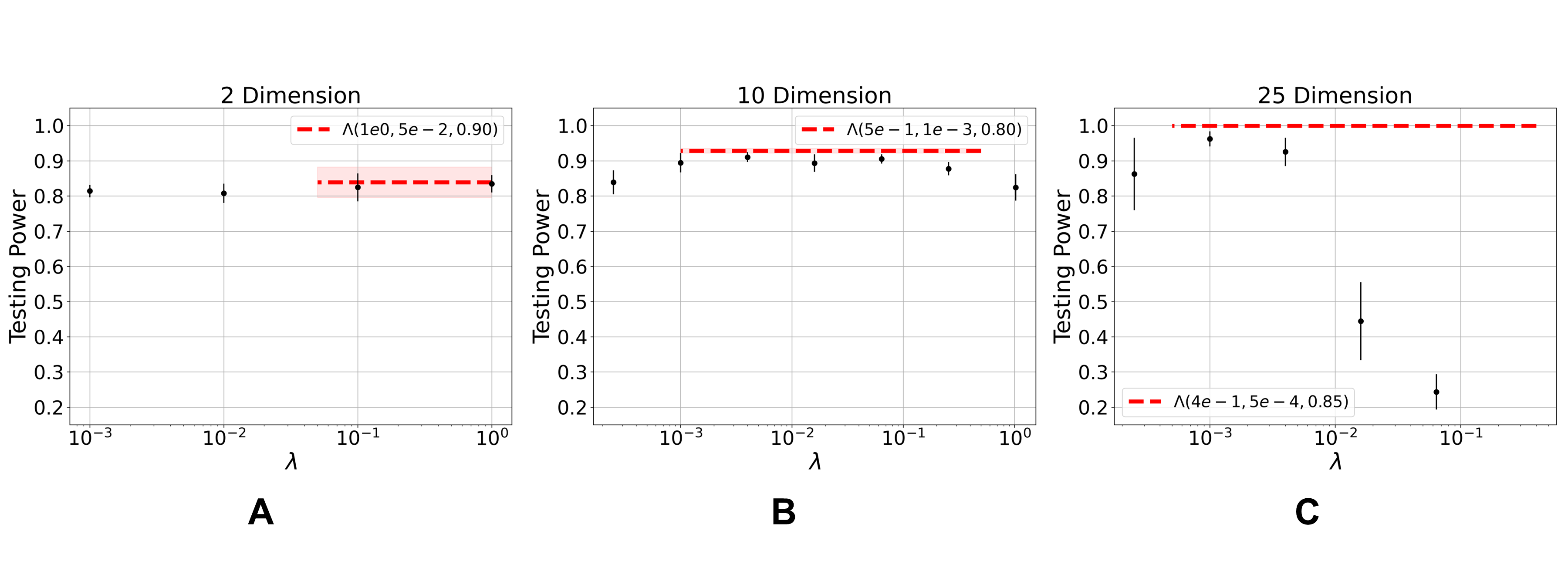}
    \vspace{-30pt}
    \caption{
    Testing power results for the Gaussian mixture model settings in 2 dimensions (\textbf{A}), in 10 dimensions (\textbf{B}), and in 25 dimensions (\textbf{C}). The mean power for fixed-$\lambda$ settings is plotted with standard deviation error bars, and the highest-average-power staged setting is plotted as a (mean) line with shaded standard deviation in each setting, where the dashed line spans the range of $\lambda$ for each staging strategy.}\label{fig:GM_power}
\end{figure}

\subsubsection{Results}\label{sec:gm_sim_result}

\textit{2-dimensional mixture.} The value of $\widehat{\rm MSE}_q$ computed via Equation \eqref{eq:mse} throughout training for all 2D critics can be seen in Figure \ref{fig:GM_MSE}(\textbf{A}), where the mean and standard deviation are plotted across the 10 networks for each staging strategy. We begin by examining the training behavior for fixed regularization weight $\lambda$. Rapid descent in $\widehat{\rm MSE}_q$ is observed for large $\lambda$ early in training while small $\lambda$ make slower and steadier progress. The staged regularization strategies exploit this rapid descent at early times, followed by steady late-stage training. However, note that the advantage of staging in 2 dimensions only yields a marginal benefit in MSE. In Figure \ref{fig:2d_critic}, we visualize the critic vector field plotted through training for the $\Lambda(1\times 10^0, 5\times 10^{-2}, 0.90)$ staging scheme. Rapid fitting of some regions can be seen at the beginning of training in Figure \ref{fig:2d_critic}(\textbf{A}), followed by more nuanced adjustments in Figures \ref{fig:2d_critic}(\textbf{B}) and \ref{fig:2d_critic}(\textbf{C}), resulting in a good fit to the theoretically optimal critic function in Figure \ref{fig:2d_critic}(\textbf{D}). Furthermore, consider the test power results displayed in Figure \ref{fig:GM_power}(\textbf{A}). While the staged regularization cases have among the lowest $\widehat{\rm MSE}_q$ values, {similar} testing power is achieved by the fixed-$\lambda$ settings. For more detail regarding the 2-dimensional experiment, see Table \ref{tab:2D_experiment} in Appendix \ref{app:supp_tables}, which shows the average GoF hypothesis testing power at the approximate ``best" training epoch as chosen by finding the average lowest {monitor {$\widehat{\rm MSE}_p^{(m)}$}} calculated using Equation \eqref{eq:validation_mse}. The table also displays the average $\widehat{\rm MSE}_q$ as calculated by Equation \eqref{eq:mse} at the ``best" epoch for each regularization scheme.
\vspace{5pt}

\noindent
\textit{10-dimensional mixture.} As in 2D, the $\widehat{\rm MSE}_q$ is plotted for critics over the course of training in Figure~\ref{fig:GM_MSE}(\textbf{B}). {The curves corresponding to some fixed $\lambda$ strategies, such as $\lambda=6.4\times10^{-2}$, $2.56\times10^{-1}$, and $1.024\times10^0$ are omitted since they quickly diverge in $\widehat{\rm MSE}_q$.} In higher dimensions, the difference in using a wider range of regularization weights $\lambda$ is more pronounced. Larger values result in networks that rapidly fit a (relatively) poor representation of the optimal critic {${\bf f}^*_\lambda$}, followed by dramatic overfitting. Smaller choices of fixed $\lambda$ delay both phenomena, fitting to a critic which has a lower value of $\widehat{\rm MSE}_q$ at best. The staged regularization strategies exploit both types of training dynamic, descending in $\widehat{\rm MSE}_q$ more rapidly than most fixed-$\lambda$ strategies while achieving lower value (and hence better fit) than any fixed strategy, on average. Examining the results in Figure \ref{fig:GM_power}(\textbf{B}), we find that the power of the $\Lambda(5\times 10^{-1}, 1\times 10^{-3}, 0.80)$ regularization strategy exceeds (on average) that of all the fixed-$\lambda$ strategies. More detail related to the networks obtained via validation can be found in Table \ref{tab:10D_experiment} in Appendix \ref{app:supp_tables}.
\vspace{5pt}

\noindent\textit{25-dimensional mixture.} As in the previous settings, the $\widehat{\rm MSE}_q$ computed via Equation \eqref{eq:mse} is plotted in Figure~\ref{fig:GM_MSE}(\textbf{C}) for the 25D regularization strategies. {As is the case in 10D, the $\widehat{\rm MSE}_q$ of some regularization strategies are relatively high and are therefore omitted from Figure~\ref{fig:GM_MSE}.} The observed trend of increasing the dimension from 2 to 10 is further exemplified by the increase to 25 dimensions. The larger choices of fixed $\lambda$ result in networks that quickly obtain a poor fit of the optimal critic, followed by overfitting. The smaller choices of fixed $\lambda$ yield more stable $\widehat{\rm MSE}_q$ curves. Combining these dynamics in our staging strategies, the $\widehat{\rm MSE}_q$ performance gap between fixed- and staged-$\lambda$ regularization dramatically increases, where the staged strategies substantially outperform the fixed strategies. Figure \ref{fig:GM_power}(\textbf{C}) further corroborates this finding. The GoF hypothesis test power of the staged critics is dramatically higher than any fixed-$\lambda$ training strategy. See Table \ref{tab:25D_experiment} in Appendix \ref{app:supp_tables} for more detail pertaining to trained networks in 25D.
\vspace{5pt}

On the simulated Gaussian mixture data, we find that staging the regularization of the neural Stein critic throughout training yields greater benefit as the dimension increases from 2 to 25. The {scaleless neural Stein} critic {$\lambda{\bf f}(\cdot,\theta)$} rapidly fits the scaleless optimal critic {${\bf f}^*$} \eqref{eq:def-fstar} at early times when $\lambda$ is large, followed by stable convergence to a low-$\widehat{\rm MSE}_q$ critic throughout training. Furthermore, we find that the GoF hypothesis testing power follows a similar trend as that of the $\widehat{\rm MSE}_q$, such that staging $\lambda$ yields an increase in power, especially in higher-dimension.

\subsection{Comparison to kernel Stein discrepancy}\label{sec:compare_to_ksd}

Using the simulated Gaussian mixture data as in Section \ref{sec:gm_sim}, we compare the testing power of the neural Stein critic GoF hypothesis test to that of KSD. To do so, we perform the GoF hypothesis test by computing the KSD test statistic outlined in \cite{liu2016kernelized,chwialkowski2016kernel}, whereby the Stein discrepancy is computed using a critic restricted to an RKHS. Following common practice in the literature, we construct this RKHS to be defined by a radial basis function (RBF) with a bandwidth equal to the median of the data Euclidean distances in a given GoF test. We also compare our method to a KSD test with RBF bandwidth which is selected to maximize the power. To compute the KSD test, we adopt the implementation of \cite{chwialkowski2016kernel}. Further details of the KSD GoF test, including selecting the most effective bandwidth, are given in Appendix~\ref{sec:ksd_gof_test}, with GoF testing power for a range of bandwidths displayed in Figure~\ref{fig:ksd_bandwidth_comparison}.

 It would be natural to ask how the neural network test compares with other traditional parametric and non-parametric tests.  
There are cases where the generalized likelihood ratio (GLR) test can be computed from parametric models and would be the optimal test. 
When a parametric model is not available, traditional non-parametric tests like Kolmogorov-Smirnov may be restricted to low-dimensional data. 
We thus focus on comparison with the KSD test which is a kernel-based non-parametric test generally applicable to high dimensional data. 

\subsubsection{Experimental setup}\label{sec:kernel_experiments}

Using the model distribution $q$ and data distribution $p$ as defined in Section \ref{sec:gm_sim_data}, we construct a two-component Gaussian mixture in 50D to compare the power and computation time of the neural Stein discrepancy GoF hypothesis test (Section \ref{sec:gof_test}) to the KSD test (Appendix \ref{sec:ksd_gof_test}). The model distribution $q$ remains as the isotropic, two-component Gaussian mixture, while the data distribution $p$ has the form of \eqref{eq:data_distribution_gm} with covariance shift $\rho_1=0.5$ and covariance scaling $\omega=0.5$.

\paragraph{Computation of test power.}
For all neural Stein critic GoF tests, a total of $n_{\rm tr}+n_{\rm GoF}=n_{\rm sample}$ samples from $p$ are used, where $n_{\rm sample}$ takes on values in $\{100, 200, 300, 500, 1000, 1500, 2500\}$. In each case, $n_{\rm tr}$ and $n_{\rm GoF}$ are equal to $n_{\rm sample}/2$. The 50\%/50\% training/testing split was chosen by comparing the testing power for various training/testing splits. Our findings indicate a range of training/testing splits exists for which the testing power performance is comparable. While the training duration contributes most significantly to the computation time, the test achieves high power once the training split reaches 50\% of the sample size. That is, the 50\%/50\% train/test split is a generic split choice in this range of high power splits, and therefore we use this split in the results outlined in Section~\ref{sec:ksd_comparison_results} and in Figure~\ref{fig:kernel_test_comparison}. The details of this finding can be found in Appendix~\ref{app:neural_split_power}, Figure~\ref{fig:neural_split_power}, and an application to a simpler setting is displayed in Figure~\ref{fig:ksd_comparison_easy}. 
Furthermore, the training data are partitioned into an 80\%/20\% train/validation split.
The testing procedure and test power computation are as described in Section~\ref{sec:gm_sim_setup}.
Specifically, the significance level $\alpha=0.05$ in all tests, the number of bootstraps $n_{\rm boot}=500$, and the efficient bootstrap ratio is $r_{\rm pool}=50$.

The KSD tests \cite{chwialkowski2016kernel} are computed over the same range of samples size. 
The KSD test is conducted with an RBF kernel using the median data distance heuristic for bandwidth, and we also examine the selected bandwidth chosen to maximize the test power.
To compute $t_{\rm thresh}$, the KSD test uses a ``wild bootstrap'' procedure \cite{chwialkowski2016kernel}.
We provide details of bandwidth selection and wild bootstrap in Appendix~\ref{sec:ksd_gof_test}.
The number of wild bootstrap samples used for KSD is equal to the number of bootstrap samples used for the neural Stein test, namely $n_{\rm boot} = 500$.

The test power is computed using $n_{\rm run}=400$ and $n_{\rm replica}=5$ for each KSD and neural Stein GoF test.

\paragraph{Comparison of test time.}
We compare the time taken to perform 
one KSD GoF test 
vs. 
training a neural Stein critic and performing a neural GoF test.
All reported times are elapsed time on a laptop with an Intel Core i7-1165G7 processor with 16 GB of RAM. 
We record and report the computation times of the tests over a range of sample sizes.

For the neural test, we measure the duration time of training followed by the computation time of a single test statistic and the computation time of the bootstrap, averaged over $n_{\rm replica}=5$ replicas with $n_{\rm run}=10$ and $n_{\rm boot}=500$.

To determine the computation time of the KSD test, we measure the duration of 10 tests using the best-selected bandwidth, which can be broken down into the time taken to compute an individual test statistic and the computation time of the wild bootstrap, averaged over $n_{\rm replica}=5$, $n_{\rm run}=10$, and $n_{\rm boot}=500$.

\paragraph{Neural network training.}
The neural Stein critics are learned using two-hidden-layer (512 nodes) MLPs with Swish activation, initialized using PyTorch standard initialization with biases set to 0. We use the Adam optimizer with a learning rate set to $5\times10^{-3}$, minibatch size equal to $40$, and the critics are trained for 60 epochs. The $\Lambda(4\times 10^{-1}, 5\times 10^{-4}, 0.85)$ staging is used in the training of the neural Stein critics, {where the frequency of $\lambda$ staging, $B_w$ batches, is again chosen to be equivalent to the number of training batches per epoch.} The networks are trained using $n_{\rm tr}$ number of training samples from the data distribution $p$, which are split into an 80\%/20\% train/validation split for model selection. As in Section \ref{sec:gm_sim_setup}, the networks are selected to minimize the {monitor $\widehat{\rm MSE}^{(m)}_p$} of Equation \eqref{eq:validation_mse}.

\begin{figure}[t]
    \centering
    \includegraphics[width=\textwidth]{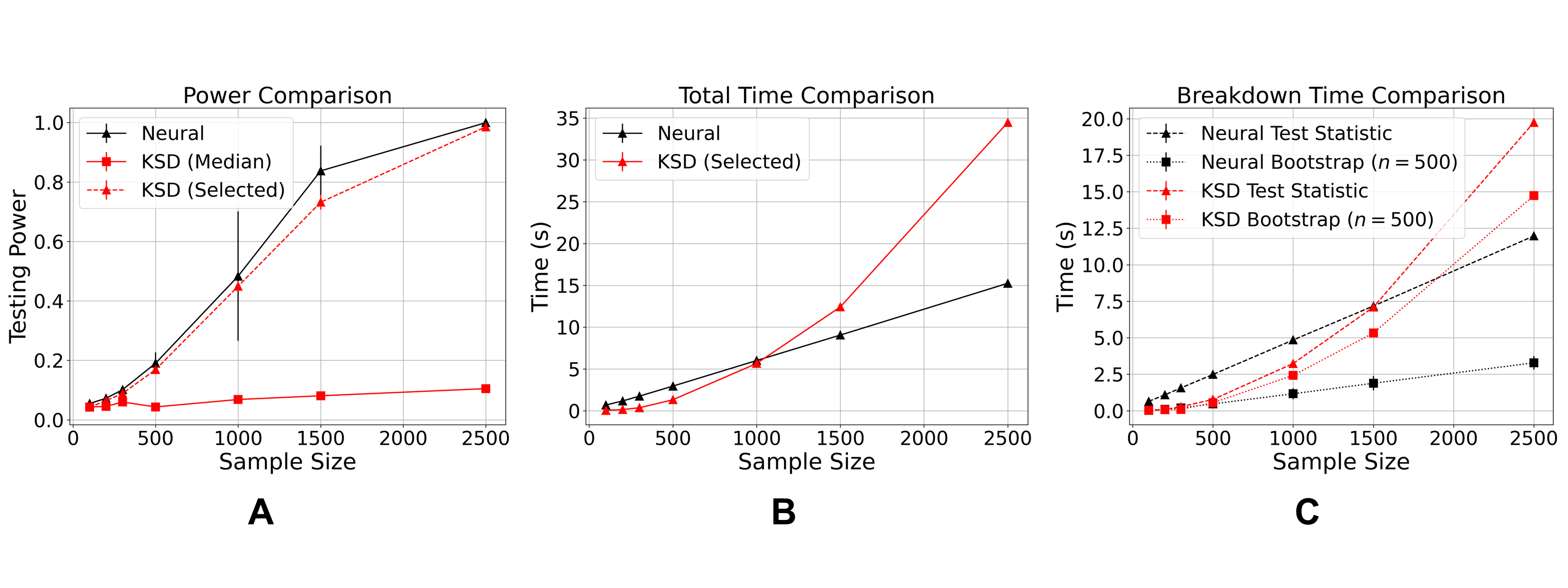}
    \vspace{-20pt}
    \caption{
    Comparison between neural Stein critic GoF test and KSD test \cite{chwialkowski2016kernel} in the 50D setting described in \ref{sec:kernel_experiments}. We perform $n_{\rm run}=400$ number of GoF tests to compute power {with $n_{\rm replica}=5$ replicas}. All results use $n_{\rm boot}=500$ bootstrap samples. The KSD is computed using an RBF kernel with median data distance heuristic bandwidth and bandwidth selected to maximize power (see Appendix~\ref{sec:ksd_gof_test} for details). In each plot, the x-axis indicates the overall sample size $n_{\rm sample}$ provided to each method, where the neural Stein GoF test uses half for training and a half for computation of the test statistic (refer to Figure~\ref{fig:neural_split_power} for a comparison of neural Stein tests with other choices of train/test split), and KSD uses the entire sample for computation of the test statistic. In (\textbf{A}), the average test power and its standard deviation error bar are shown, computed over the 5 models in each case. Figure (\textbf{B}) shows the average duration of network training plus an individual GoF test (computation of test statistic and bootstrap) for the neural Stein test and the average duration of a single test (test statistic and bootstrap) for the KSD test. Finally, (\textbf{C}) shows the breakdown between the average computation time required to compute a test statistic (including training for the neural network test) and the time required to compute the bootstrap. See Section \ref{sec:kernel_experiments} for details of the GoF test power and recording of computation time.}\label{fig:kernel_test_comparison}
\end{figure}

\subsubsection{ Results}\label{sec:ksd_comparison_results}

The result of the testing power comparison on the 50D data can be found in Figure~\ref{fig:kernel_test_comparison}. First, examining the comparison in power between the methods in Figure~\ref{fig:kernel_test_comparison}(\textbf{A}), the neural Stein critic GoF test (solid black line) remains more powerful than the KSD GoF tests for all sample sizes. The KSD GoF test using the median data distances heuristic bandwidth (solid red line) does not achieve comparable power to the staged-regularization neural Stein critic approach, even for as many as 2500 samples. While the KSD test using the best-selected bandwidth (dashed red line) achieves much higher power than the heuristic approach, the power of this method still falls below that of the neural Stein GoF test. Next, Figure~\ref{fig:kernel_test_comparison}(\textbf{B}) highlights the quadratic time complexity of the KSD GoF test, resulting in a dramatic increase in computation time that surpasses the overall time to train a neural Stein critic and compute a test for samples sizes larger than 1,000.
Furthermore, the breakdown into test statistic computation time and bootstrap computation time in Figure~\ref{fig:kernel_test_comparison}(\textbf{C}) reveals that the time to compute an individual test statistic for the neural Stein approach becomes more time-efficient than KSD when $n_{\rm sample}$ is just larger than 1500. While the test statistic computation and wild bootstrap of KSD contribute substantially to the overall computation time of KSD, the neural Stein GoF test computation time is dominated by the training period.

Our findings indicate that the neural Stein critic clearly outperforms KSD in this setting for a larger sample size, with higher power and lower computation time. This result highlights the deeper expressivity of the $L^2$ function space compared to kernel methods in addition to the benefit of the linear time complexity of the neural Stein critic GoF test, as opposed to the quadratic complexity of KSD.

\subsection{MNIST {handwritten} digits data}\label{sec:mnist}

The results of Section \ref{sec:gm_sim} indicate that the staging of regularization when training neural Stein critics yields greater benefit in higher dimensions. Therefore, we extend to a real-data example in an even higher dimension: the MNIST handwritten digits dataset \cite{lecun2010mnist}. We compare the fixed and staged regularization strategies to train critics that discriminate between an RBM and a mixture model of MNIST digits.

\subsubsection{Authentic and synthetic MNIST data}

To construct the model distribution $q$ for the MNIST setting, we follow the approach of \cite{grathwohl2020learning}. We use a Gaussian-Bernoulli RBM that models the MNIST data distribution, trained using a learned neural Stein critic to minimize the Stein discrepancy between the true MNIST density and the RBM. We declare the model distribution $q$ to be this 728-dimensional RBM. The data distribution $p$ is a mixture model composed of 97\% the RBM and 3\% true digits ``1" from the MNIST dataset. Therefore, any disparity in the distributions of $p$ and $q$ are caused by this infusion of digits from MNIST into $p$.

\subsubsection{Learned neural Stein critics}

In addition to being a realistic setting, training a neural Stein critic using MNIST digits will allow us to better interpret the discrepancy. Of course, since we do not have access to the ``true" score function $p$, we cannot calculate the $\widehat{\rm MSE}_q$ for the trained critics using Equation \eqref{eq:mse}. Furthermore, the {computation of $\widehat{\rm MSE}^{(m)}_p$} using validation data from $p$ via \eqref{eq:validation_mse} becomes less accurate in high dimension, as the method would require a large amount of validation data to be an accurate representation of the population ${\rm MSE}_p$.
Therefore, we introduce an additional validation metric to evaluate the fit of the {neural Stein} critic {${\bf f}(\cdot,\theta)$}. First, in the language of the GoF test introduced in Section \ref{sec:gof_test}, we denote the quantity computed by applying the Stein operator with respect to $q$ on {neural Stein} critic {${\bf f}(\cdot,\theta)$} evaluated at a sample $x\in\calX$ as the ``critic witness" of the sample $x$:
\begin{equation}
    w(x{,\theta}) = T_q {\bf f}(x{,\theta}).
    \label{eq:witness}
\end{equation}
{Intuitively,
since {${\bf f}(\cdot,\theta)$} is trained on samples from $p$ (the data distribution) by maximizing the Stein discrepancy,
the value of $w(x{,\theta})$ represents the magnitude of the difference between distributions $p$ and $q$ at $x\in\calX$.} Evaluating the critic witness at ${n_{\rm GoF}}$ samples $x_i\sim q$, under the central limit theorem (CLT) assumption, random variables $w(x_i{,\theta})$ have a (centered) normal distribution with standard deviation $\sigma(w)/\sqrt{{n_{\rm GoF}}}$ when ${n_{\rm GoF}}$ is large, where $\sigma(w)$ is the standard deviation of $w(x_i{,\theta})$.

Note that the test statistic \eqref{eq:test_stat} is the mean of $w(x_i{,\theta})$ computed over testing data $x_i\sim p$. As an assessment for the GoF testing power for the {neural Stein} critic function ${\bf f}{(\cdot,\theta)}$ (which is expensive to compute in such a high dimension), we may compare the mean and variance of $w(\cdot{,\theta})$ applied to a testing dataset sampled from $p$ and to a ``null" dataset drawn from $q$, both of size ${n_{\rm GoF}}$:
\begin{equation}
    \hat{P} = \frac{{\bar w}_p}{\sigma(w_p) + \sigma(w_q)}, \quad {\bar w}_p = \frac{1}{{n_{\rm GoF}}}\sum_{i=1}^{{n_{\rm GoF}}} w(x_i{,\theta}), \quad \sigma(w_p) = \frac{1}{{n_{\rm GoF}}}\sqrt{\sum_{i=1}^{{n_{\rm GoF}}} (w(x_i{,\theta})-{\bar w}_p)^2}
    \label{eq:power_proxy}
\end{equation}
This quantity acts to reflect the capability of the neural Stein critic to differentiate between the distributions in the GoF hypothesis testing setting described in Section \ref{sec:gof_test}. In addition to the $\hat{P}$ metric from Equation \eqref{eq:power_proxy}, we may also apply the Stein discrepancy evaluated at the scaleless neural Stein critic, {i.e., the ${\rm SD}[\lambda{\bf f}(\cdot,\theta)]$ \eqref{eq:sd_operator}}, to the holdout dataset from the data distribution as an evaluation metric for the models as described in Section \ref{sec:eval_generative_models}.

\subsubsection{Experimental setup}

We again train 2-hidden-layer MLP's with Swish activation, where each hidden layer is composed of 512 hidden units, using the default Adam optimizer parameters by PyTorch. The critics observe 2,000 training samples from $p$, training on mini-batches of size 100 with a learning rate $10^{-3}$. Each model is trained for 25 epochs. We consider fixed $\lambda\in\{1\times 10^{-3}, 1\times 10^{-2}, 1\times 10^{-1}, 1\times 10^0, 2\times 10^0\}$ and staging scheme $\Lambda(5\times 10^{-1}, 1\times 10^{-3}, 0.90)$. {The frequency of updates via the staging strategies is $B_w=20$ batches.} We fit 10 critics per regularization strategy. For each critic, we compute the validation SD and the power metric \eqref{eq:power_proxy} throughout training using ${n_{\rm GoF}}=1,000$ samples from $p$ and the same number of ``null" samples from $q$.

In addition to assessing the proxy for the test power of the neural Stein critic in the GoF test via Equation \eqref{eq:power_proxy}, we examine the interpretability of the critic as a diagnostic tool for anomalous observations. By Equation \eqref{eq:def-fstar}, the {scaleless} optimal critic captures the difference in the score of the data distribution and model distribution. Therefore, a trained neural Stein critic can indicate which samples in a validation dataset represent the largest departure from the distribution $q$. We isolate such samples by identifying samples with high critic witness value \eqref{eq:witness}. We do so by both visualizing the images in a holdout validation dataset sampled from $p$ which have a high $w(\cdot{,\theta})$ value, in addition to plotting a heatmap of $w(\cdot{,\theta})$ reduced using a t-SNE embedding \cite{van2008visualizing} applied to the entire validation dataset.

\begin{figure}[t]
    \centering
    \includegraphics[width=\textwidth]{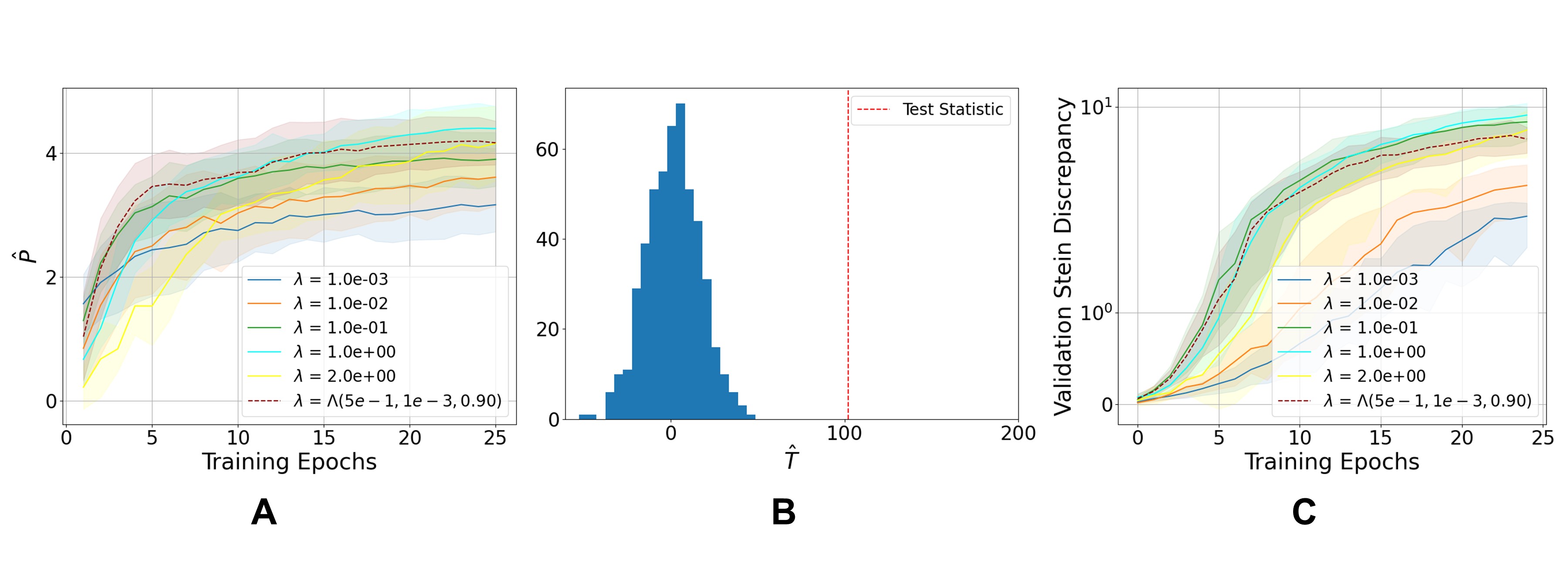}\vspace{-20pt}
    \caption{In (\textbf{A}), we visualize the $\hat{P}$ metric from Equation \eqref{eq:power_proxy} throughout training for various regularization strategies when training neural Stein critics using the MNIST mixture dataset. To compute $\hat{P}$, the model is applied to a validation dataset of 1,000 samples. In each case, 10 models are trained for 25 epochs using each regularization strategy; the mean and standard deviation are visualized over these 10 models. In (\textbf{B}), a distribution of null statistics \eqref{eq:test_stat} are plotted alongside a statistic calculated over an ${n_{\rm GoF}}=100$ sample testing set from $p$, computed using a $\Lambda(5\times 10^{-1}, 1\times 10^{-3}, 0.90)$ staged critic. In (\textbf{C}), the Stein discrepancy evaluated at the scaleless neural Stein critic $\lambda{\bf f}(\cdot,\theta)$ through training is visualized, which is computed using the same validation datasets of 1,000 samples used to compute $\hat{P}$ in (\textbf{A}).}\label{fig:mnist_power}
\end{figure}

\subsubsection{Results}

The power approximation using Equation \eqref{eq:power_proxy} is plotted in Figure \ref{fig:mnist_power}(\textbf{A}) for each regularization strategy, where the means and standard deviations are calculated using the ten models for each regularization scheme. While the staging strategy does not exhibit such an advantage as in the high-dimension Gaussian mixture data, staging provides a more rapid increase in the validation metric in the early training period, yielding a final model of comparable performance to the fixed-$\lambda$ strategies. In Figure \ref{fig:mnist_power}(\textbf{B}), we observe that the test statistic \eqref{eq:test_stat} exhibits clear separation from its bootstrapped ($n_{\rm boot}=500$) null distribution, even in the case when the number of test samples is relatively small (in this case, ${n_{\rm GoF}}=100$). While Figure~\ref{fig:mnist_power}(\textbf{B}) shows this distribution for the $\Lambda(5\times 10^{-1}, 1\times 10^{-3}, 0.90)$ staged regularization strategy, this holds for fixed-$\lambda$ training as well. {Finally, the Stein discrepancy evaluated at the scaleless neural Stein critics {$\lambda{\bf f}(\cdot,\theta)$} applied to the holdout dataset from $p$ is visualized throughout training for each model in Figure \ref{fig:mnist_power}(\textbf{C}). This result is similar to the result from Figure \ref{fig:mnist_power}(\textbf{A}) in that the staged approach performs comparably to the best fixed-$\lambda$ regularization strategies.}

\begin{figure}[t!]
    \centering
    \includegraphics[width=\textwidth]{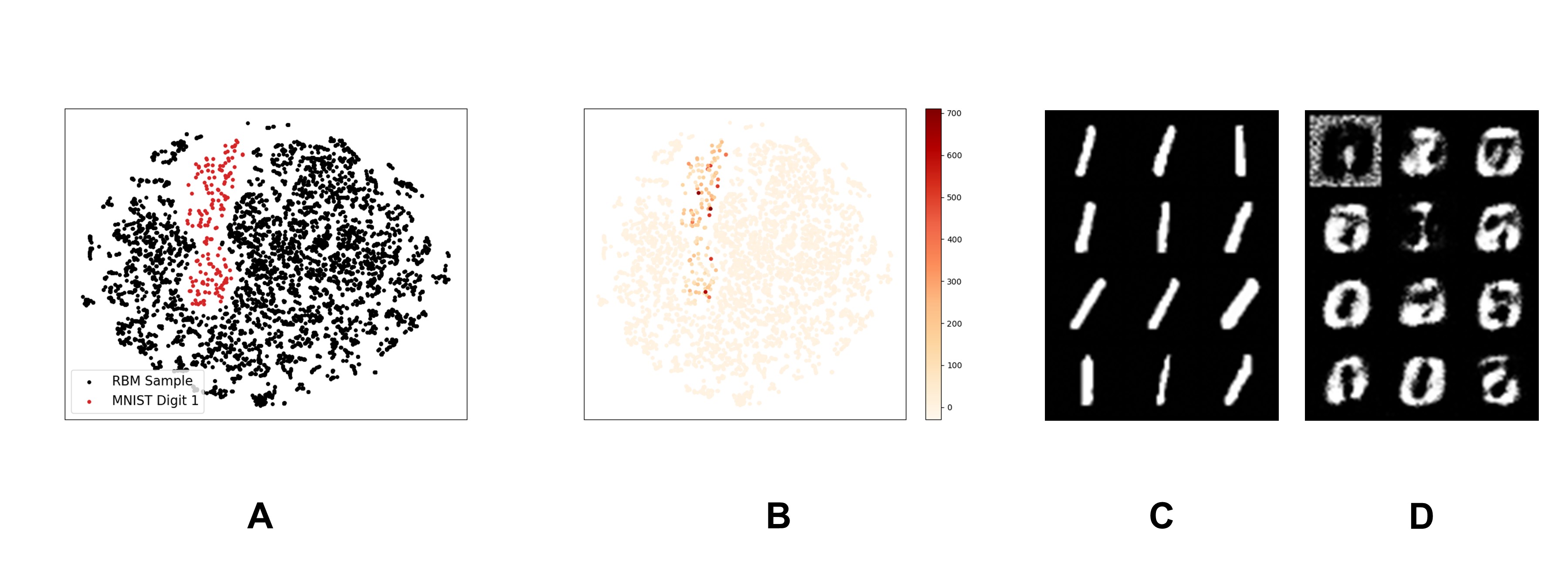}
        \vspace{-30pt}
    \caption{Embedding via t-SNE of the critic witness function \eqref{eq:witness} scaled by the regularization weight $\lambda$ using a $\Lambda(5\times 10^{-1}, 1\times 10^{-3}, 0.90)$ staged, regularized neural Stein critic applied to a validation dataset of 6,000 samples from the data distribution $p$ used to evaluate the diagnostic capacity of the neural Stein critic. After 25 epochs of training, the critic was selected to produce these critic witness values. In (\textbf{A}), the black points represent the portion of the validation set coming from the RBM, while the red points represent the portion that is true digits 1 from MNIST. In (\textbf{B}), the points with high critic witness value are more darkly colored. In (\textbf{C}), the 12 images in the validation set with the highest critic witness value are shown, applying the critic trained by the staging strategy. Similarly, (\textbf{D}) shows the 12 images in the validation set with the lowest critic witness value.}\label{fig:tsne}
\end{figure}

For a more direct understanding of how the {neural Stein} critics perform, consider instances of validation data sampled from $p$ in which the critic witness \eqref{eq:witness} value is very large. These samples indicate the largest deviation between the model distribution $q$ and the data distribution $p$. In the staged $\Lambda(5\times 10^{-1}, 1\times 10^{-3}, 0.90)$ setting, we visualize the critic witness applied to a validation set of 6,000 samples from $p$ by reducing the images to a two-dimensional embedding via t-SNE. In Figure \ref{fig:tsne}(\textbf{A}), we observe the embedding of the validation data into this space, where the true MNIST points are highlighted in red. In Figure \ref{fig:tsne}(\textbf{B}), the true MNIST digits are found to have a larger critic witness value than those of RBM samples in the validation dataset.

Furthermore, in the $\Lambda(5\times 10^{-1}, 1\times 10^{-3}, 0.90)$ case, visualizing the images in the validation set from $p$ which have the highest critic witness value in Figure \ref{fig:tsne}(\textbf{C}) and those which have the lowest critic witness value in Figure \ref{fig:tsne}(\textbf{D}), we find that this approach correctly identifies true digits one from MNIST as anomalous while accepting those generated by the RBM as normal. In the case of the fixed-$\lambda$ regularization strategies, it seems that all methods do well to identify the true digits ``1" in samples from the data distribution $p$.

\section{Discussion}

We have developed a novel training approach for learning neural Stein critics by starting with strong $L^2$ regularization and progressively decreasing the regularization weight over the course of training. 
The advantage of staged regularization is empirically observed in experiments, especially for high-dimension data. In all the experiments, it is observed that critics trained using larger regularization weights at the beginning of training enjoy a more rapid approximation of the target function, 
including in the task of detecting distribution differences between authentic and synthetic MNIST data.
%In the task of detecting distribution differences between authentic and synthetic MNIST data, the staged regularization training strategy makes more rapid progress at early times in training than the fixed-$\lambda$ methods, achieving comparable performance by the end of training. 
These experimental findings are consistent with the NTK lazy-training phenomenon that happens with large regularization strength at the beginning phase of training, which we theoretically prove. 
We apply the neural Stein method to GoF tests for which we derive a theoretical guarantee of test power.

The work can be extended in several future directions.
Theoretically, within the lazy-training framework,
it would be interesting to characterize the expressiveness of the NTK kernel (the assumption of Proposition \ref{prop:NTK-stein}).
One can also explore a more advanced analysis of the neural network training dynamic, including the SGD training using mini-batches and going beyond kernel learning.
A related question is to theoretically analyze the later training stage, for example, to prove the convergence guarantee with decreasing $\lambda$,
and to find the theoretically optimal annealing strategy of $\lambda$ in later stages. 
A possible way is to track the change of the NTK kernel as time evolves, which remains a challenge in studying NTK theory. 
For the algorithm, one can further investigate the optimal staging scheme and explore other types of regularization of the neural Stein critic than the $L^2$-regularization.
For the application to GoF testing, here we consider the typical setting for the Goodness-of-Fit test, which is a simple null hypothesis, i.e., we test whether or not the data follows a particular distribution/model; a potentially interesting direction is extending to cases where the null hypothesis is composite, i.e., comparing the data to a set of $K$ target distributions. An interesting strategy would be training $K$ neural networks jointly with possible sharing parameters so as to reduce model size and computation. 
Finally, further applications can be conducted on other modern generative modeling approaches in the same manner as the Gaussian-Bernoulli RBM, including normalizing flow architectures.

\section{Proofs}\label{sec:proofs}

\subsection{Proofs in Sections \ref{subsec:theory-GD}-\ref{subsec:lazy-training-approx}}

\begin{proof}[Proof of Lemma \ref{lemma:eigen}]
Introduce the notation 
\[
\left[ {\bf K}_0(x,x') \right]_{ij } = K( (x,i), (x',j)),
\]
where $K( z, z')$ is a positive semi-definite (PSD) kernel defined on the space of $\calZ = \calX \times [d]$, that is,
\[ 
z = (x,i),  \quad x \in\calX, \quad i\in [d],
\]
with $dz$ being the induced product measure.
The kernel $K$ is PSD due to the definition \eqref{eq:def-NTK-0}. One can also verify that  $K$ is Hilbert-Schmidt because
\begin{align}
& \int_\calZ \int_\calZ K(z,z')^2 dz dz' 
 = \sum_{i,,j=1}^d  \int_\calX \int_\calX  \left[ {\bf K}_0(x,x') \right]_{ij }^2 p(x) p(x') dx dx'  \nonumber \\
&~~~ 
= \sum_{i,,j=1}^d  \int_\calX \int_\calX  
	\langle \partial_\theta  f_i( x, \theta(0) ),  \partial_\theta  f_j(x',\theta (0)) \rangle_\Theta^2 
	p(x) p(x') dx dx' \nonumber \\
&~~~ 
\le  \left ( \sum_{i=1}^d  
	\int_\calX  \| \partial_\theta  f_i( x, \theta(0) )\|_{\Theta}^2  p(x) dx  \right)^2 < \infty,
	\quad \text{(by Cauchy-Schwarz)}
\end{align}
and the integrability follows by the $L^2$ integrability of $ \partial_\theta  f_i( x, \theta(0) )$ in $(\calX, p(x) dx)$ assumed in the condition of the lemma. 

As a result, the spectral theorem implies that $K(z,z')$ has discrete spectrum $\mu_k$, $k=1,2,\cdots$ which decreases to 0,
each $\mu_k$ is associated with an eigenfunction $v_k( z)$, and $\{ { v}_k\}_{k=1}^\infty$ form an orthonormal basis on $\calZ$.  
This means that 
\[
\sum_{j=1}^d \int_\calX \left[ {\bf K}_0(x,x') \right]_{ij }  v_k( x',j) p(x') dx' = \mu_k v_k(x,i), \quad \forall x\in\calX, \, i \in [d],
\]
and  for any $k, l = 1, 2, \cdots$, 
\[
\sum_{i=1}^d \int_\calX v_k(x, i) v_l(x, i) p(x) dx  = \delta_{kl}.
\]
At last, because the neural network has a finite width, 
$\Theta$ is in a finite-dimensional Euclidean space of dimensionality $M_\Theta$. By \eqref{eq:def-NTK-0}, the kernel $K$ has a finite rank at most $M_\Theta$. 
Let the rank of $K$ be $M$, then $\mu_1 \ge \cdots \ge \mu_M > 0$, and the other eigenvalues are zero.
Defining ${\bf v}_k (x)$ by $( v_k(x, i) )_{i=1}^d$ finishes the proof. 
\end{proof}

\begin{proof}[Proof of Proposition \ref{prop:NTK-stein}]
As shown in the proof of Lemma \ref{lemma:eigen}, there are ortho-normal basis $\{ {\bf v}_k \}_{k=1}^\infty$ of $\calX$ with respect to $\langle \cdot, \cdot \rangle_p$, where the first $M$ many consist of eigen-functions of $ {\bf K}_0(x,x') $. 
{Using the ortho-normal basis,  ${\bf f}^*$ has the following expansion with coefficients $c_k \in \R$,}
%Because ${\bf f}^*$ is squared integrable on $(\calX, p(x) dx)$, we have that for $c_k \in \R$,
\begin{equation}\label{eq:f-vk}
{\bf f}^* = \sum_{k=1}^\infty c_k {\bf v}_k, 
\quad
 \|  {\bf f}^* \|^2_{p} = \sum_{k=1}^\infty c_k^2 < \infty.
\end{equation}
The uniqueness of orthogonal decomposition gives that 
\begin{equation}\label{eq:f1-f2-vk}
{\bf f}_1^* = \sum_{k=1}^m c_k {\bf v}_k, \quad
{\bf f}_2^* = \sum_{k=m+1}^\infty c_k {\bf v}_k.
\end{equation}
To prove the proposition, it suffices to prove \eqref{eq:bound-NTK}, and then \eqref{eq:2eps-bound} follows by that 
\[
\| {\bf f}_1^* \|_{p }   \le \| {\bf f}^* \|_{p } 
\]
which follows from \eqref{eq:f1-f2-vk}, and that  $e^{-t\delta \lambda} \le \epsilon$ when $ t  \ge  {\log (1/\epsilon)}/{(\delta \lambda)}$.

To prove \eqref{eq:bound-NTK}: By \eqref{eq:evolution-baru}, and define
\[
{\bf w} (x, t) := \lambda \bar{\bf u }( x, t) -  {\bf f}^*(x), 
\]
we have
 \begin{equation}\label{eq:parital-w}
{\partial_t} {\bf w} (x, t)
 =  - \lambda \int_\calX  {\bf K}_0(x,x')   \circ  {\bf w} (x', t)  p(x') dx',
\end{equation}
and, by that $\bar{\bf u }( x, 0) =0 $, 
we have
\[
{\bf w} (x, 0) = -  {\bf f}^*(x).
\]
Because ${\bf w} (x, 0) \in L^2(p)$, the evolution equation  \eqref{eq:parital-w} implies that
\[
{\bf w} (x, t) = \sum_{k=1}^\infty b_k(t)  {\bf v}_k(x),
\quad
\dot{b}_k(t) = \begin{cases}
- \lambda \mu_k b_k(t), & k \le M \\
0, & k > M
\end{cases},
\quad 
b_k(0) = -c_k.
\]
Using the notation $\mu_l = 0$ for $l > M$, we have
\[
b_k(t) = -c_k e^{- t \lambda \mu_k}, \quad k=1,2,\cdots
\]
This gives that that for any $t > 0$,
\[
- {\bf w} (x, t) 
=  \sum_{k=1}^\infty e^{- t \lambda \mu_k} c_k {\bf v}_k(x)
= \sum_{k=1}^m e^{- t \lambda \mu_k} c_k {\bf v}_k(x) + 
	\sum_{k=m+1}^\infty e^{- t \lambda \mu_k} c_k {\bf v}_k(x)
=: \textcircled{1} + 	\textcircled{2}.
\]
Because $\mu_k \ge \delta$ for $k=1,\cdots, m$, 
\[
\| \textcircled{1} \|_{p}^2 = \sum_{k=1}^m e^{- 2 t \lambda \mu_k} c_k^2
\le  e^{- 2 t \lambda \delta}  \sum_{k=1}^m  c_k^2
= e^{- 2 t \lambda \delta}  \| {\bf f}_1^* \|_{p}^2,
\]
where the last equality is by \eqref{eq:f1-f2-vk}.
In addition,
\[
\| \textcircled{2} \|_{p}^2 = \sum_{k=m+1}^\infty e^{- 2 t \lambda \mu_k} c_k^2
\le \sum_{k=m+1}^\infty c_k^2 
 {=}  \| {\bf f}_2^* \|_{p}^2.
\]
Putting together, this gives that 
\[
\| {\bf w} (x, t)  \|_{p}
\le \| \textcircled{1} \|_{p} + \| \textcircled{2} \|_{p}
\le e^{- t \lambda \delta}  \| {\bf f}_1^* \|_{p}+ \| {\bf f}_2^* \|_{p},
\]
which proves \eqref{eq:bound-NTK}.
\end{proof}

%\subsection{Proofs in Section \ref{subsec:lazy-training-approx}}

%
\begin{proof}[Proof of Proposition \ref{prop:NTK-approx}]
\vspace{5pt}
\noindent
\underline{Proof of (i)}:
Recall that $\dot{\theta}(t) = - {\partial_\theta} L_\lambda( \theta (t ))$ as in \eqref{eq:def-GD-theta},
and then by the chain rule, 
\[
\frac{d}{dt} L_\lambda( \theta (t )) = 
\langle {\partial_\theta} L_\lambda( \theta (t )),  \dot{\theta}(t) \rangle_\Theta
= - \| {\partial_\theta} L_\lambda( \theta (t )) \|_\Theta^2 \le 0,
\]
which means that $L_\lambda( \theta (t )) $ is monotonically decreasing over time. Together with  \eqref{eq:Llambda-is-MSE-const}, we have that $\forall t \ge 0$, 
\begin{align}
 \int_0^t  \|  \dot{\theta}(s) \|_\Theta^2  ds 
& =  \int_0^t  \| {\partial_\theta} L_\lambda( \theta (s )) \|_\Theta^2  ds
 = L_\lambda( \theta (0 )) - L_\lambda( \theta (t )) 
= \calL_\lambda[ {\bf u}( \cdot, 0 ) ] -  \calL_\lambda[ {\bf u}( \cdot, t ) ]  \nonumber \\
& =   \frac{1}{2 \lambda } \left(  \|  \lambda  {\bf u}( \cdot, 0 ) - {\bf f}^*   \|_p^2  - \|  \lambda  {\bf u}( \cdot, t ) - {\bf f}^*   \|_p^2    \right) 
\le  \frac{1}{2 \lambda }  \|  \lambda  {\bf u}( \cdot, 0 ) - {\bf f}^*   \|_p^2   
= \frac{1}{2 \lambda }  \|  {\bf f}^*   \|_p^2.   \label{eq:bound-thetadot-square-int}
\end{align}
In the last equality, we have used the condition that ${\bf u}( x, 0 ) = 0$.
\eqref{eq:bound-thetadot-square-int} bounds the change of $\theta(t)$ by Cauchy-Schwarz, that is,
\[
\|\theta(t) - \theta(0)\| 
\le    \int_0^t \| \dot{\theta}(s) \| ds
\le \sqrt{t}  \left( \int_0^t \| \dot{\theta}(s) \|^2 ds \right)^{1/2}
\le \sqrt{t}  \left(  \frac{1}{2 \lambda }  \|  {\bf f}^*   \|_p^2 \right)^{1/2},
\]
which proves \eqref{eq:thetat-theta0}.
Finally, by that $ \| \theta(0)\| < r/2$, 
we have $\| \theta(t)\| < r/2 + \sqrt{ {t}/{ (2\lambda)} } \| {\bf f}^*\|_p$, 
and then the upper bound of $t$ in the condition of (i) ensures that $\theta(t) \in \Theta$.
Thus, for all such $t$, ${\bf u} (\cdot, t)  =   {\bf f}(\cdot, \theta(t))$ is well-defined 
and  ${\bf f}(\cdot, \theta(t)) \in L^2(p)$ by Assumption \ref{assump:C1-C2},
so that the involved integrals in the loss  $L_\lambda( \theta (t ))$ are all finite.

\vspace{5pt}
\noindent
\underline{Proof of (ii)}:
We first note that
 \begin{equation}\label{eq:thetat-in-Br}
 \theta(t) \in B_r, \quad \forall 0 \le t \le  \frac{1}{2}( \frac{r}{ \| {\bf f}^*\|_p } )^2 \lambda.
  \end{equation}
  This is because $r/2+ \sqrt{ {t}/{ (2\lambda)} } \| {\bf f}^*\|_p \le r$ and $B_r \subset \Theta$, 
  thus the condition needed for $t$ by (i) is satisfied, 
  and then \eqref{eq:thetat-theta0} gives that  $\| \theta(t) - \theta(0)\| \le r/2$ for the range of $t$ being considered.
The claim \eqref{eq:thetat-in-Br} then follows together with the assumption that $\| \theta(0) \| < r/2$.

In the below, we may omit the dependence on $t$ in notation when there is no confusion, e.g., we write $\lambda {\bf u}( x' ,t) -   {\bf f}^*(x') $ as $(\lambda {\bf u}-   {\bf f}^*)(x')  $.
To prove \eqref{eq:y-bary},
recall the evolution equations of  ${\bf u}$ and   $\bar{\bf u}$ as in \eqref{eq:evolution-u} and \eqref{eq:evolution-baru},
\begin{equation}\label{eq:pde-u-ubar}
\begin{cases}
 {\partial_t}{\bf u }( x, t) 
 =  - \E_{x' \sim p}  {\bf K}_t(x,x')   \circ \left( \lambda {\bf u} - {\bf f}^* \right)( x')  \\
 {\partial_t}\bar{\bf u }( x, t) 
 =  - \E_{x' \sim p}
 {\bf K}_0(x,x')   \circ 	\left( \lambda \bar{\bf u} - {\bf f}^*  \right) (x') \\
 {\bf u}( x ,0)   = \bar{\bf u}( x ,0) = 0
\end{cases}
\end{equation}
which gives that $( {\bf u} - \bar{\bf u} )|_{t=0} = 0$ and
\begin{align}
  \partial_t (   {\bf u} -  \bar{\bf u} )
& =  - \E_{x' \sim p}  ({\bf K}_t - {\bf K}_0)(x,x')   \circ \left( \lambda {\bf u} - {\bf f}^* \right)( x')
 - \lambda \E_{x' \sim p} {\bf K}_0(x,x')   \circ 	\left(  {\bf u}  -  \bar{\bf u}   \right) (x').
\end{align}
By taking $\langle \cdot, \cdot \rangle_p$ inner-product  with $ {\bf u} -  \bar{\bf u}$ on both sides,
and that ${\bf K}_0$  has the expression \eqref{eq:def-NTK-0} and thus is PSD, we have that
\begin{align}
\frac{d}{dt} \frac{1}{2} \|  {\bf u} -  \bar{\bf u} \|_p^2
&  \le  \langle  {\bf u} -  \bar{\bf u} ,
   - \E_{x' \sim p}  ({\bf K}_t - {\bf K}_0)(\cdot, x')   \circ \left( \lambda {\bf u} - {\bf f}^* \right)( x') \rangle_p \nonumber \\
& \le   \| {\bf u} -  \bar{\bf u}  \|_p
	\|  {\bf K}_t - {\bf K}_0 \|_{p}  \|  \lambda {\bf u} - {\bf f}^* \|_p, \label{eq:dt-delta2-1}
\end{align}
where $ \|  {\bf K}_t - {\bf K}_0 \|_{p}$ stands for the operator norm of the kernel integral operator in $L^2(p)$. 
We claim that
 \begin{align} 
&\text{(Claim 1)} \quad  \|  \lambda {\bf u}(\cdot, t) - {\bf f}^* \|_p \le \|{\bf f}^* \|_p,  \quad \forall t \ge 0,
	\label{eq:claim1-pf} \\
& \text{(Claim 2)} \quad  \|  {\bf K}_t - {\bf K}_0 \|_{p}  \le L_1 L_2  \sqrt{ \frac{ 2 t}{  \lambda} } \|{\bf f}^* \|_p, 
\quad \forall t  \le  \frac{1}{2}( \frac{r}{ \| {\bf f}^*\|_p } )^2 \lambda. 
\label{eq:claim2-pf}
 \end{align}
 If both claims  hold, then \eqref{eq:dt-delta2-1} continues as
 \begin{equation}
 \frac{d}{dt} \frac{1}{2} \|  {\bf u} -  \bar{\bf u} \|_p^2
 \le  L_1 L_2  \sqrt{ \frac{ 2 t}{  \lambda} } \|{\bf f}^* \|_p^2  \| {\bf u} -  \bar{\bf u}  \|_p,
 \quad \forall t  \le  \frac{1}{2}( \frac{r}{ \| {\bf f}^*\|_p } )^2 \lambda. 
 \label{eq:dt-delta2-2}
  \end{equation}
 We define 
 \[
 \Delta(t): = \|  \lambda {\bf u}(\cdot, t ) -  \lambda\bar{\bf u} (\cdot, t )\|_p \ge 0,
 \quad \Delta(0) = 0,
 \]
 and $ \Delta(t)$ is continuous.
 For fixed $t  \le  \frac{1}{2}( \frac{r}{ \| {\bf f}^*\|_p } )^2 \lambda$, suppose 
 \[
 \sup_{s \in [0, t]}  \Delta(s)= \Delta(t^*), \quad \text{for some $t^* \le t$},
 \]
 then by \eqref{eq:dt-delta2-2}, 
 \begin{align*}
 \frac{1}{2} \Delta(t^*)^2 
& = \lambda^2 \int_0^{t^*}  (\frac{d}{dt}  \frac{1}{2} \|  ({\bf u} -  \bar{\bf u} )(\cdot, s)  \|_p^2 ) ds
 	\le \lambda \int_0^{t^*}  L_1 L_2  \sqrt{ \frac{ 2 s}{  \lambda} } \|{\bf f}^* \|_p^2  \Delta(s) ds  \\
 &\le L_1 L_2 \sqrt{  2  \lambda  }  \|{\bf f}^* \|_p^2   \Delta(t^*)  \int_0^{t^*} \sqrt{s} ds
 \end{align*}
 this gives that 
 \[
\Delta(t^*)  \le   \frac{4 \sqrt{  2}}{3} L_1 L_2  \sqrt{ \lambda  }  \|{\bf f}^* \|_p^2 ( t^*)^{3/2},
 \]
 which proves (ii) by that $\Delta(t)  \le \Delta(t^*)  $ and $t^* \le t$.

To finish the proof, it remains to show \eqref{eq:claim1-pf} and \eqref{eq:claim2-pf}.

 \vspace{5pt}
 \noindent
 \underline{Proof of \eqref{eq:claim1-pf}}:
 Define ${\bf w}(x,t) : = \lambda {\bf u}( x, t) - {\bf f}^*(x)$,
 By \eqref{eq:pde-u-ubar},
 \[
 \partial_t {\bf w}(x,t) = - \lambda \E_{x' \sim p}  {\bf K}_t(x,x')   \circ  {\bf w} ( x'),
 \]
 and thus
 \[
\frac{d}{dt}\frac{1}{2} \| {\bf w}(\cdot, t) \|_p^2 
= -  \lambda  \langle {\bf w}, \E_{x' \sim p}  {\bf K}_t( \cdot ,x')   \circ  {\bf w} ( x') \rangle_p.
 \]
Denoting the $i$-th entry of ${\bf w}$ as $w_i$, by  \eqref{eq:def-NTK-t}, one can verify that 
 \begin{align}
& \langle {\bf w} ,  \E_{x' \sim p} {\bf K}_t ( \cdot, x') \circ {\bf w}(x')  \rangle_p
 =  \E_{x \sim p}  \E_{x' \sim p} \sum_{i,j=1}^d 
\langle \partial_\theta  f_i( x, \theta(t) ),  \partial_\theta  f_j(x',\theta (t)) \rangle_\Theta w_i(x) w_j(x') \nonumber \\
& ~~~ 
=  \left\langle   \sum_{i=1}^d \E_{x \sim p}    \partial_\theta  f_i( x, \theta(t) ) w_i(x) , 
  \sum_{j=1}^d \E_{x' \sim p}  \partial_\theta  f_j(x',\theta (t))   w_j(x') \right\rangle_\Theta 
=   \|  \E_{x \sim p}    \partial_\theta  {\bf f}( x, \theta(t) ) \cdot {\bf w}(x) \|_\Theta^2.
\label{eq:v-Kt-v-expression}
 \end{align}
This shows that $\frac{d}{dt}\frac{1}{2} \| {\bf w}(\cdot, t) \|_p^2  \le 0$, 
that is, $ \| {\bf w}(\cdot, t) \|_p^2$ monotonically decreases over time.
Thus for any $t \ge 0$,
 \[
 \|  \lambda {\bf u}(\cdot, t) - {\bf f}^* \|_p^2
 =  \|   {\bf w}(\cdot, t) \|_p^2
 \le \|   {\bf w}(\cdot, 0) \|_p^2
= \|{\bf f}^* \|_p^2,
 \]
 due to that $ {\bf u}( x, 0) = 0 $.

 \vspace{5pt}
 \noindent
 \underline{Proof of \eqref{eq:claim2-pf}}:
It suffices to show that for any ${\bf v} \in L^2(p)$, 
 \begin{equation}\label{eq:claim2-2}
 \langle {\bf v} ,  \E_{x' \sim p}({\bf K}_t - {\bf K}_0)( \cdot, x') \circ {\bf v}(x')  \rangle_p
 \le  L_1 L_2  \sqrt{ \frac{ 2 t}{  \lambda} } \| {\bf f}^*\|_p  \|  {\bf v}  \|_p^2.
 \end{equation}

Denote the $i$-th entry of ${\bf v}$ as $v_i$,
and, similarly as in \eqref{eq:v-Kt-v-expression} and by \eqref{eq:def-NTK-t} and  \eqref{eq:def-NTK-0}, one has 
 \begin{align*}
\langle {\bf v} ,  \E_{x' \sim p} {\bf K}_t ( \cdot, x') \circ {\bf v}(x')  \rangle_p
&=   \|  \E_{x \sim p}    \partial_\theta  {\bf f}( x, \theta(t) ) \cdot {\bf v}(x) \|_\Theta^2, \\
 \langle {\bf v} ,  \E_{x' \sim p} {\bf K}_0 ( \cdot, x') \circ {\bf v}(x')  \rangle_p
& = \|  \E_{x \sim p}    \partial_\theta  {\bf f}( x, \theta(0) ) \cdot {\bf v}(x) \|_\Theta^2.
 \end{align*}
 Subtracting the two gives that 
 \begin{align}
&  \langle {\bf v} ,  \E_{x' \sim p} ({\bf K}_t - {\bf K}_0 )( \cdot, x') \circ {\bf v}(x')  \rangle_p 
 =  \|  \E_{x \sim p}    \partial_\theta  {\bf f}( x, \theta(t) ) \cdot {\bf v}(x) \|_\Theta^2
   - \|  \E_{x \sim p}    \partial_\theta  {\bf f}( x, \theta(0) ) \cdot {\bf v}(x) \|_\Theta^2 \nonumber \\
&~~~ 
= ( \|  \E_{x \sim p}    \partial_\theta  {\bf f}( x, \theta(t) ) \cdot {\bf v}(x) \|_\Theta 
	+ \|  \E_{x \sim p}    \partial_\theta  {\bf f}( x, \theta(0) ) \cdot {\bf v}(x) \|_\Theta) \nonumber \\
&~~~~~~~~~ 
( \|  \E_{x \sim p}    \partial_\theta  {\bf f}( x, \theta(t) ) \cdot {\bf v}(x) \|_\Theta 
	- \|  \E_{x \sim p}    \partial_\theta  {\bf f}( x, \theta(0) ) \cdot {\bf v}(x) \|_\Theta).
 \label{eq:v(Kt-K0)v-bound}
 \end{align}
Because $\theta(t) \in B_r$,  cf. \eqref{eq:thetat-in-Br},
 \[
 \|  \E_{x \sim p}    \partial_\theta  {\bf f}( x, \theta(t) ) \cdot {\bf v}(x) \|
 \le \E_{x \sim p}   \| \partial_\theta  {\bf f}( x, \theta(t) ) \| \| {\bf v}(x) \|
 \le  L_1  \E_{x \sim p} \| {\bf v}(x) \| 
 \le L_1   \| {\bf v} \|_p,
 \]
 where the second inequality is by (C1). We have the same upper bound of 
 $\|  \E_{x \sim p}    \partial_\theta  {\bf f}( x, \theta(0) ) \cdot {\bf v}(x) \|_\Theta$
 because $\theta(0) \in B_r$. 
 Meanwhile, triangle inequality gives that 
 \begin{align}
&  \|  \E_{x \sim p}    \partial_\theta  {\bf f}( x, \theta(t) ) \cdot {\bf v}(x) \|_\Theta 
	- \|  \E_{x \sim p}    \partial_\theta  {\bf f}( x, \theta(0) ) \cdot {\bf v}(x) \|_\Theta \nonumber \\
&~~~
\le  \|  \E_{x \sim p} (   \partial_\theta  {\bf f}( x, \theta(t) )  - \partial_\theta  {\bf f}( x, \theta(0) )
	 	)\cdot {\bf v}(x)     \| \nonumber \\
&~~~
\le  \E_{x \sim p}  \|     \partial_\theta  {\bf f}( x, \theta(t) )  - \partial_\theta  {\bf f}( x, \theta(0) ) \| \| {\bf v}(x) \|  \nonumber \\
&~~~
\le L_2 \| \theta(t) - \theta(0)  \|  \E_{x \sim p}  \| {\bf v}(x) \|  \quad \text{(by (C2) and both $\theta(t)$ and $\theta(0)$ are in $B_r$)}
\nonumber \\
&~~~
\le L_2 \| \theta(t) - \theta(0)  \|   \| {\bf v} \|_p. \label{eq:subtraction-term-upperbound}
 \end{align}
 By (i)  which is proved,  the r.h.s. of \eqref{eq:subtraction-term-upperbound} is upper-bounded by
 $L_2  \sqrt{ {t}/{ (2\lambda)} } \| {\bf f}^*\|_p   \| {\bf v} \|_p$.
Putting back to \eqref{eq:v(Kt-K0)v-bound}, we have
 \[
 \text{(r.h.s. of \eqref{eq:v(Kt-K0)v-bound})}
 \le 2 L_1 L_2    \sqrt{ {t}/{ (2\lambda)} } \| {\bf f}^*\|_p   \| {\bf v} \|_p^2
  \]
 which proves \eqref{eq:claim2-2}.
\end{proof}

\begin{proof}[Proof of Theorem \ref{thm:NTK-stein}]
The condition on the largeness of $\lambda$ guarantees that the range of $t_0$ in \eqref{eq:cond-t0-thm} is non-empty,
and \eqref{eq:cond-t0-thm} ensures that 
$ \frac{ \log(1/\epsilon)}{\delta } \frac{1}{\lambda }  \le t \le \frac{1}{2}( \frac{r}{ \| {\bf f}^*\|_p } )^2 \lambda$.
For this range of $t$, Proposition \ref{prop:NTK-approx} applies to give \eqref{eq:y-bary}. 
Meanwhile, Proposition \ref{prop:NTK-stein} requires that $\partial_\theta {\bf f} (\cdot, \theta(0))$ are in $L^2(p)$, which follows from the  uniform boundedness of $\partial_\theta {\bf f} (\cdot, \theta(0))$ on $\calX$ by Assumption \ref{assump:C1-C2}(C1) due to that $\theta(0) \in B_r$.
The condition $t \ge \frac{ \log(1/\epsilon)}{\delta } \frac{1}{\lambda }  $
allows Proposition \ref{prop:NTK-stein} to apply to bound
 $\| \lambda \bar{\bf u}(  \cdot ,t)  - {\bf f}^* \|_{p }$ as in \eqref{eq:2eps-bound}
 for the range of $t$ being considered.
Putting together \eqref{eq:y-bary} and \eqref{eq:2eps-bound}, {by triangle inequality,} we have
\[
\| \lambda {\bf u}( x ,t)  - {\bf f}^* \|_{p }  
\le 2 \epsilon  \| {\bf f}^* \|_{p } +\frac{4\sqrt{2}}{3} L_1 L_2 \sqrt{\lambda} t^{3/2} \| {\bf f}^* \|_{p }^2
\]
and by definition 
$\sqrt{\lambda} t^{3/2} = \lambda^{-1} ( t_0  { \log(1/\epsilon)}/{\delta } )^{3/2}$, which proves \eqref{eq:2eps-bound-thm}.
\end{proof}

\subsection{{Proofs in Section \ref{subsec:NTK-finite-sample}}}

\begin{proof}[Proof of Lemma \ref{lemma:hatthetat-theta0}]
Similarly as in the proof of Proposition \ref{prop:NTK-approx}(i), we have
\[
\frac{d}{dt} \hat{L}_\lambda( \hat{\theta} (t )) = 
\langle {\partial_\theta} \hat{L}_\lambda( \hat{\theta} (t )),  \dot{ \hat{\theta}} (t) \rangle_\Theta
= - \| {\partial_\theta} \hat{L}_\lambda( \hat{\theta} (t )) \|_\Theta^2 \le 0,
\]
and for any $t \ge 0$,

\begin{align}\label{eq:bound-hatthetadot-integral}
 \int_0^t  \|  \dot{ \hat{\theta}}(s) \|_\Theta^2  ds 
& =  \int_0^t  \| {\partial_\theta} \hat{L}_\lambda( \hat{\theta} (s )) \|_\Theta^2  ds
 = \hat{L}_\lambda( \theta (0 )) - \hat{L}_\lambda( \hat{\theta} (t )).
\end{align}
By the definition \eqref{eq:def-hatL_lambda}, 
because ${\bf f}(x, \theta(0)) = \hat{\bf u}( x ,0)  = 0$,
$\hat{L}_\lambda( \theta (0 )) = 0$;
Meanwhile, let $\hat{\bf f}(x):= {\bf f}(x,  \hat{\theta}(t))$,
\begin{equation}\label{eq:hatLhatthetat-1}
 \hat{L}_\lambda( \hat{\theta}(t)) = 
    \E_{x\sim \hat{p}} \left(  
    \frac{\lambda}{2}  \| \hat{\bf f}(x) \|^2     
     -    {\bf s}_q(x) \cdot \hat{\bf f}(x)   \right)
     -  \E_{x\sim \hat{p}}  \nabla_x \cdot  \hat{\bf f}(x).  
\end{equation}
If the last term is with $\E_{x\sim {p}}$,  we have that (by that $\hat{\bf f}(x)$ is in $L^2(p)\cap \calF_0(p)$ by Assumption \ref{assump:f(x,theta)-L2-F0})
\[
- \E_{x\sim {p}}  \nabla_x \cdot  \hat{\bf f}(x)
= \E_{x\sim {p}}  \hat{\bf f}(x) \cdot {\bf s}_p(x)
= \E_{x\sim \hat{p}}  \hat{\bf f} \cdot {\bf s}_p - ( \E_{x\sim \hat{p}}  -  \E_{x\sim {p}} ) \hat{\bf f} \cdot {\bf s}_p,
\]
and thus
\begin{align}
 -  \E_{x\sim \hat{p}}  \nabla_x \cdot  \hat{\bf f}(x)
& = - \E_{x\sim {p}}  \nabla_x \cdot  \hat{\bf f}(x)
 - (\E_{x\sim \hat{p}} - \E_{x\sim {p}} ) \nabla_x \cdot  \hat{\bf f}(x)  \nonumber \\
 & =  \E_{x\sim \hat{p}}  \hat{\bf f} \cdot {\bf s}_p 
   - ( \E_{x\sim \hat{p}}  -  \E_{x\sim {p}} ) \hat{\bf f} \cdot {\bf s}_p
    - (\E_{x\sim \hat{p}} - \E_{x\sim {p}} ) \nabla_x \cdot  \hat{\bf f}(x) \nonumber \\
 & =  \E_{x\sim \hat{p}}  \hat{\bf f} \cdot {\bf s}_p  - \textcircled{1} - \textcircled{2},
\end{align}
where
\[
\textcircled{1} := ( \E_{x\sim \hat{p}}  -  \E_{x\sim {p}} ) {\bf f}(x,  \hat{\theta}(t)) \cdot {\bf s}_p(x),
\quad
\textcircled{2} := (\E_{x\sim \hat{p}} - \E_{x\sim {p}} ) \nabla_x \cdot  {\bf f}(x,  \hat{\theta}(t)). 
\]
We then have, recalling that ${\bf f}^* = {\bf s}_q -  {\bf s}_p $, 
\begin{align}
 - \hat{L}_\lambda( \hat{\theta} (t ))
& =    - \E_{x\sim \hat{p}} \left(  
    \frac{\lambda}{2}  \| \hat{\bf f} \|^2     
     -    {\bf s}_q \cdot \hat{\bf f} 
     +   {\bf s}_p \cdot \hat{\bf f}  \right)
      + \textcircled{1} + \textcircled{2}   \nonumber \\
 & =       - \E_{x\sim \hat{p}} \left(  
    \frac{ \| \lambda \hat{\bf f}  - {\bf f}^* \|^2}{2\lambda }      
     -  \frac{ \|  {\bf f}^* \|^2}{2\lambda }  \right)
      + \textcircled{1} + \textcircled{2}   \nonumber \\
  & \le      \frac{  \E_{x\sim \hat{p}}  \|  {\bf f}^* \|^2}{2\lambda }
      + \textcircled{1} + \textcircled{2} \nonumber \\
 & =        \frac{ \|  {\bf f}^* \|_p^2}{2\lambda } +  \textcircled{3}
	 + \textcircled{1} + \textcircled{2},
	 \quad 
	  \textcircled{3} :=     \frac{  (\E_{x\sim \hat{p}} - \E_{x\sim {p}}  )  \|  {\bf f}^*(x) \|^2}{2\lambda }.
	  \label{eq:bound-minushatL-hatthetat}
\end{align}
The difference $(\E_{x\sim \hat{p}} - \E_{x\sim {p}}  )  \|  {\bf f}^*(x) \|^2$ is the deviation of an independent sum sample average from its expectation.
{By Assumption \ref{assump:C3-C5}(C6), there is an integer $n_6$ (possibly depending on constants $b_2$, $b_q$) s.t. when $n > n_6$, 
under a good event $E_1$ which happens w.p. $\ge 1-n^{-10}$, 
\[
 (\E_{x\sim \hat{p}} - \E_{x\sim {p}}  )  \|  {\bf f}^*(x) \|^2 \le  \sqrt{20} b_2 \sqrt{ \frac{\log n }{n} }.
\]
}
%\old{
%By that $ \|{\bf f}^*(x)\|^2$ is uniformly bounded on $\calX$ under Assumption \ref{assump:C3-C5},
%suppose $\sup_{x \in \calX }\|{\bf f}^*(x)\|^2 \le b_f$,
%the Hoeffding's inequality (see Lemma \ref{lemma:Hoeffding}) gives that, for $b_1 := \sqrt{20} b_f $, 
%under a good event $E_1$ which happens w.p. $\ge 1-n^{-10}$, 
%\[
% (\E_{x\sim \hat{p}} - \E_{x\sim {p}}  )  \|  {\bf f}^*(x) \|^2 \le b_1 \sqrt{ \frac{\log n }{n} }.
%\]}
As a result, $ \textcircled{3} \le \frac{  \sqrt{20} b_2 }{2\lambda} \sqrt{ \frac{\log n }{n} }$.
As for $\textcircled{1}$ and $\textcircled{2}$, we know that if $\hat{\theta}(t) \in {B}_r$,
then under the good event in (C5), called $E_2$, 
\begin{equation}\label{eq:bound-circle1-circle2-E2}
\max\{ |\textcircled{1}|, \, |\textcircled{2}| \}
\le \frac{C_\calF}{n^\gamma} + b_\calF \sqrt{\frac{\log n }{n}}.
\end{equation}
However, we have not shown $\hat{\theta}(t) \in B_r$ yet.
We now let $n_\lambda$  such that when $n > n_\lambda$,
 \begin{equation}\label{eq:def-n-lambda}
 n > n_6, \quad 
 { \sqrt{20} b_2} \sqrt{ \frac{\log n }{n} } + 
 4\lambda (  \frac{C_\calF}{n^\gamma} + b_\calF \sqrt{\frac{\log n }{n}} )
 <    { \|  {\bf f}^* \|_p^2}.
 \end{equation}
 Since $\|  {\bf f}^* \|_p^2 >0$ (by Assumption \ref{assump:p-q-L2}) is a  fixed constant,
 for any $\lambda > 0$,
 %the l.h.s. is $\tilde{O}(n^{-1/2})/\lambda + {O}(n^{-\gamma}) + \tilde{O}(n^{-1/2})$ 
 the l.h.s. is $\tilde{O}(n^{-1/2}) + \lambda \tilde{O}(n^{-\gamma})$ 
 and thus will be less than the r.h.s. when $n$ is large enough.
We also let $t_{\rm max}$ be defined s.t. $  \|{\bf f}^*\|_p \sqrt{{t_{\rm max}}/{\lambda}} = {r}/{2}$.
Consider being under the intersection of good events $E_1$ and $E_2$ which happens w.p. $\ge 1-3 n^{-10}$, 
 we claim that, for any $t \le t_{\rm max}$, $ \hat{\theta} (t) \in B_r$.
 
We prove this by contradiction.
If not, then by the continuity of $\hat{\theta}(t)$ over time, there must be a $ 0 < t_1 \le t_{\rm max}$ s.t. $\| \hat{\theta} (t_1)\| = r$.
By that $\|\theta(0)\| < r/2$, we have 
\begin{equation}\label{eq:thetat1-theta0-contra}
\|  \hat{\theta} (t_1) -  {\theta} (0)\| > r/2.
\end{equation}
Meanwhile, by \eqref{eq:bound-hatthetadot-integral}\eqref{eq:bound-minushatL-hatthetat},
\[
 \int_0^{t_1}  \|  \dot{ \hat{\theta}}(s) \|_\Theta^2  ds 
= - \hat{L}_\lambda( \hat{\theta} (t_1 ))
\le   \frac{ \|  {\bf f}^* \|_p^2}{2\lambda } 
	 + \textcircled{1} + \textcircled{2} +  \textcircled{3}
\]
where $\textcircled{1}$ and $\textcircled{2}$ are taking value at $t=t_1$. 
Under $E_1$, we already have $ \textcircled{3} \le \frac{  \sqrt{20} b_2}{2\lambda} \sqrt{ \frac{\log n }{n} }$;
now $\hat{\theta}(t_1) \in \overline{B}_r$, applying (C5) we know that under $E_2$ the bound \eqref{eq:bound-circle1-circle2-E2} holds. 
By the definition of $n_\lambda$ in \eqref{eq:def-n-lambda}, we have that 
\[
\textcircled{1} + \textcircled{2} +  \textcircled{3} < \frac{ \|  {\bf f}^* \|_p^2}{2\lambda }, 
\]
and as a result, 
\[
 \int_0^{t_1}  \|  \dot{ \hat{\theta}}(s) \|_\Theta^2  ds 
 < \frac{ \|  {\bf f}^* \|_p^2}{\lambda }.
 \]
 Then, by Cauchy-Schwarz, 
\[
\|\hat{\theta}(t_1) - \theta(0)\| 
\le    \int_0^{t_1} \| \dot{ \hat{\theta}}(s) \|_\Theta ds
\le \sqrt{ t_1} ( \int_0^{t_1}  \|  \dot{ \hat{\theta}}(s) \|_\Theta^2  ds)^{1/2}
< \sqrt{ \frac{ t_1}{\lambda}}  \|  {\bf f}^*   \|_p
\le \sqrt{ \frac{ t_{\rm max}}{\lambda}}  \|  {\bf f}^*   \|_p = \frac{r}{2},
\]
which means that $\|\hat{\theta}(t_1) - \theta(0)\|  < r/2$ 
and this contradicts with \eqref{eq:thetat1-theta0-contra}.

Now we have shown that for $ \hat{\theta} (t) \in B_r$ for any $t \le  t_{\rm max}$. 
Applying the same argument to bound $\int_0^{t}  \|  \dot{ \hat{\theta}}(s) \|_\Theta^2  ds $ for any such $t$, 
and in particular the upper bound of $\textcircled{1} + \textcircled{2} +  \textcircled{3} < \frac{ \|  {\bf f}^* \|_p^2}{2\lambda }$ holds, 
one can verify that 
\[
\|\hat{\theta}(t) - \theta(0)\| \le  \sqrt{ \frac{ t}{\lambda}}  \|  {\bf f}^*   \|_p
\]
and this finishes the proof of the lemma.
\end{proof}

\begin{proof}[Proof of Proposition \ref{prop:hatu-baru}]

Recall the evolution equation of $\bar{\bf u }( x, t) $ in \eqref{eq:evolution-baru}, which can be written as
\[
{\partial_t} \bar{\bf u }( x, t) 
 =  - \E_{x' \sim p} {\bf K}_0(x,x')   \circ ( \lambda \bar{\bf u}( x' ,t) - {\bf s}_q(x')) 
   - \E_{x' \sim p} {\bf K}_0(x,x')   \circ {\bf s}_p (x').
\]
Because $\| \partial_\theta {\bf f}( x ,\theta) \|$ is uniformly bounded by $L_1$ by Assumption \ref{assump:C1-C2}(C1), each column of the $d$-by-$M_\Theta$ matrix $ \partial_\theta {\bf f}( x ,\theta(0))$, which can be viewed as a vector field on $\calX$, is uniformly bounded and thus in $L^2(p) \cap \calF_0(p)$, and then
\[
\E_{x' \sim p} \partial_\theta {\bf f}( x' ,\theta(0)) \cdot {\bf s}_p (x')
= - \E_{x' \sim p} \nabla_{x'}  \cdot \partial_\theta {\bf f}( x' ,\theta(0)).
\]
As a result, we have
\[
\E_{x' \sim p} {\bf K}_0(x,x')   \circ {\bf s}_p (x') = - \E_{x' \sim p} \nabla_{x'}  \cdot {\bf K}_0(x,x').  
\]
By comparing to \eqref{eq:evolution-hatu}, we have
 \begin{align}
{\partial_t}  (\hat{\bf u } - \bar{\bf u })( x, t) 
 & =
 - \E_{x' \sim \hat{p}}  \hat{\bf K}_t(x,x')   \circ \left( \lambda  \hat{\bf u}( x' ,t) -  {\bf s}_q(x')  \right)
+ \E_{x' \sim p} {\bf K}_0(x,x')   \circ ( \lambda \bar{\bf u}( x' ,t) - {\bf s}_q(x') )  \nonumber \\
&~~~~~
   +  \E_{x' \sim \hat{p}}  \nabla_{x'}  \cdot  \hat{\bf K}_t(x,x') 
   - \E_{x' \sim p} \nabla_{x'}  \cdot {\bf K}_0(x,x'),  \nonumber \\
& =  \textcircled{1} +  \textcircled{2}  + \textcircled{3} + \textcircled{4} + \textcircled{5},
\label{eq:partialt-(haru-baru)-pf1}
\end{align}
where
\begin{align*}
 \textcircled{1}  & := 
  - \E_{x' \sim \hat{p}}  \hat{\bf K}_t(x,x')   \circ ( \lambda  \hat{\bf u}( x' ,t) -  \lambda  \bar{\bf u}( x' ,t)  ) \\
 \textcircled{2}  & :=  
  - \E_{x' \sim \hat{p}} (  \hat{\bf K}_t(x,x')  -  {\bf K}_0(x,x')  ) \circ ( \lambda  \bar{\bf u}( x' ,t) -  {\bf s}_q(x')  ) \\
 \textcircled{3}  & :=   
 - ( \E_{x' \sim \hat{p}} - \E_{x' \sim p}   ) {\bf K}_0(x,x')  \circ ( \lambda  \bar{\bf u}( x' ,t) -  {\bf s}_q(x')  ) \\
  \textcircled{4}  & :=  
   \E_{x' \sim \hat{p}}  ( \nabla_{x'}  \cdot  \hat{\bf K}_t(x,x')  - \nabla_{x'}  \cdot {\bf K}_0 (x,x') ) \\
   \textcircled{5}  & :=    
    ( \E_{x' \sim \hat{p}} -  \E_{x' \sim {p}} ) \nabla_{x'}  \cdot {\bf K}_0 (x,x').
\end{align*}
We will analyze each of the five terms respectively toward deriving a bound of $\|  \lambda ( \hat{\bf u} - \bar{\bf u})( \cdot ,t)  \|_{p } $ as has been done in the proof of Proposition \ref{prop:NTK-approx}(ii).
We restrict to when $ t \le  ( \frac{r}{ 2\| {\bf f}^*\|_p } )^2 \lambda $ by default,
and consider being under the good event in  Lemma \ref{lemma:hatthetat-theta0} assuming that $n > n_\lambda$ 
(by \eqref{eq:def-n-lambda}, this ensures that $n$ is greater than the large $n$ threshold integer $n_6$ required by (C6)),
where we have $ \hat{\theta} (t) \in B_r$ and
\begin{equation}\label{eq:bound-hatthetat-theta0-pf1}
\| \hat{\theta} (t) - {\theta} (0) \| \le  \sqrt{ \frac{t}{ \lambda}  } \|{\bf f}^*\|_p.
\end{equation}
In the below, we may omit the dependence on $t$ in notation when there is no confusion, e.g., we write $\lambda  \hat{\bf u}( x' ,t) -   {\bf s}_q(x') $ as $( \lambda  \hat{\bf u}-   {\bf s}_q)(x')  $.

\vspace{5pt}
\underline{Bound of $\| \textcircled{5} \|$}: 
By definition, 
$ \textcircled{5} =  
   \langle \partial_\theta  {\bf f}( x, {\theta}(0) ),
    ( \E_{x' \sim \hat{p}} -  \E_{x' \sim {p}} ) \nabla_{x'}  \cdot  \partial_\theta  {\bf f}( x', {\theta}(0) ) \rangle_\Theta$,
and thus
\[
\| \textcircled{5}  \| 
\le    \|  \partial_\theta  {\bf f}( x, {\theta}(0) ) \|
    \| ( \E_{x' \sim \hat{p}} -  \E_{x' \sim {p}} ) \nabla_{x'}  \cdot  \partial_\theta  {\bf f}( x', {\theta}(0) ) \|.
\]
By (C1), $ \|  \partial_\theta  {\bf f}( x, {\theta}(0) ) \| \le L_1$;
To bound $\| ( \E_{x \sim \hat{p}} -  \E_{x \sim {p}} ) \nabla_{x}  \cdot  \partial_\theta  {\bf f}( x, {\theta}(0) ) \|$, 
we utilize the matrix Bernstein inequality (see, e.g., \cite[Theorem 6.1.1]{tropp2015introduction}, reproduced as Lemma \ref{lemma:matrix-bernstein}):
The matrix $ \nabla_{x}  \cdot  \partial_\theta  {\bf f}( x_i, {\theta}(0) )$ is of size $1$-by-$M_\Theta$, and the operator norm is bounded by $L_3$ by Assumption \ref{assump:C3-C5}(C3). By Lemma \ref{lemma:matrix-bernstein}, 
let $n_5 >2 $ be an integer s.t. $n > n_5$ implies 
\begin{equation}\label{eq:def-n5}
(10 \log n + \log(2 M_\Theta))/n < (3/2)^2,
\end{equation}
then when $n > n_5$, there is a good event $E_5$ that happens w.p. $\ge 1-n^{-10}$ under which
\[
\| ( \E_{x \sim \hat{p}} -  \E_{x \sim {p}} ) \nabla_{x}  \cdot  \partial_\theta  {\bf f}( x, {\theta}(0) ) \|
\le 4 \sqrt{11} L_3 \sqrt{ \frac{\log n + \log M_\Theta}{n}}.
\]
As a result, when $n > n_5$ and under $E_5$,
\begin{equation}\label{eq:bound-circle5-pf2}
\| \textcircled{5}  \| 
\le 4 \sqrt{11} L_1 L_3 \sqrt{ \frac{\log n + \log M_\Theta}{n}}.
\end{equation}

\vspace{5pt}
\underline{Bound of $\| \textcircled{4} \|$}:
By definition, 
\begin{align*}
 \nabla_{x'}  \cdot  \hat{\bf K}_t(x,x')  - \nabla_{x'}  \cdot {\bf K}_0 (x,x') 
& =  \langle \partial_\theta  {\bf f}( x,  \hat{\theta}(t) ),  \nabla_{x'} \cdot \partial_\theta  {\bf f}(x', \hat{\theta} (t)) \rangle_\Theta \\
&~~~ 
   - \langle  \partial_\theta  {\bf f}( x, {\theta}(0) ),  \nabla_{x'} \cdot  \partial_\theta  {\bf f}(x', {\theta} (0)) \rangle_\Theta \\
& = \langle \partial_\theta  {\bf f}( x,  \hat{\theta}(t) )- \partial_\theta  {\bf f}( x,  {\theta}(0) ),  
		\nabla_{x'} \cdot \partial_\theta  {\bf f}(x', \hat{\theta} (t)) \rangle_\Theta    \\
&~~~ 
    	+	\langle \partial_\theta  {\bf f}( x,  {\theta}(0) ),
     			\nabla_{x'} \cdot \partial_\theta  {\bf f}(x', \hat{\theta} (t)) 
			- \nabla_{x'} \cdot \partial_\theta  {\bf f}(x', {\theta} (0))  \rangle_\Theta.    
\end{align*}
By that $ \hat{\theta} (0), \, \hat{\theta} (t) \in B_r$  and (C1)(C2)(C3)(C4),
\[
\|  \nabla_{x'}  \cdot  \hat{\bf K}_t(x,x')  - \nabla_{x'}  \cdot {\bf K}_0 (x,x')  \|
\le (L_2L_3 + L_1L_4) \| \hat{\theta} (t) - {\theta} (0) \|.
\]
Together with \eqref{eq:bound-hatthetat-theta0-pf1}, we have that 
\begin{equation}\label{eq:bound-circle4-pf2}
\| \textcircled{4}  \| 
\le (L_2L_3 + L_1L_4)  \sqrt{ {t}/{ \lambda}  } \|{\bf f}^*\|_p.
\end{equation}

\vspace{5pt}
\underline{Bound of $\| \textcircled{3} \|$}: 
By definition, 
$ \textcircled{3}  =   
 - \langle  \partial_\theta  {\bf f}( x, {\theta}(0) ), ( \E_{x' \sim \hat{p}} - \E_{x' \sim p}   ) \partial_\theta  {\bf f}( x', {\theta}(0) )^T
   ( \lambda  \bar{\bf u}( x' ,t) -  {\bf s}_q(x')  ) \rangle_\Theta$,
   thus 
 \begin{align}
\|  \textcircled{3}  \| 
&  \le \|  \partial_\theta  {\bf f}( x, {\theta}(0) ) \| 
 \| ( \E_{x' \sim \hat{p}} - \E_{x' \sim p}   ) \partial_\theta  {\bf f}( x', {\theta}(0) )^T
   ( \lambda  \bar{\bf u}( x' ,t) -  {\bf s}_q(x')  ) \| 
     \le L_1 \|  \textcircled{a}_3 - \textcircled{b}_3 \|,  \nonumber 
 \end{align}
where
\begin{align*}
\textcircled{a}_3 
& := ( \E_{x' \sim \hat{p}} - \E_{x' \sim p}   ) \partial_\theta  {\bf f}( x', {\theta}(0) )^T
    \lambda  \bar{\bf u}( x' ,t), \\
\textcircled{b}_3
& := 
( \E_{x' \sim \hat{p}} - \E_{x' \sim p}   ) \partial_\theta  {\bf f}( x', {\theta}(0) )^T
     {\bf s}_q(x').
\end{align*}

For $\textcircled{b}_3$, we need to bound $\| (\E_{x \sim \hat{p}} - \E_{x \sim {p}}  )  {\bf s}_q(x)^T \partial_\theta  {\bf f}( x, {\theta}(0) )\|$.
At each $x$, ${\bf s}_q(x)^T \partial_\theta  {\bf f}( x, {\theta}(0) )$ is a $1$-by-$M_\Theta$ vector,
 and the norm is bounded by {$L_1 b_q$ uniformly on $\calX$ by (C6)}.
Applying Lemma \ref{lemma:matrix-bernstein}
and consider  when $n > n_5$ which implies  $(10 \log n + \log(1+ M_\Theta))/n < (3/2)^2$, 
there is a good event $E_b$ that happens w.p. $\ge 1-n^{-10}$ under which
\begin{equation}\label{eq:bound-b2-pf1}
\| (\E_{x \sim \hat{p}} - \E_{x \sim {p}}  )  {\bf s}_q(x)^T \partial_\theta  {\bf f}( x, {\theta}(0) )\|
< 2 \sqrt{11} L_1 b_q \sqrt{\frac{ \log n + \log M_\Theta}{n}}.
\end{equation}
This shows that when $n > n_5$ and under $E_b$, 
$\| \textcircled{b}_3 \|$ is bounded by the r.h.s. of \eqref{eq:bound-b2-pf1}.

For $\textcircled{a}_3$,  we will bound the deviation to be $\tilde{O}(n^{-1/2})$ uniformly for all $t$. The observation is that though  $\bar{\bf u}(x, t) $ changes over $t$, the concentration argument of the sample average only involves the kernel ${\bf K}_0(x,x') $. 
Specifically, by definition,
\begin{align}\label{eq:expression-lambda-baruxt}
\lambda \bar{\bf u}(x, t) 
 =   - \lambda  \int_0^t \E_{y \sim p} [ {\bf K}_0(x,y) ] (\lambda \bar{\bf u}(y, s) -  {\bf f}^*(y)) ds,
\end{align}
and then 
\[
\textcircled{a}_3 
=   - \lambda \int_0^t   
\Big[ ( \E_{x' \sim \hat{p}} - \E_{x' \sim p}   ) \partial_\theta  {\bf f}( x', {\theta}(0) )^T
   \partial_\theta  {\bf f}( x', {\theta}(0) ) \Big]
   \E_{y \sim p}  \partial_\theta  {\bf f}( y, {\theta}(0) )^T 
   (\lambda \bar{\bf u}(y, s) -  {\bf f}^*(y)) ds.
\]
Note that 
$ \|\E_{y \sim p}  \partial_\theta  {\bf f}( y, {\theta}(0) )^T 
   (\lambda \bar{\bf u}(y, s) -  {\bf f}^*(y)) \| 
   \le L_1 \E_{y \sim p} \| \lambda \bar{\bf u}(y, s) -  {\bf f}^*(y)\|$ by (C1), 
and $\E_{y \sim p} \| \lambda \bar{\bf u}(y, s) -  {\bf f}^*(y)\| \le \| \lambda \bar{\bf u}(\cdot, s) -  {\bf f}^*\|_p$.
The latter can be bounded by $\| {\bf f}^*\|_p $ using the same argument as in (Claim 1) \eqref{eq:claim1-pf} proved in the proof of Proposition \ref{prop:NTK-approx}, that is, for any $t \ge 0$,
\begin{equation}\label{eq:claim1-baru}
\|  \lambda \bar{\bf u}(\cdot, t) - {\bf f}^* \|_p 
\le \|  \lambda \bar{\bf u}(\cdot, 0) - {\bf f}^* \|_p 
= \|{\bf f}^* \|_p.
\end{equation}
Putting together, this implies that 
\begin{align*}
\| \textcircled{a}_3  \|
& \le \lambda L_1 \int_0^t   
\Big\| ( \E_{x' \sim \hat{p}} - \E_{x' \sim p}   ) \partial_\theta  {\bf f}( x', {\theta}(0) )^T
   \partial_\theta  {\bf f}( x', {\theta}(0) ) \Big\|
   \E_{y \sim p} \| \lambda \bar{\bf u}(y, s) -  {\bf f}^*(y)\| ds \\
 & \le   \lambda t L_1     \| {\bf f}^*\|_p 
\Big\| ( \E_{x' \sim \hat{p}} - \E_{x' \sim p}   ) \partial_\theta  {\bf f}( x', {\theta}(0) )^T
   \partial_\theta  {\bf f}( x', {\theta}(0) ) \Big\|.
\end{align*}
Next, we bound $\| (\E_{x \sim \hat{p}} - \E_{x \sim {p}}  )    \partial_\theta  {\bf f}( x, {\theta}(0) )^T \partial_\theta  {\bf f}( x, {\theta}(0) )  \|$ by matrix Bernstein inequality Lemma \ref{lemma:matrix-bernstein}:
The matrix $   \partial_\theta  {\bf f}( x, {\theta}(0) )^T \partial_\theta  {\bf f}( x, {\theta}(0) )$ is of size $M_\Theta$-by-$M_\Theta$, and its operator norm is bounded by $L_1^2$ by (C1). 
%When $n$ is large s.t. $(10 \log n + \log(2 M_\Theta))/n < (3/2)^2$, 
When $n > n_5$ as defined by \eqref{eq:def-n5},
Lemma \ref{lemma:matrix-bernstein} gives that w.p. $\ge 1-n^{-10}$, 
\begin{equation}\label{eq:bound-a2-pf1}
 \| (\E_{x \sim \hat{p}} - \E_{x \sim {p}}  )    \partial_\theta  {\bf f}( x, {\theta}(0) )^T \partial_\theta  {\bf f}( x, {\theta}(0) ) \| 
	\le 2 L_1^2 \sqrt{\frac{10 \log n + \log (2M_\Theta)}{n}}
	< 2\sqrt{11} L_1^2 \sqrt{\frac{\log n+\log M_\Theta}{n}}.
\end{equation}
We call the above good event $E_a$. Thus, when $n > n_5$ and under $E_a$, 
\[
\| \textcircled{a}_3  \| 
\le  2\sqrt{11}  (\lambda t) L_1^3     \| {\bf f}^*\|_p 
 \sqrt{\frac{\log n+\log M_\Theta}{n}}.
\]
Putting together, we have that when $n > n_5$ and under $E_a \cap E_b$,
\begin{equation}\label{eq:bound-circle3-pf2}
\| \textcircled{3}  \| 
\le L_1 (\| \textcircled{a}_3  \| + \| \textcircled{b}_3  \|)
\le 2\sqrt{11} L_1^2 ( 
   (\lambda t) L_1^2     \| {\bf f}^*\|_p  +  b_q) 
\sqrt{\frac{ \log n + \log M_\Theta}{n}}.
\end{equation}

\vspace{5pt}
\underline{Bound of $\| \textcircled{2} \|$}: By definition,
\begin{align}
 \hat{\bf K}_t(x,x')  -  {\bf K}_0(x,x')
 & = \partial_\theta  {\bf f}( x,  \hat{\theta}(t) ) \partial_\theta  {\bf f}(x', \hat{\theta} (t))^T
    - \partial_\theta  {\bf f}( x, {\theta}(0) ) \partial_\theta  {\bf f}(x', {\theta} (0))^T.  \nonumber 
\end{align}
Since ${\theta} (0), \, \hat{\theta} (t) \in B_r$, by (C1)(C2), $\forall x, x' \in \calX$,
\begin{align*}
\|  \hat{\bf K}_t(x,x')  -  {\bf K}_0(x,x') \|
& \le 
 \| \partial_\theta  {\bf f}( x,  \hat{\theta}(t) ) \| \|  \partial_\theta  {\bf f}(x', \hat{\theta} (t)) -  \partial_\theta  {\bf f}(x', {\theta} (0)) \| \\
&~~~  
   + \| \partial_\theta  {\bf f}( x,  \hat{\theta}(t) ) -\partial_\theta  {\bf f}( x, {\theta}(0) )  \| \| \partial_\theta  {\bf f}(x', {\theta} (0)) \| \\
& \le  2  L_1 L_2 \| \hat{\theta} (t) - {\theta} (0) \|  \\
& \le  2  L_1 L_2  \sqrt{ \frac{t}{ \lambda}  } \|{\bf f}^*\|_p,
\end{align*}
where we used \eqref{eq:bound-hatthetat-theta0-pf1} in the last inequality. This gives that
\begin{equation}\label{eq:bound-circle2-pf1}
\| \textcircled{2}\| \le 
2  L_1 L_2  \sqrt{ \frac{t}{ \lambda}  } \|{\bf f}^*\|_p
 \E_{x' \sim \hat{p}}   \| \lambda  \bar{\bf u}( x' ,t) -  {\bf s}_q(x')  ) \|.
\end{equation}
We will bound $ \E_{x' \sim \hat{p}}   \| \lambda  \bar{\bf u}( x' ,t) -  {\bf s}_q(x')  ) \|$ to be $O(1)$. 
First, note that we have
\begin{align}
\|  \lambda \bar{\bf u}(\cdot, t) - {\bf s}_q \|_p 
& \le \|  \lambda \bar{\bf u}(\cdot, t) - ({\bf s}_q- {\bf s}_p)\|_p  + \| {\bf s}_p\|_p  \nonumber \\
& \le \|{\bf f}^* \|_p + \| {\bf s}_p\|_p  \quad \text{(by \eqref{eq:claim1-baru})} \nonumber \\
& \le  
2 \|{\bf f}^* \|_p  + b_q. 
\quad \text{(by $\| {\bf s}_p\|_p \le \|{\bf f}^* \|_p +  \| {\bf s}_q\|_p$ and (C6))}
\label{eq:bound-lambdabaru-sq-pf1}
\end{align}
%namely,
%\begin{equation}\label{eq:bound-lambdabaru-sq-pf2}
%$\|  \lambda \bar{\bf u}(\cdot, t) - {\bf s}_q \|_p^2 \le  C_{2,1}^2$.
%\end{equation}
We claim that when $n$ is large enough and under a high probability good event, to be specified below,
\begin{equation}\label{eq:claim-concentration-hatp2norm}
\text{(Claim)} \quad
(\E_{x \sim \hat{p}} - \E_{x \sim {p}}  ) \|  \lambda \bar{\bf u}(x, t) - {\bf s}_q(x) \|^2
\le  
\left( L_1^2  \| {\bf f}^*\|_p   + b_q \right)^2.
\end{equation}
Together with \eqref{eq:bound-lambdabaru-sq-pf1}, this will imply that 
\begin{align}
\E_{x \sim \hat{p}}   \| \lambda  \bar{\bf u}( x ,t) -  {\bf s}_q(x)  )\|
& \le 
( \E_{ x\sim \hat{p}} \|  \lambda \bar{\bf u}(x, t) - {\bf s}_q(x) \|^2 )^{1/2}  \nonumber \\
& \le { (2+L_1^2) \| {\bf f}^*\|_p  + 2 b_q }
:=C_{2,2}.
%\old{2  \| {\bf s}_p\|_\infty +  2  \|{\bf s}_q\|_\infty  +   L_1^2  \| {\bf f}^*\|_p }
\label{eq:def-const-C22}
\end{align}
%namely, $C_{2,2}$ is an $O(1)$ constant determined by $\|{\bf s}_p\|_\infty$, $\|{\bf s}_q\|_\infty$. %and $\| {\bf f}^*\|_p$.
Then \eqref{eq:bound-circle2-pf1} gives
\begin{equation}\label{eq:bound-circle2-pf2}
\| \textcircled{2}\| \le 
2    L_1 L_2 C_{2,2} \sqrt{ {t}/{ \lambda}  } \|{\bf f}^*\|_p.
\end{equation}

To prove the claim \eqref{eq:claim-concentration-hatp2norm}, 
we will bound $(\E_{x \sim \hat{p}} - \E_{x \sim {p}}  )\|  \lambda \bar{\bf u}(x, t) - {\bf s}_q(x) \|^2$ to be $\tilde{O}(n^{-1/2})$ uniformly for all $t$
similarly as in the analysis of $\textcircled{a}_3$ above, where we leverage the expression \eqref{eq:expression-lambda-baruxt} of $\bar{\bf u}(x, t) $ via ${\bf K}_0(x,x') $.
Based on the expression, we have
\begin{align}
\|\lambda \bar{\bf u}(x, t)\|^2
 & =    \lambda^2  \int_0^t  \int_0^t  \E_{x_1' \sim p}  \E_{x_2' \sim p} 
  (\lambda \bar{\bf u}(x_1', s_1) -  {\bf f}^*(x_1'))^T  \nonumber \\
& ~~~ 
[{\bf K}_0(x,x_1')^T   {\bf K}_0(x,x_2')] (\lambda \bar{\bf u}(x_2', s_2) -  {\bf f}^*(x_2')) ds_1 ds_2. \nonumber
\end{align}
Meanwhile, 
$\|  \lambda \bar{\bf u}(x, t) - {\bf s}_q(x) \|^2
= \|  \lambda \bar{\bf u}(x, t) \|^2 - 2  {\bf s}_q(x)\cdot \lambda \bar{\bf u}(x, t) +  \| {\bf s}_q(x)\|^2$,
and then 
\[
(\E_{x \sim \hat{p}} - \E_{x \sim {p}}  ) \|  \lambda \bar{\bf u}(x, t) - {\bf s}_q(x) \|^2
= \textcircled{a}_2  -2 \textcircled{b}_2 + \textcircled{c}_2,
\]
where 
\begin{align*}
 \textcircled{a}_2 
 & := 
   \lambda^2  \int_0^t  \int_0^t  \E_{x_1' \sim p}  \E_{x_2' \sim p} 
  (\lambda \bar{\bf u}(x_1', s_1) -  {\bf f}^*(x_1'))^T   \\
& ~~~~~~ 
\Big[ (\E_{x \sim \hat{p}} - \E_{x \sim {p}}  )  {\bf K}_0(x,x_1')^T   {\bf K}_0(x,x_2') \Big] 
(\lambda \bar{\bf u}(x_2', s_2) -  {\bf f}^*(x_2')) ds_1 ds_2 \\
& = 
   \lambda^2  \int_0^t  \int_0^t  \E_{x_1' \sim p}  \E_{x_2' \sim p} 
  (\lambda \bar{\bf u}(x_1', s_1) -  {\bf f}^*(x_1'))^T  \partial_\theta  {\bf f}(x_1', {\theta} (0))  \\
& ~~~~~~ 
\Big[ (\E_{x \sim \hat{p}} - \E_{x \sim {p}}  )  
  \partial_\theta  {\bf f}( x, {\theta}(0) )^T
\partial_\theta  {\bf f}( x, {\theta}(0) )  \Big] 
\partial_\theta  {\bf f}(x_2', {\theta} (0))^T (\lambda \bar{\bf u}(x_2', s_2) -  {\bf f}^*(x_2')) ds_1 ds_2,
\end{align*}
and, (C1) implies that $\|  \partial_\theta  {\bf f}(x_1', {\theta} (0))  \|, \, \|\partial_\theta  {\bf f}(x_2', {\theta} (0)) \| \le L_1$, 
\begin{align}
| \textcircled{a}_2 | 
& \le 
  \lambda^2 L_1^2  
    \Big\| (\E_{x \sim \hat{p}} - \E_{x \sim {p}}  )    \partial_\theta  {\bf f}( x, {\theta}(0) )^T \partial_\theta  {\bf f}( x, {\theta}(0) )  \Big\|
\left(  \int_0^t   \E_{x_1' \sim p}  \| \lambda \bar{\bf u}(x_1', s_1) -  {\bf f}^*(x_1')\|ds_1  \right)^2 \nonumber \\
& \le (\lambda t)^2 L_1^2  \| {\bf f}^*\|_p^2
   \Big\| (\E_{x \sim \hat{p}} - \E_{x \sim {p}}  )    \partial_\theta  {\bf f}( x, {\theta}(0) )^T \partial_\theta  {\bf f}( x, {\theta}(0) )  \Big\|,
   \label{eq:bound-circle2-a-pf1}
\end{align}
where the last inequality used that
$ \E_{x \sim p}  \| \lambda \bar{\bf u}(x, s) -  {\bf f}^*(x)\| \le  \| \lambda \bar{\bf u}(\cdot, s) -  {\bf f}^*\|_p \le \| {\bf f}^*\|_p  $  by \eqref{eq:claim1-baru};
Meanwhile,
\begin{align*}
 \textcircled{b}_2
 & := -   \lambda  \int_0^t \E_{x' \sim p} 
 \Big[ (\E_{x \sim \hat{p}} - \E_{x \sim {p}}  )  {\bf s}_q(x)^T {\bf K}_0(x,x') \Big] 
 (\lambda \bar{\bf u}(x', s) -  {\bf f}^*(x')) ds\\
 &= -   \lambda  \int_0^t \E_{x' \sim p} 
 \Big[ (\E_{x \sim \hat{p}} - \E_{x \sim {p}}  )  {\bf s}_q(x)^T \partial_\theta  {\bf f}( x, {\theta}(0) )\Big]
  \partial_\theta  {\bf f}( x', {\theta}(0) )^T  (\lambda \bar{\bf u}(x', s) -  {\bf f}^*(x')) ds,
  \end{align*}
  and, similarly, 
\begin{align}
| \textcircled{b}_2 | 
& \le    \lambda L_1  \Big\| (\E_{x \sim \hat{p}} - \E_{x \sim {p}}  )  {\bf s}_q(x)^T \partial_\theta  {\bf f}( x, {\theta}(0) )\Big\|
 \int_0^t \E_{x' \sim p}   \|\lambda \bar{\bf u}(x', s) -  {\bf f}^*(x')\| ds \nonumber \\
& \le  ( \lambda t ) L_1 \| {\bf f}^*\|_p
\Big\| (\E_{x \sim \hat{p}} - \E_{x \sim {p}}  )  {\bf s}_q(x)^T \partial_\theta  {\bf f}( x, {\theta}(0) )\Big\|;
\label{eq:bound-circle2-b-pf1}
\end{align}
Finally, 
\[
\textcircled{c}_2 :=  (\E_{x \sim \hat{p}} - \E_{x \sim {p}}  )  \| {\bf s}_q(x)\|^2.
\]
%and 
%\begin{equation}
%| \textcircled{c}_2 |
% = |  (\E_{x \sim \hat{p}} - \E_{x \sim {p}}  )  \| {\bf s}_q(x)\|^2 |.
%\end{equation}
To bound $| \textcircled{a}_2 | $,  $| \textcircled{b}_2 | $, and $ \textcircled{c}_2  $, we use the concentration argument respectively. 

For $\textcircled{a}_2$, recall that we have bounded  $\| (\E_{x \sim \hat{p}} - \E_{x \sim {p}}  )    \partial_\theta  {\bf f}( x, {\theta}(0) )^T \partial_\theta  {\bf f}( x, {\theta}(0) )  \|$ in \eqref{eq:bound-a2-pf1}. 
This means that when $n > n_5$ and under $E_a$, 
\begin{equation}\label{eq:bound-circle2-a-pf2}
| \textcircled{a}_2 | 
\le  2\sqrt{11}  (\lambda t)^2 L_1^4  \| {\bf f}^*\|_p^2
    \sqrt{\frac{\log n+\log M_\Theta}{n}}.
\end{equation}

For $\textcircled{b}_2$, we have bounded $\| (\E_{x \sim \hat{p}} - \E_{x \sim {p}}  )  {\bf s}_q(x)^T \partial_\theta  {\bf f}( x, {\theta}(0) )\|$ in \eqref{eq:bound-b2-pf1}. Thus, when $n > n_5$ and under $E_b$, \eqref{eq:bound-circle2-b-pf1} gives that
\begin{equation}\label{eq:bound-circle2-b-pf2}
| \textcircled{b}_2 | 
\le  2 \sqrt{11}  ( \lambda t ) L_1^2 \| {\bf f}^*\|_p
	 b_q
	 \sqrt{\frac{ \log n + \log M_\Theta}{n}}.
\end{equation}

For $\textcircled{c}_2$, by Assumption \ref{assump:C3-C5}(C6) and that $n > n_\lambda$ ensures $n > n_6$, 
there is a good event $E_c$ that happens w.p. $\ge 1-n^{-10}$ under which
%\old{because $\| {\bf s}_q(x_i)\|^2 \le \| {\bf s}_q\|_\infty^2$, by the Hoeffding's inequality Lemma \ref{lemma:Hoeffding},}
\begin{equation}\label{eq:bound-circle2-c-pf2}
 \textcircled{c}_2  
 \le \sqrt{20} b_q^2 \sqrt{\frac{\log n}{n}}.
 \end{equation}
Putting together \eqref{eq:bound-circle2-a-pf2}\eqref{eq:bound-circle2-b-pf2}\eqref{eq:bound-circle2-c-pf2}, 
when $n > \max \{n_\lambda, n_5 \} $ and  under the intersection of $E_a$, $E_b$ and $E_c$,
we have
\begin{align*}
(\E_{x \sim \hat{p}} - \E_{x \sim {p}}  ) \|  \lambda \bar{\bf u}(x, t) - {\bf s}_q(x) \|^2
&   \le 
|\textcircled{a}_2|  +2 |\textcircled{b}_2| + \textcircled{c}_2 \\
%& \le 
%2\sqrt{11}  (\lambda t)^2 L_1^4  \| f^*\|_p^2
%    \sqrt{\frac{\log n+\log M_\Theta}{n}}
%+ 4 \sqrt{11}  ( \lambda t ) L_1^2 \| f^*\|_p \| {\bf s}_q\|_\infty \sqrt{\frac{ \log n + \log M_\Theta}{n}}
%+ 	 \sqrt{20} \| {\bf s}_q\|_\infty^2 \sqrt{\frac{\log n}{n}} \\
& \le
2\sqrt{11} 
\left( (\lambda t) L_1^2  \| {\bf f}^*\|_p   + b_q \right)^2
\sqrt{\frac{\log n+\log M_\Theta}{n}}.
\end{align*}
Let $n_{\lambda,t}$ be the integer s.t. when $n > n_{\lambda,t}$, 
\begin{equation}\label{eq:def-n-lambda-t}
2\sqrt{11} \max\{ (\lambda t)^2, 1\} \sqrt{\frac{\log n+\log M_\Theta}{n}} < 1,
\end{equation}
then we have shown that the claim \eqref{eq:claim-concentration-hatp2norm} holds when $n > \max\{ n_\lambda, n_5, n_{\lambda,t} \}$  and under $E_a \cap E_b \cap E_c$.

\vspace{5pt}
\underline{Analysis of \eqref{eq:partialt-(haru-baru)-pf1}}:
We handle \textcircled{1} based on the proved bounds of the other four terms in \eqref{eq:partialt-(haru-baru)-pf1}.
Define ${\textcircled{II} }(x,t) := \textcircled{2}  + \textcircled{3} + \textcircled{4} + \textcircled{5}$.
Collecting \eqref{eq:bound-circle5-pf2}\eqref{eq:bound-circle4-pf2}\eqref{eq:bound-circle3-pf2}\eqref{eq:bound-circle2-pf2},
we have that when $n > \max\{ n_\lambda, n_5, n_{\lambda,t} \}$  and under the intersection (of the good event in Lemma \ref{lemma:hatthetat-theta0} and) $E_a \cap E_b \cap E_c \cap E_5$, 
\begin{align}
\| \textcircled{II}(x,t) \|
& \le
 ( 2    L_1 L_2 C_{2,2}   + L_2L_3 + L_1L_4  ) \sqrt{ {t}/{ \lambda}  } \|{\bf f}^*\|_p \nonumber \\
&~~~
+ 2\sqrt{11}  ( 
   \lambda t  L_1^4     \| {\bf f}^*\|_p  + L_1^2  b_q + 2  L_1 L_3) 
	\sqrt{\frac{ \log n + \log M_\Theta}{n}} \nonumber \\
& \le C_{2,1} \left(  \sqrt{ {t}/{ \lambda}  } \|{\bf f}^*\|_p + (\lambda t + 1) \sqrt{\frac{ \log n + \log M_\Theta}{n}}  \right)=: g(t), 
\label{eq:def-g(t)-growth}
\end{align}
where 
\begin{equation}\label{eq:def-const-C21}
C_{2,1} : = \max\{
 2    L_1 L_2 C_{2,2}   + L_2L_3 + L_1L_4  , \,
2\sqrt{11} L_1^4  \| {\bf f}^*\|_p, \,
2\sqrt{11} (L_1^2 b_q + 2  L_1 L_3) \}, 
\end{equation}
and $C_{2,2}$ is defined as in \eqref{eq:def-const-C22}.
%$C_{2,2} \le (2+L_1^2)(\| {\bf s}_p\|_\infty + \| {\bf s}_q\|_\infty  )$ 
% and we used $ \| {\bf f}^*\|_p \le  \| {\bf f}^*\|_\infty \le \| {\bf s}_p\|_\infty + \| {\bf s}_q\|_\infty$ in the last inequality.
Back to \eqref{eq:partialt-(haru-baru)-pf1}, we have
\begin{equation}\label{eq:partialt-(haru-baru)-pf2}
{\partial_t}  (\hat{\bf u } - \bar{\bf u })( x, t) 
=  \textcircled{1}  +   \textcircled{II}(x,t)
=  - \E_{x' \sim \hat{p}}  \hat{\bf K}_t(x,x')   \circ \lambda ( \hat{\bf u} - \bar{\bf u} ) ( x' ,t)   +  \textcircled{II}(x,t). 
\end{equation}
We first derive a bound of 
\[ 
\| \lambda( \hat{\bf u } - \bar{\bf u })(\cdot, t) \|_{\hat{p}}
:=
(\E_{x\sim \hat{p}}\| \lambda( \hat{\bf u } - \bar{\bf u })(x, t) \|^2)^{1/2},
\]
 namely the mean squared error on training samples. By \eqref{eq:partialt-(haru-baru)-pf2}, 
\begin{align}
\frac{d}{dt} \frac{1}{2} \| \lambda( \hat{\bf u } - \bar{\bf u })(\cdot, t) \|_{\hat{p}}^2
&= \lambda^2
\E_{x\sim \hat{p}} (\hat{\bf u } - \bar{\bf u })( x, t)^T {\partial_t}  (\hat{\bf u } - \bar{\bf u })( x, t) \nonumber \\
&\le
 \lambda^2
\E_{x\sim \hat{p}} (\hat{\bf u } - \bar{\bf u })( x, t)^T \textcircled{II}(x,t)  \nonumber \\
& \le 
 \lambda^2 \E_{x\sim \hat{p}} \| (\hat{\bf u } - \bar{\bf u })( x, t)\| \|\textcircled{II}(x,t) \| \nonumber  \\
 & \le  \lambda^2 g(t) \E_{x\sim \hat{p}} \| (\hat{\bf u } - \bar{\bf u })( x, t)\| \nonumber \\
 & \le  \lambda g(t)  \| \lambda (\hat{\bf u } - \bar{\bf u })( \cdot, t)\|_{\hat{p}}, \nonumber
\end{align}
where in the first inequality we used that  
\[
 \E_{x\sim \hat{p}} (\hat{\bf u } - \bar{\bf u })( x, t)^T   \textcircled{1}
 = 
 -  \lambda  \E_{x\sim \hat{p}}  \E_{x' \sim \hat{p}}  (\hat{\bf u } - \bar{\bf u })( x, t)^T [ \hat{\bf K}_t(x,x')] (\hat{\bf u} -    \bar{\bf u} )( x' ,t)  
\le 0. 
\]
Using the same argument as in the proof of Proposition \ref{prop:NTK-approx}(ii) based on \eqref{eq:dt-delta2-2}, we have that for any $t$ being considered,
\begin{align}
\| \lambda( \hat{\bf u } - \bar{\bf u })(\cdot, t) \|_{\hat{p}} 
& \le 2 \lambda \int_0^t g(s) ds  \nonumber \\
& \le  2 \lambda \int_0^t g(t) ds   \quad \text{($g(t)$ is monotone in $t$)} \nonumber \\
%& = 2 C_2   
%	 \left(  \frac{2}{3}   \sqrt{ \lambda  } t^{3/2}   \|{\bf f}^*\|_p 
%	 + (\lambda t   + \frac{(\lambda t)^2}{2}) \sqrt{\frac{ \log n + \log M_\Theta}{n}}  \right) 
& = 2 \lambda t g(t)  .
\label{eq:hatp-hatu-baru-bound}
\end{align}
Next, we use \eqref{eq:hatp-hatu-baru-bound} to bound $\| \lambda( \hat{\bf u } - \bar{\bf u })(\cdot, t) \|_{{p}}$.
Again from \eqref{eq:partialt-(haru-baru)-pf2}, we have
\begin{align}
& \frac{d}{dt} \frac{1}{2} \| \lambda( \hat{\bf u } - \bar{\bf u })(\cdot, t) \|_{{p}}^2
= \lambda^2 \E_{x\sim{p}} (\hat{\bf u } - \bar{\bf u })( x, t)^T {\partial_t}  (\hat{\bf u } - \bar{\bf u })( x, t)  \nonumber  \\
&~~~
 =-  \lambda^2 \E_{x\sim{p}}  \E_{x' \sim \hat{p}} 
 	(\hat{\bf u } - \bar{\bf u })( x, t)^T [ \hat{\bf K}_t(x,x')]
	 \lambda ( \hat{\bf u} - \bar{\bf u} ) ( x' ,t) 
	 + \lambda^2 \E_{x\sim{p}} (\hat{\bf u } - \bar{\bf u })( x, t)^T   \textcircled{II}(x,t)  \nonumber  \\
&~~~ 
\le  \lambda L_1^2 \E_{x\sim{p}} \| \lambda(\hat{\bf u } - \bar{\bf u })( x, t) \|  \E_{x' \sim \hat{p}}   \| \lambda(\hat{\bf u} - \bar{\bf u} ) ( x' ,t) \|
 + \lambda \E_{x\sim{p}} \| \lambda (\hat{\bf u } - \bar{\bf u })( x, t)\|  \| \textcircled{II}(x,t)\|  \nonumber  \\
&~~~ 
\le  \lambda L_1^2 \| \lambda( \hat{\bf u } - \bar{\bf u })(\cdot, t) \|_{{p}}  \| \lambda( \hat{\bf u } - \bar{\bf u })(\cdot, t) \|_{\hat{p}} 
 + \lambda g(t)  \| \lambda( \hat{\bf u } - \bar{\bf u })(\cdot, t) \|_{{p}}  \nonumber  \\
&~~~
\le   \lambda ( L_1^2  2\lambda t +   1 ) g(t)  \| \lambda( \hat{\bf u } - \bar{\bf u })(\cdot, t) \|_{{p}}  \quad \text{(by \eqref{eq:hatp-hatu-baru-bound})}
\nonumber 
\end{align}
where in the first inequality we used the upper bound
$\|  \hat{\bf K}_t(x,x')\|
= \|  \partial_\theta  {\bf f}( x, \hat{\theta}(t) )  \partial_\theta  {\bf f}( x, \hat{\theta}(t) )^T \|
\le L_1^2$ by (C1). 
Using the argument by integrating over $t$ again, we have
\begin{align*}
\| \lambda( \hat{\bf u } - \bar{\bf u })(\cdot, t) \|_{{p}} 
& \le 2 \lambda  \int_0^t       ( 2 L_1^2  \lambda s +   1 ) g(s)   ds \\
& \le 
2 \lambda t      ( 2 L_1^2  \lambda t +   1 ) g(t)  
\quad \text{($ ( 2 L_1^2  \lambda t +   1 ) g(t)$ is monotone in $t$)}
%& = 2L_1^2 2C_2  \left(  \frac{4}{15}   \lambda^{3/2} t^{5/2}   \|{\bf f}^*\|_p 
%	 + ( \frac{(\lambda t)^2}{2}   +  \frac{(\lambda t)^3}{6}) \sqrt{\frac{ \log n + \log M_\Theta}{n}}  \right) 
%+ G(t) \\
%& =  2C_2 
% \left(  (\frac{8L_1^2}{15}   \lambda^{3/2} t^{5/2} + \frac{2}{3}  \lambda^{1/2} t^{3/2} )   \|{\bf f}^*\|_p  
%	 + (\lambda t +  \frac{(2 L_1^2 +1)}{2}(\lambda t)^2   +  \frac{L_1^2}{3}(\lambda t)^3 ) \sqrt{\frac{ \log n + \log M_\Theta}{n}}  \right)
\end{align*}
Inserting the definition of $g(t)$ in \eqref{eq:def-g(t)-growth}, we have 
\begin{align*}
\| \lambda( \hat{\bf u } - \bar{\bf u })(\cdot, t) \|_{{p}} 
& \le  2 C_{2,1} 
(   1 + 2 L_1^2  \lambda t  )  \left(   \lambda^{1/2} t^{3/2} \|{\bf f}^*\|_p + (\lambda t + 1)  \lambda t  \sqrt{\frac{ \log n + \log M_\Theta}{n}}  \right), 
\end{align*}
which proves \eqref{eq:bound-prop-hatu-baru} by defining $C_2:= 2 C_{2,1}\max\{1,  2L_1^2\}  $.
By the definition of $C_{2,1}$ in \eqref{eq:def-const-C21} and $C_{2,2}$ in \eqref{eq:def-const-C22}, 
\[
C_2 \le c(L_1, \cdots, L_4) (1+ b_q +  \| {\bf f}^*\|_p),
\]
where $c(L_1, \cdots, L_4)$  an $O(1)$ constant determined by $L_1, L_2, L_3, L_4$.
\end{proof}

\begin{proof}[Proof of Theorem \ref{thm:NTK-stein-finite-sample}]
Similarly as in the proof of Theorem \ref{thm:NTK-stein}, the required largeness of $\lambda$ guarantees that the range of $t_0$ in \eqref{eq:cond-t0-thm-finite-sample} is non-empty.
Proposition \ref{prop:NTK-stein} applies same as before,
and $t_0 \ge 1$ ensures that the bound of $\| \lambda \bar{\bf u}(  \cdot ,t)  - {\bf f}^* \|_{p }$ as in \eqref{eq:2eps-bound} holds. 

The upper bound of $t_0$ in \eqref{eq:cond-t0-thm-finite-sample} ensures that $t$ falls in the needed range by Proposition \ref{prop:hatu-baru};
We consider large $n$ as required by Proposition \ref{prop:hatu-baru},
and, by Remark \ref{rk:n-C-prop-finite-sample}, 
the three thresholds $ n_\lambda, n_5, n_{\lambda,t}$ are determined by $\lambda$, $ \|{\bf f}^*\|_p$, $\log M_\Theta$ and $\lambda t = t_0 \frac{ \log (1/\epsilon) }{\delta}$. 
We also consider being under the intersection of the good events in Proposition \ref{prop:hatu-baru}, which happens w.p. $\ge 1-7 n^{-10}$, then \eqref{eq:bound-prop-hatu-baru} holds. 

Putting together \eqref{eq:2eps-bound} and \eqref{eq:bound-prop-hatu-baru}, by triangle inequality, we have
\[
\| \lambda \hat{\bf u}( x ,t)  - {\bf f}^* \|_{p }  
\le 2 \epsilon  \| {\bf f}^* \|_{p } 
+ C_2 
(   1 +   \lambda t  )\left(   \lambda^{1/2} t^{3/2} \|{\bf f}^*\|_p + (1+\lambda t )  \lambda t  \sqrt{\frac{ \log n + \log M_\Theta}{n}}  \right). 
\]
Then \eqref{eq:2eps-bound-thm-finite-sample} follows by the equality $\lambda t = t_0 \frac{ \log (1/\epsilon) }{\delta}$ and defining
\begin{align*}
\kappa_1:= (1+ t_0 \frac{ \log (1/\epsilon) }{\delta}) \left( t_0 \frac{ \log (1/\epsilon) }{\delta} \right)^{3/2},
\quad
\kappa_2 :=(1+t_0 \frac{ \log (1/\epsilon) }{\delta})^2 t_0\frac{ \log (1/\epsilon) }{\delta}.
\end{align*}
\end{proof}

\subsection{{Proof in Section \ref{sec:gof_test}}}

\begin{proof}[Proof of  Corollary \ref{cor:GoF-test-power}]
Applying Theorem \ref{thm:NTK-stein-finite-sample} gives that, under $E_{\rm tr}$,
$ \| \lambda \hat{\bf f} - {\bf f}^* \|_p $ is upper bounded by the r.h.s. of \eqref{eq:2eps-bound-thm-finite-sample} where setting $n = n_{\rm tr}$.
The condition of the corollary further implies that
\[
 \| \lambda \hat{\bf f} - {\bf f}^* \|_p  < 0.9 \| {\bf f}^* \|_{p }.
\]
Then, 
\[
\langle \lambda \hat{\bf f} , {\bf f}^*\rangle - \| {\bf f}^* \|_p^2 
=
\langle \lambda \hat{\bf f} - {\bf f}^*, {\bf f}^*\rangle
\ge - \| \lambda \hat{\bf f} - {\bf f}^* \|_p \| {\bf f}^* \|_p
> - 0.9  \| {\bf f}^* \|_p^2,
\]
which gives that $\langle \lambda \hat{\bf f} , {\bf f}^*\rangle >  0.1 \| {\bf f}^* \|_p^2$.
Meanwhile, by that $ \hat{\bf f} \in \calF_0(p) \cap L^2(p)$ (Assumption \ref{assump:f(x,theta)-L2-F0}), 
\eqref{eq:SD-as-innerproduct} gives that 
\begin{equation}\label{eq:bound-mean-gof-power-pf}
\mathbb{E}_{x\sim p}  T_q \hat{\bf f}(x) 
= \langle {\bf f}^*, \hat{\bf f} \rangle_p 
> \frac{0.1}{\lambda} \| {\bf f}^* \|_p^2.
\end{equation}

Next, we bound the deviation of $\hat{T}$ and $\hat{T}_{\rm null}$ by concentration.
By definition,
\[
\hat{T} = \E_{x \sim \hat{p}} T_q \hat{\bf f}(x), \quad \hat{p} := \frac{1}{n_{\rm GoF}} \sum_{i=1}^{n_{\rm GoF}} \delta_{x_i}.
\]
Recall that $\hat{\theta}(t) \in B_r$ by Lemma \ref{lemma:hatthetat-theta0}, then, by Assumption \ref{assump:C3-C5}(C5)(C6),
\[
|T_q \hat{\bf f}(x_i) | 
\le | {\bf s}_q \cdot \hat{\bf f}(x_i) |+  |\nabla \cdot \hat{\bf f}(x_i) |
\le b_{(0)} b_q + b_{(1)} =: b.
\]
By the Hoeffding's inequality, Lemma \ref{lemma:Hoeffding}, for any $t > 0$, denote $n_{\rm GoF}$ by $n$, 
\[
\Pr [ 
\E_{x \sim \hat{p}} T_q \hat{\bf f}(x) - \E_{x \sim {p}} T_q \hat{\bf f}(x) \le -t 
| E_{\rm tr}]
\le \exp\{- \frac{ n t^2}{ 2 b^2} \}.
\]
By union bound, and the condition that $7n_{\rm tr}^{-10} < 0.1 \min\{\alpha,  \beta\}$,
\[
\Pr [ 
\E_{x \sim \hat{p}} T_q \hat{\bf f}(x) - \E_{x \sim {p}} T_q \hat{\bf f}(x) \le -t ]
\le \exp\{- \frac{ n t^2}{ 2 b^2} \} + 7n_{\rm tr}^{-10} 
< \exp\{- \frac{ n t^2}{ 2 b^2} \} + 0.1 \beta.
\]
This gives that w.p. $\ge 1-\beta$, 
\[
\hat{T}
\ge  \E_{x \sim {p}} T_q \hat{\bf f}(x) - t_1, \quad t_1:= \sqrt{2} b \sqrt{ \frac{  \log \frac{1}{0.9 \beta}  }{n_{\rm GoF}}}.
\]
Similarly, 
\[
\hat{T}_{\rm null} 
= \E_{y \sim \hat{q}} T_q \hat{\bf f}(y), \quad \hat{q} := \frac{1}{n_{\rm GoF}} \sum_{i=1}^{n_{\rm GoF}} \delta_{y_i},
\]
and w.p. $\ge 1-\alpha$,
 \[
\hat{T}_{\rm null}
\le  t_2, \quad t_2:= \sqrt{2} b \sqrt{ \frac{  \log \frac{1}{0.9 \alpha}  }{n_{\rm GoF}}}.
\]
We set $t_2$ to be the test threshold, 
combined with \eqref{eq:bound-mean-gof-power-pf}, we know that the target  significance level and power can be achieved if
$ t_1 + t_2 \le  \frac{0.1}{\lambda} \| {\bf f}^* \|_p^2$. 
Together with that $\sqrt{ \log \frac{1}{0.9 \alpha} }\le 0.5 + \sqrt{ \log \frac{1}{\alpha}}$ and similarly for $\sqrt{\log \frac{1}{0.9 \beta} }$,
this proves the corollary.
\end{proof}

\section*{Acknowledgement}

The work is supported by NSF DMS-2134037. 
M.R. and Y.X. are partially supported by an NSF CAREER CCF-1650913, and NSF DMS-2134037, CMMI-2015787, CMMI-2112533, DMS-1938106, and DMS-1830210. X.C. is also partially supported by NSF DMS-1818945, NIH R01GM131642 and the Alfred P. Sloan Foundation.

\bibliographystyle{plain}
\bibliography{main}

\appendix

\setcounter{figure}{0} \renewcommand{\thefigure}{A.\arabic{figure}}
\setcounter{table}{0} \renewcommand{\thetable}{A.\arabic{table}}
\setcounter{equation}{0} \renewcommand{\theequation}{A.\arabic{equation}}
\setcounter{remark}{0} \renewcommand{\theremark}{A.\arabic{remark}}

\section{Additional technical lemmas and proofs}\label{app:add-proofs}

\subsection{Proofs of technical lemmas}
\begin{proof}[Proof of Lemma \ref{lemma:sol-SD-L2}]
For any ${\bf f} \in \calF_0(p)$, because $p {\bf f} |_{\partial \calX} =0$, we have
\[
\int_\calX p \nabla \cdot {\bf f}
=  - \int_\calX \nabla p \cdot {\bf f}
\]
Then for any ${\bf f} \in L^2(p) \cap \calF_0(p)$, 
\begin{align}
\mathbb{E}_{x\sim p}  T_q {\bf f}(x) 
&= \int_{\calX}  p ( {\bf s}_q \cdot {\bf f} + \nabla \cdot {\bf f} ) 
= \int_{\calX}  p {\bf s}_q \cdot {\bf f} - \int_\calX \nabla p  \cdot {\bf f}  \nonumber \\
& =\int_{\calX}  p ( {\bf s}_q  - {\bf s}_p )\cdot {\bf f} 
= \E_{x \sim p }  {\bf f}^* \cdot {\bf f}  = \langle {\bf f}^*, {\bf f} \rangle_p. \label{eq:SD-L2-objective-pf}
\end{align}
Note that the integrals are well-defined because both ${\bf s}_p$ and ${\bf s}_q$ are in $L^2(p)$.
By \eqref{eq:SD-L2-objective-pf}, whenever $ \| {\bf f} \|_{p} \le r$, 
\begin{equation}\label{eq:bound-SD-f-pf}
\mathbb{E}_{x\sim p}  T_q {\bf f}(x)  \le \| {\bf f}^* \|_{p}  \| {\bf f} \|_{p} 
\le r \| {\bf f}^* \|_{p},
\end{equation}
where $\| {\bf f}^* \|_{p} < \infty$ because $ {\bf s}_q$ and  ${\bf s}_p$ are in $L^2(p)$.
This proves that ${\rm SD}_{r} (p,q)$ is upper bounded by $r \| {\bf f}^* \|_{p}$
and thus is finite. 
To show that ${\rm SD}_{r} (p,q) = r \| {\bf f}^* \|_{p}$, 
first
when ${\bf f}^* = 0$ in $L^2(p)$ then  ${\rm SD}_{r} (p,q) = 0$ by \eqref{eq:bound-SD-f-pf}.
If $ \|   {\bf f}^* \|_{p } > 0$, 
by property of the inner product, 
the ${\bf f}$ that is a positive multiple of ${\bf f}^*$ and has
$ \| {\bf f} \|_{p} = r $ achieves the supremum of value $r \| {\bf f}^* \|_{p}$.
{Because ${\bf f}^* \in \calF_0(p)$, 
this maximizer $ {\bf f}  =  ( { r }/{  \|   {\bf f}^* \|_{p } }) {\bf f}^* $ is also in $\calF_0(p)$.}
\end{proof}
\begin{proof}[Proof of Lemma \ref{lemma:GD-theta-eqn}]

By the definition of $L_\lambda$ in \eqref{eq:def-L_lambda-theta} and \eqref{eq:Llambda-is-MSE-const}, {the latter applies to ${\bf f} = {\bf f}(\cdot, \theta)$ because ${\bf f}(\cdot, \theta)$ is in $L^2(p) \cap \calF_0(p)$ by Assumption \ref{assump:f(x,theta)-L2-F0}},
\[
L_\lambda( \theta) 
  = - \langle {\bf f}^*, {\bf f}(\cdot, \theta) \rangle_p + \frac{\lambda}{2} \| {\bf f} (\cdot, \theta) \|_p^2
  = \E_{x \sim p } \left(  - {\bf f}^*(x) \cdot    {\bf f}(x, \theta) + \frac{\lambda}{2} \| {\bf f}(x, \theta)\|^2 \right).
\]
Then we have that 
\begin{align*}
\frac{\partial}{\partial \theta}L_\lambda( \theta) 
& = \E_{x \sim p } \big(
- {\bf f}^*(x) \cdot \partial_\theta {\bf f}(x, \theta)   
+  \lambda   { \bf f}( x, \theta ) \cdot  \partial_\theta {\bf f}(x, \theta)    \big) \\
& =\E_{x \sim p } \partial_\theta {\bf f}(x, \theta) \cdot \big(
  \lambda   { \bf f}( x, \theta )    - {\bf f}^*(x)  \big).
\end{align*}
This proves \eqref{eq:GD-theta-eqn} according to the GD dynamic defined in \eqref{eq:def-GD-theta}.
\end{proof}

\begin{proof}[Proof of Lemma \ref{lemma:u(x,t)-dynamic}]
Inserting \eqref{eq:GD-theta-eqn} to \eqref{eq:dudt-1} gives that 
\begin{equation}
 \frac{ \partial}{\partial t}{\bf u }( x, t) 
 = - \langle \partial_\theta {\bf f}( x, \theta(t) ),  \E_{x' \sim p} \partial_\theta {\bf f}(x',\theta (t)) \cdot \big(  \lambda  {\bf u}( x' , t) - {\bf f}^*(x')  \big)
 \rangle_\Theta,
\end{equation}
where we used the definition \eqref{eq:def-u(x,t)}.
The lemma then follows by the definition \eqref{eq:def-NTK-t} and the linearity of inner-product $ \langle \cdot, \cdot  \rangle_\Theta $.
\end{proof}

\begin{proof}[Proof of Lemma \ref{lemma:hatu-dynamic}]
By  \eqref{eq:def-hatL_lambda},
\begin{equation}
\partial_\theta \hat{L}_\lambda(\theta) = 
 \E_{x\sim \hat{p}} \left(  
 \partial_\theta  {\bf f}(x, \theta) \cdot  (   \lambda   {\bf f}(x, \theta) - {\bf s}_q (x) )
  - \nabla_x \cdot  \partial_\theta {\bf f} (x, \theta)
   \right),
\end{equation}
and combined with \eqref{eq:def-GD-theta-n}\eqref{eq:partialt_hatu_1}, we have
\begin{align*}
 {\partial_t} \hat{\bf u }( x, t) 
%& = \langle \partial_\theta {\bf f}( x, \hat{\theta}(t) ), \dot{ \hat{\theta}}(t) \rangle_\Theta \\
& = - \langle \partial_\theta {\bf f}( x, \hat{\theta}(t) ), 
 \partial_\theta \hat{L}_\lambda( \hat{\theta}(t)) 
 \rangle_\Theta \\
&=  - \left\langle \partial_\theta {\bf f}( x, \hat{\theta}(t) ), 
 \E_{x' \sim \hat{p}} \left(  
 \partial_\theta  {\bf f}(x', \hat{\theta}(t)) \cdot  (  \lambda  \hat{\bf u}(x', t) - {\bf s}_q (x') )
  - \nabla_{x'} \cdot  \partial_\theta {\bf f} (x', \hat{\theta}(t))
   \right)
    \right\rangle_\Theta \\
& =  -  \E_{x' \sim \hat{p}} 
\left( 
\langle \partial_\theta {\bf f}( x, \hat{\theta}(t) ),   
 \partial_\theta  {\bf f}(x', \hat{\theta}(t)) 
\rangle_\Theta 
 \circ (  \lambda  \hat{\bf u}(x', t) - {\bf s}_q (x') )
  - \langle \partial_\theta {\bf f}( x, \hat{\theta}(t) ), \nabla_{x'} \cdot  \partial_\theta {\bf f} (x', \hat{\theta}(t))
    \rangle_\Theta
    \right).
\end{align*}
This proves the lemma by the definition of $ \hat{\bf K}_t(x,x') $ in \eqref{eq:def-NTK-t-n}.
\end{proof}

\subsection{Concentration lemmas}

The following lemma can be derived from Classical Hoeffding's inequality.
\begin{lemma}[Hoeffding]\label{lemma:Hoeffding}
Suppose $X_1, \cdots, X_n$ are independent random variables, $|X_i| \le L$. Then for any $t \ge 0$, 
\[
\Pr [  \frac{1}{n}\sum_{i=1}^n (X_i - \E X_i)  \ge t ], \,
\Pr [  \frac{1}{n}\sum_{i=1}^n (X_i - \E X_i)  \le -t ],
\le \exp\{ - \frac{ n t^2}{  2 L^2 }\}.
\]

%Suppose $X_1, \cdots, X_n$ are independent random variables which are bounded, $a_i \le X_i \le b_i$. 
%Let $S_n = \sum_{i=1}^n X_i$, then
%for any $\alpha \ge 0$, 
%\[
%\Pr [  S_n - \E S_n  > \alpha ], \,
% \Pr [  S_n - \E S_n  < - \alpha ] \le \exp\{ - \frac{2 \alpha^2}{ \sum_{i=1}^n (b_i-a_i)^2}\}.
%\]
\end{lemma}

The following lemma can be directly derived from \cite[Theorem 6.1.1]{tropp2015introduction}.
\begin{lemma}[Matrix Bernstein \cite{tropp2015introduction}]\label{lemma:matrix-bernstein}
Let $X_i$ be a sequence of $n$ independent, random, real-valued matrices of size $d_1$-by-$d_2$. 
Assume that $\E X_i = 0$ and $\| X_i \| \le L$ for each $i$, and  $\nu > 0$ be such that
\[
\| \frac{1}{n} \sum_{i=1}^n  \E X_i X_i^T\| , \, \| \frac{1}{n} \sum_{i=1}^n  \E X_i^T X_i\|   \le \nu.
\]
Then, for any $t \ge 0$,
\[
\Pr [\| \frac{1}{n} \sum_{i=1}^n X_i \| \ge t] \le (d_1 + d_2) \exp\{ - \frac{ nt^2}{ 2(\nu + L t/3)} \}.
\]
\end{lemma}

\section{{Additional experiments}}

\subsection{{Supplementary Tables}}\label{app:supp_tables}

\begin{table}[h!]
    \centering
    \begin{tabular}{r|c c|c c}
    \hline
    \rule{0pt}{10pt}Regularization ($\lambda$) & Mean Power & Epoch & Mean Selected $\widehat{\rm MSE}_q$ & Mean Stopping Epoch  \\
    \hline
    \rule{0pt}{10pt}$1\mathrm{e}-3$ & 0.814 $\pm 0.018$ & 60 & 0.041 & 59.5 \\ 
    $1\mathrm{e}-2$ & 0.808 $\pm 0.027$ & 50 & 0.027 & 49.4 \\ 
    $1\mathrm{e}-1$ & 0.825 $\pm 0.040$ & 30 & 0.021 & 27.8 \\ 
    $1\mathrm{e}-0$ & 0.835 $\pm 0.025$ & 20 & 0.025 & 21.8 \\ 
    $\Lambda(1\mathrm{e}0, 5\mathrm{e}-2, 0.90)$ & 0.839 $\pm 0.044$ & 30 & 0.023 & 31.5 \\
    $\Lambda(1\mathrm{e}0, 5\mathrm{e}-2, 0.95)$ & 0.838 $\pm 0.029$ & 30 & 0.020 & 34.1 \\
    \hline
    \end{tabular}
    \caption{
    For the mixture data in 2D as in \eqref{eq:data_distribution_gm} of Section~\ref{sec:gm_sim_setup}, we consider the approximate epoch at which the {monitor $\widehat{\rm MSE}^{(m)}_p$} \eqref{eq:validation_mse} is minimized. At this epoch of training, we display the average GoF testing power ($n_{\rm GoF}=75$) and $\widehat{\rm MSE}_q$ \eqref{eq:mse} over $n_{\rm replica}=10$ models for each regularization strategy.}
    \label{tab:2D_experiment}
\end{table}

\begin{table}[h!]
    \centering
    \begin{tabular}{r|c c|c c}
    \hline
    \rule{0pt}{10pt}Regularization ($\lambda$) & Mean Power & Epoch & Mean Selected $\widehat{\rm MSE}_q$ & Mean Stopping Epoch  \\
    \hline
    \rule{0pt}{10pt}$2.500\mathrm{e}-4$ & 0.839 $\pm 0.034$ & 50 & 0.120 & 52.4 \\ 
    $1.000\mathrm{e}-3$ & 0.894 $\pm 0.027$ & 50 & 0.105 & 50.1 \\ 
    $4.000\mathrm{e}-3$ & 0.910 $\pm 0.014$ & 45 & 0.114 & 46.0 \\ 
    $1.600\mathrm{e}-2$ & 0.893 $\pm 0.025$ & 30 & 0.127 & 34.1 \\ 
    $6.400\mathrm{e}-2$ & 0.905 $\pm 0.013$ & 20 & 0.128 & 21.5 \\ 
    $2.560\mathrm{e}-1$ & 0.878 $\pm 0.019$ & 12 & 0.132 & 12.8 \\ 
    $1.024\mathrm{e}-0$ & 0.824 $\pm 0.037$ & 12 & 0.172 & 12.4 \\ 
    $\Lambda(5\mathrm{e}-1, 1\mathrm{e}-3, 0.80)$ & 0.928 $\pm 0.007$ & 50 & 0.088 & 49.8 \\
    $\Lambda(5\mathrm{e}-1, 1\mathrm{e}-3, 0.85)$ & 0.920 $\pm 0.021$ & 45 & 0.095 & 40.1 \\
    \hline
    \end{tabular}
    \caption{For the mixture data in 10D as in \eqref{eq:data_distribution_gm}, we consider the epoch that minimizes {monitor $\widehat{\rm MSE}^{(m)}_p$}, at which point we display the average testing power (${n_{\rm GoF}}=200$) and $\widehat{\rm MSE}_q$ over $n_{\rm replica}=10$ models for each regularization strategy.}
    \label{tab:10D_experiment}
\end{table}

\begin{table}[h!]
    \centering
    \begin{tabular}{r|c c|c c}
    \hline
    \rule{0pt}{10pt}Regularization ($\lambda$) & Mean Power & Epoch & Mean Selected $\widehat{\rm MSE}_q$ & Mean Stopping Epoch  \\
    \hline
    \rule{0pt}{10pt}$2.500\mathrm{e}-4$ & 0.862 $\pm0.103$ & 45 & 0.260 & 37.9 \\ 
    $1.000\mathrm{e}-3$ & 0.962 $\pm0.021$ & 45 & 0.262 & 42.1 \\ 
    $4.000\mathrm{e}-3$ & 0.925 $\pm0.040$ & 25 & 0.276 & 23.9 \\ 
    $1.600\mathrm{e}-2$ & 0.444 $\pm0.111$ & 10 & 0.301 & 9.3 \\ 
    $6.400\mathrm{e}-2$ & 0.243 $\pm0.050$ & 4 & 0.310 & 3.8 \\ 
    $\Lambda(4\mathrm{e}-1, 5\mathrm{e}-4, 0.80)$ & 0.998 $\pm0.002$ & 45 & 0.208 & 46.2 \\
    $\Lambda(4\mathrm{e}-1, 5\mathrm{e}-4, 0.85)$ & 1.000 $\pm0.001$ & 50 & 0.190 & 49.1 \\
    $\Lambda(4\mathrm{e}-1, 5\mathrm{e}-4, 0.90)$ & 0.995 $\pm0.002$ & 55 & 0.202 & 54.4 \\
    \hline
    \end{tabular}
    \caption{For the mixture data in 25D as in \eqref{eq:data_distribution_gm}, we consider the epoch that minimizes {monitor $\widehat{\rm MSE}^{(m)}_p$}, at which point we display the average testing power (${n_{\rm GoF}}=500$) and $\widehat{\rm MSE}_q$ over $n_{\rm replica}=10$ models for each regularization strategy.}
    \label{tab:25D_experiment}
\end{table}

\subsection{{1D Gaussian mixture}}\label{sec:1d_experiment}

Consider a setting in which $q$ is a 1-dimensional Gaussian mixture of two equally-weighted, unit variance components at positions -1 and 1. Let $p$ be a similar mixture, however at positions -0.8 and 1, and the component at 1 has variance 0.25. Therefore, the score of $q$ has the form:
\begin{align*}
    {\bf s}_q(x) =\nabla q(x) / q(x) = &\left(
    \frac{1}{\sqrt{2\pi}} \exp \left( -\frac{(x+1)^2}{2} \right)
    + \frac{1}{\sqrt{2\pi}} \exp \left( -\frac{(x-1)^2}{2} \right)  
    \right) ^ {-1}
    \times \\
    &\left(
    -\frac{x+1}{\sqrt{2\pi}} \exp \left( -\frac{(x+1)^2}{2} \right)
    - \frac{x-1}{\sqrt{2\pi}} \exp \left( -\frac{(x-1)^2}{2} \right)  
    \right),
\end{align*}
where ${\bf s}_p$ has similar form. Therefore, the scaleless optimal critic ${\bf f}^*={\bf s}_q - {\bf s}_p$ \eqref{eq:def-fstar} has analytical form:
\begin{equation}
    \begin{aligned}
        {\bf f}^*(x) = &\left(
        \frac{1}{\sqrt{2\pi}} \exp \left( -\frac{(x+1)^2}{2} \right)
        + \frac{1}{\sqrt{2\pi}} \exp \left( -\frac{(x-1)^2}{2} \right)  
        \right) ^ {-1}
        \times \\
        &\left(
        -\frac{x+1}{\sqrt{2\pi}} \exp \left( -\frac{(x+1)^2}{2} \right)
        - \frac{x-1}{\sqrt{2\pi}} \exp \left( -\frac{(x-1)^2}{2} \right)  
        \right) - \\
        &\left(
        \frac{1}{\sqrt{2\pi}} \exp \left( -\frac{(x+0.8)^2}{2} \right)
        + \frac{1}{(0.5)\sqrt{2\pi}} \exp \left( -\frac{(x-1)^2}{2\times(0.5)^2} \right)  
        \right) ^ {-1}
        \times \\
        &\left(
        -\frac{x+1}{\sqrt{2\pi}} \exp \left( -\frac{(x+0.8)^2}{2} \right)
        - \frac{x-1}{(0.5)^3\sqrt{2\pi}} \exp \left( -\frac{(x-1)^2}{2\times(0.5)^2} \right)  
        \right),
    \end{aligned}
    \label{eq:analyic_f_star}
\end{equation}
which can be computed in a similar analytical manner in higher dimension given distributions of the form \eqref{eq:data_distribution_gm}.

A neural Stein critic (MLP of 2 hidden layers with 512-node width) is trained using staged $\Lambda(1,10^{-3},0.9)$ regularization in this setting given 1,000 training samples from $q$ with batch size 200, learning rate $10^{-3}$, and $B_w=5$.

\begin{figure}[t]
    \centering
    \includegraphics[width=0.5\textwidth]{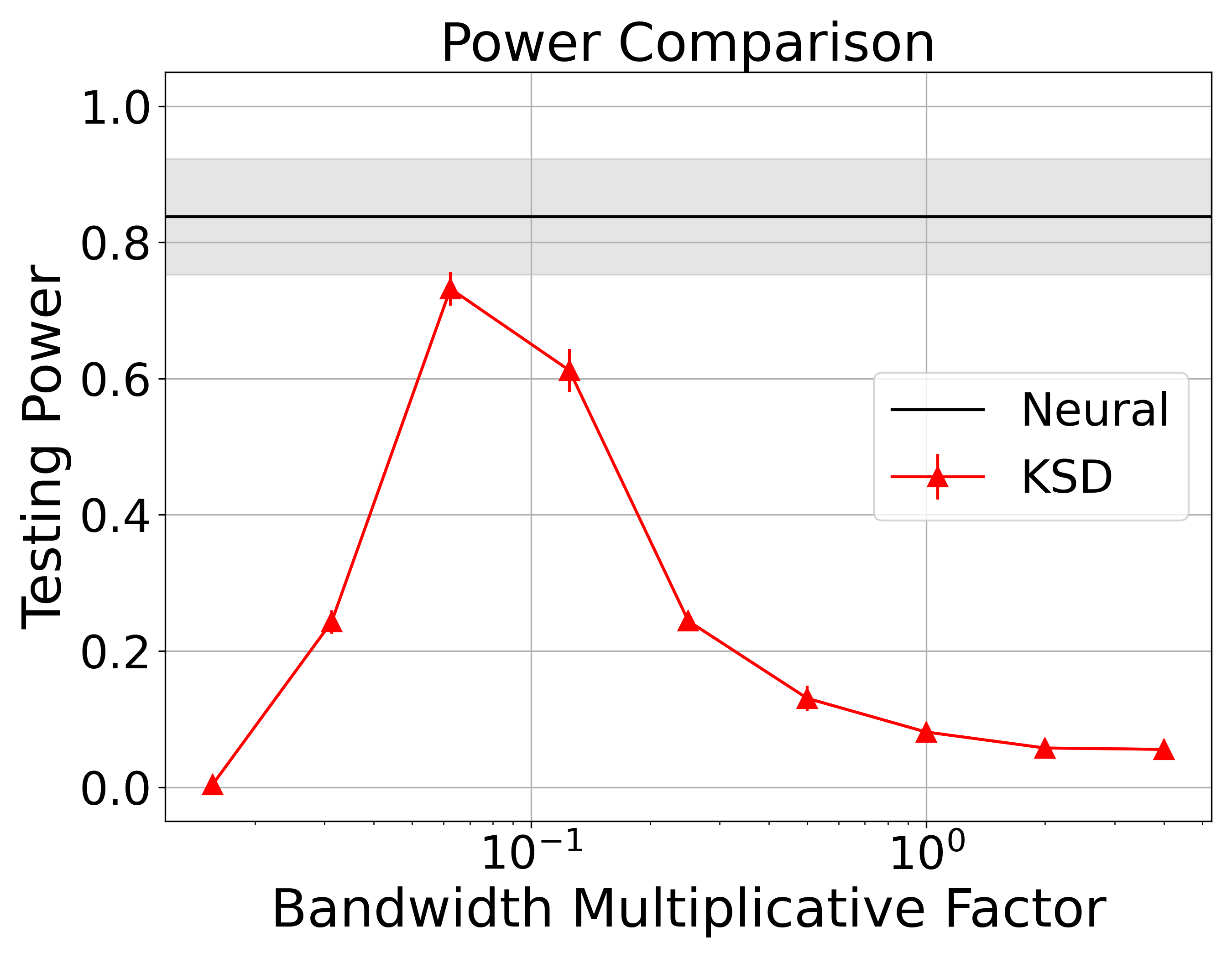}
    \caption{
    Comparison between the testing power of the KSD GoF test in the 50D setting described in Section~\ref{sec:kernel_experiments} for a range of bandwidth choices with $n_{\rm sample}=1500$. The horizontal axis indicates the scaling factor $\delta$ of the kernel bandwidth (see Appendix \ref{sec:ksd_gof_test}) and the vertical axis displays the test power. The red line and error bars display the mean and standard deviation of the testing power for each choice of kernel bandwidth, and the black line and shaded region indicate the mean and standard deviation of the testing power of the neural Stein GoF test in this setting.}\label{fig:ksd_bandwidth_comparison}
\end{figure}

\subsection{Kernel Stein discrepancy goodness-of-fit test}\label{sec:ksd_gof_test}

{As in Section \ref{sec:gof_test}, consider the setting in which we are provided with probability distribution $q$ supported on $\calX\subset\mathbb{R}^d$ from which we can sample, and consider the scenario wherein we are provided with a finite sample of data $x_i\sim p$ of ${n_{\rm GoF}}$ samples. As an alternative to learning the Stein critic function using a neural network for estimation of the test statistic in \eqref{eq:test_stat}, one may define the function space $\mathcal{F}$ to be an RKHS defined by a kernel denoted $k(\cdot, \cdot)$. In this case, the KSD admits computation in closed form by the following relation:}
\begin{equation}
    {\rm KSD} = \mathbb{E}_{x,x'\sim p}[u_q(x,x')],
    \label{eq:ksd}
\end{equation}
where
\begin{align}
    u_q(x,x') =&\  \mathbf{s}_q(x)^\text{T} k(x,x') \mathbf{s}_q(x') + \mathbf{s}_q(x)^\text{T} \nabla_{x'} k(x,x') \nonumber\\&   + \nabla_x k(x,x')^\text{T}\mathbf{s}_q(x') + {\rm tr}(\nabla_{x,x'} k(x, x')).
    \label{eq:ksd_u_func}
\end{align}
{Using this formulation, \cite{chwialkowski2016kernel} defines a quadratic-time $V$-statistic which estimates the KSD as follows:}
\begin{equation}
    \hat{V}_q = \frac{1}{n_{\rm GoF}^2} \sum_{i,j=1}^{{n_{\rm GoF}}} u_q(x_i, x_j).
    \label{eq:ksd_test_stat}
\end{equation}
{Given the test statistic \eqref{eq:ksd_test_stat}, a ``wild bootstrap'' is used to approximate the distribution of the test statistic under the null hypothesis, which is outlined in \cite{chwialkowski2016kernel}. Given the test set of ${n_{\rm GoF}}$ samples $x_i\sim p$, the null hypothesis may be rejected if \eqref{eq:ksd_test_stat} computed using these $x_i$ exceeds the ($1-\alpha$) quantile of the wild bootstrapped distribution under the null hypothesis (where, again, $\alpha$ is selected to tune the Type-I error). The power is then estimated using $n_{\rm run}$ number of such GoF tests.}

\begin{figure}[h]
    \centering
    \includegraphics[width=\textwidth]{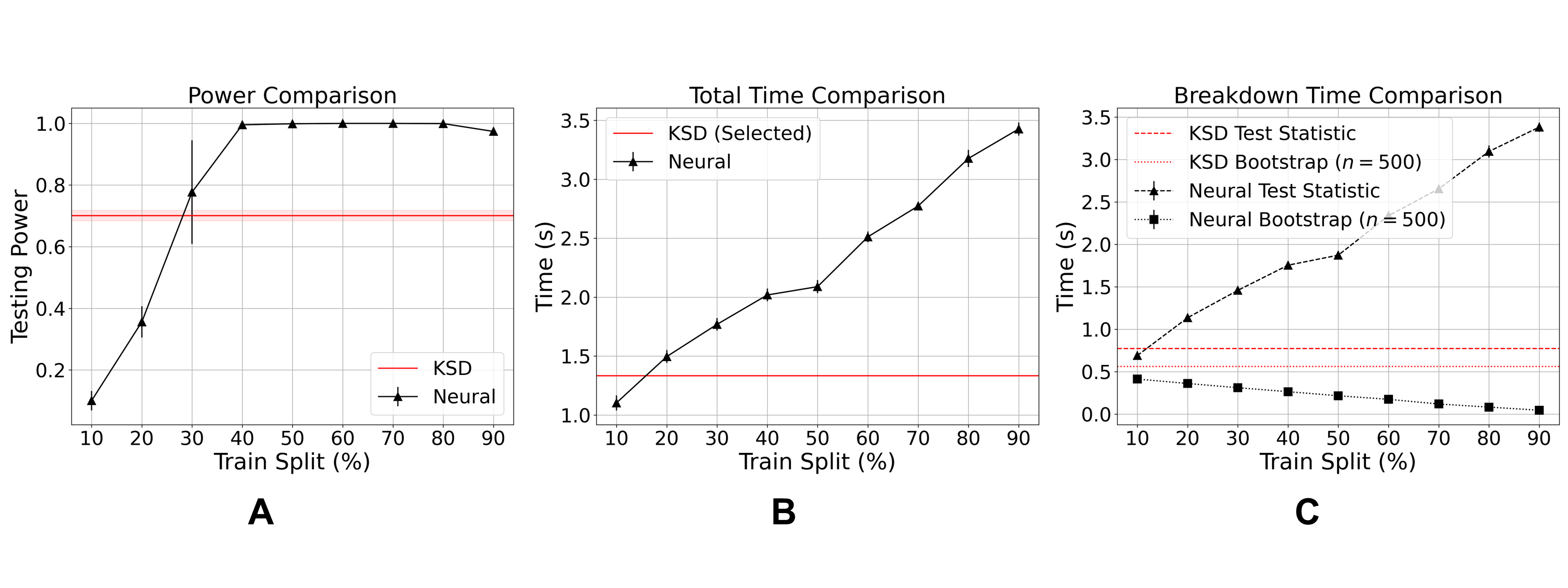}
        \vspace{-30pt}
    \caption{ Comparison between neural Stein critic GoF test in the 50D setting described in Appendix~\ref{app:neural_split_power} for various proportions of training/testing split, where the testing procedure and power/time computation is otherwise the same as that in Section~\ref{sec:kernel_experiments}. In all cases, the total number of points is $n_{\rm sample}=500$. The power and duration of the KSD test using the selected bandwidth according to Appendix~\ref{sec:ksd_gof_test} for $n_{\rm sample}=500$ is shown in red in each figure. The x-axis of each plot refers to the percentage of these points which are allocated to training for the neural network test. The remainder are used for computation of the neural Stein test statistic. The y-axes of (\textbf{A}), (\textbf{B}), and (\textbf{C}) correspond to those of Figure~\ref{fig:kernel_test_comparison}. 
   }\label{fig:neural_split_power}
\end{figure}

{Following the approaches of \cite{chwialkowski2016kernel,liu2016kernelized}, we use an RBF kernel to define the RKHS used for KSD in our experiments. This kernel has the following form:}
\begin{equation}
    k(x,x') = \exp\left(-\frac{\|x-x'\|_2^2}{2\sigma^2}\right) = \exp\left(-\gamma\|x-x'\|_2^2\right),
\end{equation}
{where $\sigma$ is the kernel bandwidth and $\gamma=1/(2\sigma^2)$. A standard heuristic we employ is the choice of the median Euclidean distance between the data considered as the bandwidth $\sigma$ of this kernel. However, we also conduct an experiment to select a bandwidth which has the capacity to achieve higher testing power when used in the KSD GoF test. We consider bandwidths which are scaled versions of the median data distances heuristic bandwidth. That is, we compute the testing power for bandwidth factors $\delta\in\{2^{-6}, 2^{-5}, 2^{-4}, 2^{-3}, 2^{-2}, 2^{-1}, 2^{0}, 2^{1}, 2^{2}\}$, where the scaling follows $\gamma'=1/(2\delta\sigma^2)$ for $\sigma$ equal to the median Euclidean distance between data samples. 
To do so, we consider the same model and data distributions as in Section~\ref{sec:kernel_experiments}. We then fix $n_{\rm sample}=1500$, and compute the KSD test power from $n_{\rm run}=400$ GoF tests using  $n_{\rm boot}=500$ bootstrapped samples in each test. This is conducted using $n_{\rm replica}=5$ replicas to generate a mean and standard deviation of power computed for each choice of bandwidth. Finally, we compare this to the GoF testing power computed from $n_{\rm replica}=5$ neural Stein critics trained using $n_{\rm sample}=1500$ samples split into a $50\%/50\%$ train/test split (where $20\%$ of the train split is used for validation). The result can be seen in Figure~\ref{fig:ksd_bandwidth_comparison}. This shows that the factor of $\delta=2^{-4}$ maximizes power in this case. Therefore, this factor is chosen to create the ``best'' selected bandwidth for the experiments in Section~\ref{sec:kernel_experiments}. This figure also displays that the KSD GoF test power is still less than that of the neural Stein GoF test, despite better choices of bandwidth.}

\begin{figure}
    \centering
    \includegraphics[width=0.5\textwidth]{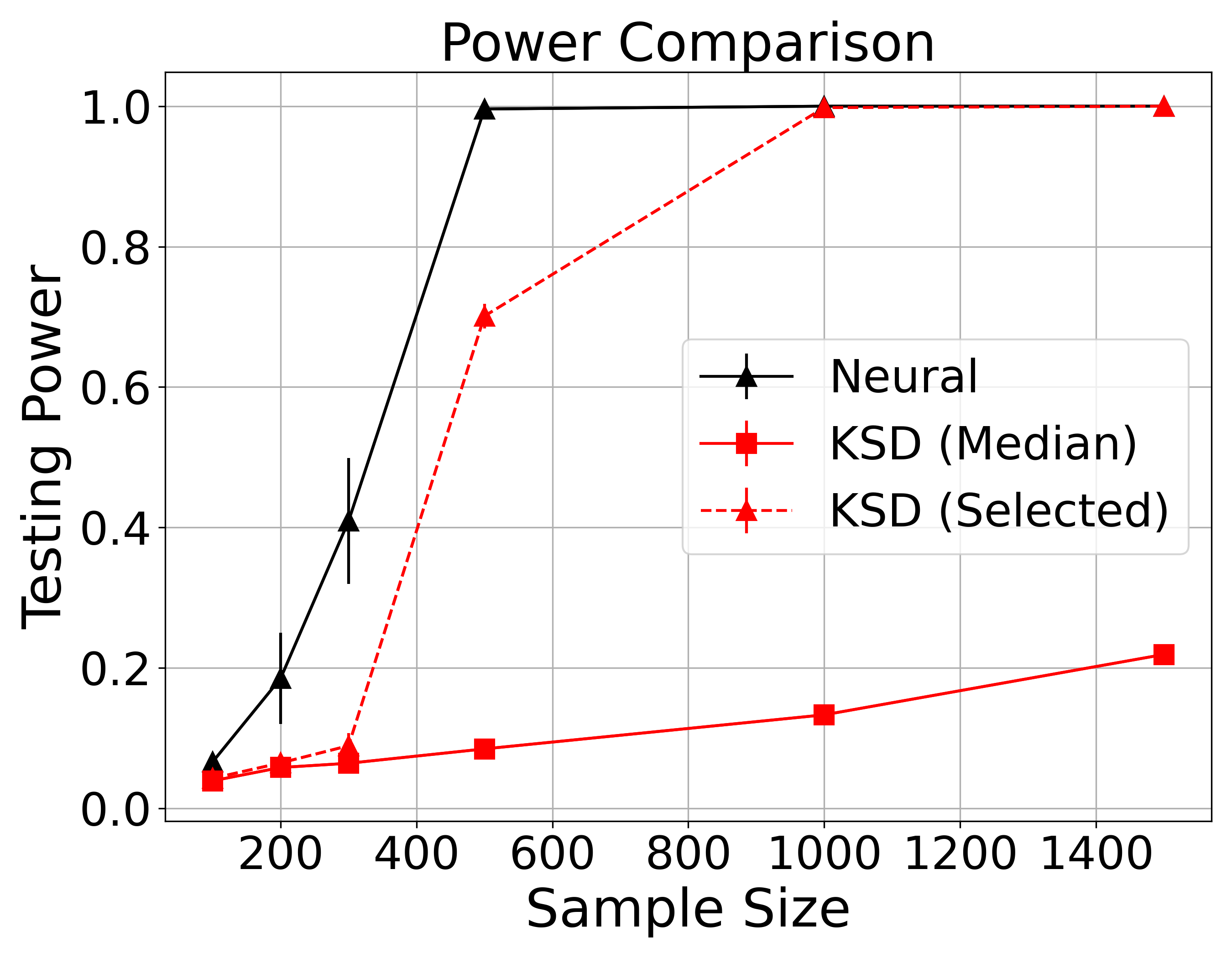}
    \caption{
    Comparison between neural Stein critic GoF test and KSD test in the 50D setting described in Appendix~\ref{app:neural_split_power}. Other than the change in distribution perturbation parameters, the experiment is identical to the experiments described in Section~\ref{sec:kernel_experiments}. The sample size $n_{\rm sample}$ provided to each method is indicated on the x axis. The neural method uses half of these data for training and the other half for testing, and the KSD uses all $n_{\rm sample}$ samples for test statistic computation. The mean and standard deviation of the testing power are plotted in each case, which is computed over 5 replicas.}\label{fig:ksd_comparison_easy}
\end{figure}

\subsection{Comparison over different training-test splits}\label{app:neural_split_power}

In Section~\ref{sec:kernel_experiments}, we outline a procedure for comparing the neural Stein test to the KSD GoF test by considering varying overall sample size $n_{\rm sample}$, training the neural network using half of these samples and conducting tests using the other half. This section analyzes the power and running time of the neural Stein GoF hypothesis testing procedure as outlined in Section~\ref{sec:gof_test} for a variety of training/testing splits. Using the notation of Section~\ref{sec:gm_sim_data} for the model and data distributions, let the model distribution $q$ be an isotropic, two-component Gaussian mixture in $\mathbb{R}^{50}$ with means $\mu_1=\mathbf{0}_{d}$ and $\mu_2=0.5\times\mathbf{1}_{d}$, and let the data distribution $p$ have the form of \eqref{eq:data_distribution_gm} with covariance shift $\rho_1=0.99$ and scale $\omega=0.05$. In this section, we fix $n_{\rm sample}=500$ and consider other choices of training/testing split, comparing to the KSD with ``best'' selected bandwidth (see Appendix~\ref{sec:ksd_gof_test}) for $n_{\rm sample}=500$. That is, we partition the 500 points into splits whereby the number of training samples takes on a proportion in the range of $\left\{10\%,20\%,30\%,40\%,50\%,60\%,70\%,80\%,90\%\right\}$ of $n_{\rm sample}$. The testing procedure is otherwise identical to that of Section~\ref{sec:kernel_experiments}. The result can be seen in Figure~\ref{fig:neural_split_power}. Figure~\ref{fig:neural_split_power}(\textbf{A}) shows that there exists a range of splits for which the power is nearly one and that since the computation time is dominated by the training, as seen in Figures~\ref{fig:neural_split_power}(\textbf{B}) and (\textbf{C}), a 50\%/50\% split is a good generic choice of the split. Therefore, this split is used in the experiments outlined in Section~\ref{sec:kernel_experiments} and displayed in Figure~\ref{fig:kernel_test_comparison}.

We also consider the comparison of testing power in this case for the neural Stein critic GoF test and the KSD test with the median data distances heuristic RBF kernel bandwidth and the KSD with bandwidth selected to maximize power according to the result outlined in Appendix~\ref{sec:ksd_gof_test}. The methodology for comparing the power between these tests is identical to that of Section~\ref{sec:kernel_experiments}. That is, for each sample size, the train/test split is $50\%/50\%$, with $20\%$ of the training partition dedicated to validation. The power is then computed in each case using $n_{\rm run}=400$ GoF tests, $n_{\rm boot}=500$ bootstrapped samples in each individual test, {efficient bootstrap ratio $r_{\rm pool}=50$, and $n_{\rm replica}=5$}. The result can be found in Figure~\ref{fig:ksd_comparison_easy}, which shows that the power of the neural Stein critic GoF test dominates that of both KSD test methods initially. However, the neural Stein test achieves near unit power by the time $n_{\rm sample}=500$, and the KSD with selected bandwidth achieves unit power by the time $n_{\rm sample}=1,000$. This increase in power compared to Figure~\ref{fig:kernel_test_comparison}(\textbf{A}) is due to the fact that the perturbation between the model and data distribution in this scenario is larger in this setting than in the setting outlined in Section~\ref{sec:kernel_experiments}.

\subsection{{Adaptive Staging}}\label{app:adaptive_staging}

\begin{figure}[t]
    \centering
    \includegraphics[width=\textwidth]{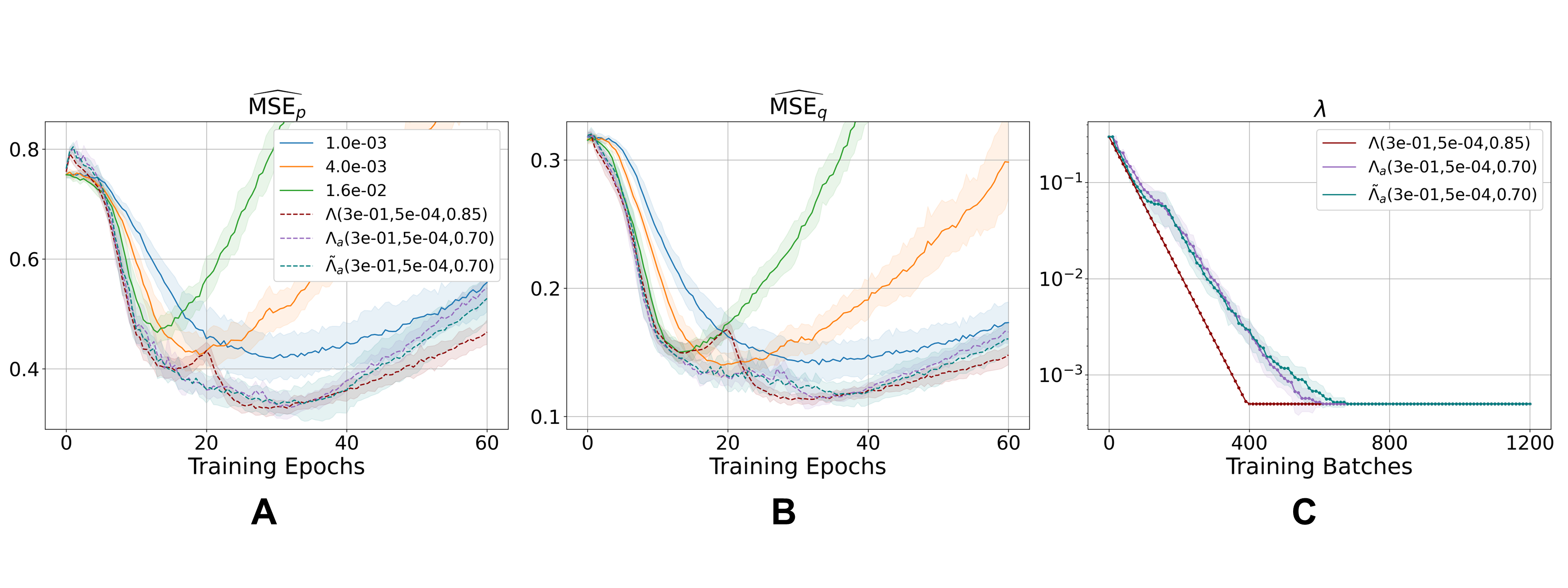}
    \vspace{-30pt}
    \caption{
    The $\widehat{\rm MSE}_p$ (\textbf{A}) and $\widehat{\rm MSE}_q$ (\textbf{B}) of adaptive staged regularization strategies $\Lambda_a$ and $\tilde{\Lambda}_a$ compared to that of nonadaptive staged regularization $\Lambda$ and fixed-$\lambda$ strategies. The trajectories of $\lambda$ over the course of training in the staged regularization strategies is depicted in (\textbf{C}). See Appendix~\ref{app:adaptive_staging} for a description of the experiment setup, which is largely the same as the experiments described in Section~\ref{sec:gm_sim_setup}. The mean MSE over 10 models is represented as a solid line for each regularization strategy, while the standard deviation is represented as the shaded region. Similarly, for the adaptive staging strategies, the mean value of $\lambda$ for each stage is represented by the solid line, while the standard deviation is reflected by the shaded region.
    }\label{fig:adaptive_staging}
\end{figure}

In Section~\ref{subsec:staged}, we specified a method for staging the weight of regularization $\lambda$ via an exponential decay over the course of optimization. 
Here we develop an adaptive procedure which stages $\lambda$ according to the monitor $\widehat{\rm MSE}^{(m)}_p$ \eqref{eq:validation_mse}, and compare with the non-adaptive staging in \eqref{eq:lam_staging}.

The adaptive staging method decides to stage the weight of regularization down to a smaller value when the monitored $\widehat{\rm MSE}^{(m)}_p$ is observed to increase. 
Specifically, we examine two adaptive staging approaches

\begin{itemize}
    \item The strategy, denoted as $\Lambda_a(\lambda_{\rm init},\lambda_{\rm term},\beta)$, begins with $\lambda=\lambda_{\rm init}$. 
Over the course of training, $\widehat{\rm MSE}^{(m)}_p$ is computed after every $B_w=10$ mini-batches of updates given $n_{\rm val}$ samples from $p$ (in these examples, the validation samples are  re-drawn from $p$ for each computation). 
If $\widehat{\rm MSE}^{(m)}_p$ is at any observation greater than the previous observation, the weight of regularization $\lambda$ is multiplied by the factor $\beta<1$. 
This process is repeated until $\lambda=\lambda_{\rm term}$, at which point $\lambda$ is fixed for the remainder of training. 
We require that $\widehat{\rm MSE}^{(m)}_p$ decrease at least once with each stage of $\lambda$, so that $\lambda$ may not stage down until first an improvement followed by a decrease in performance in MSE is observed. 

\item 
The staging, denoted by $\tilde{\Lambda}_a$, uses $\widehat{\rm MSE}_p$ instead of $\widehat{\rm MSE}^{(m)}_p$ as the monitor function to determine the staging. For simulated high dimensional Gaussian mixture data, the computation of $\widehat{\rm MSE}_p$ uses the prior information of the scaleless optimal critic ${\bf f}^*$.
\end{itemize}

To assess the performance of the adaptive staging strategy, we conduct an experiment using the Gaussian mixture data of Section~\ref{sec:gm_sim_data} in 25D with $\omega=0.8$ and $\rho_1=0.5$. The neural Stein critic architecture and initialization are as in Section~\ref{sec:gm_sim_setup}. We again use the Adam optimizer with default parameters, with learning rate $10^{-3}$ and batch size 200 samples. In this setting, the networks are trained using 4,000 samples from $p$ for 60 epochs each. For each staging strategy, 10 critics are trained. The $\widehat{\rm MSE}^{(m)}_p$ and $\widehat{\rm MSE}_p$ are computed each using $n_{\rm val}=20,000$ samples from $p$ and the $\widehat{\rm MSE}_q$ is computed using $n_{\rm te}=20,000$ samples from $q$. We compare the staging to fixed regularization strategies with $\lambda\in\{1\times 10^{-3}, 4\times 10^{-3}, 1.6\times 10^{-2}\}$. The staging strategies used are $\Lambda(3\times10^{-1}, 5\times10^{-4}, 0.85)$, $\Lambda_a(3\times10^{-1}, 5\times10^{-4}, 0.70)$, and $\tilde{\Lambda}_a(3\times10^{-1}, 5\times10^{-4}, 0.70)$, where an update is made according to the nonadaptive staging strategy every $B_w=10$ batches.

The $\widehat{\rm MSE}_p$ and $\widehat{\rm MSE}_q$ in the adaptive staging comparison experiments are shown in Figure~\ref{fig:adaptive_staging}(\textbf{A}) and (\textbf{B}), respectively. As in Section~\ref{sec:gm_sim_result}, all staging strategies outperform the fixed-$\lambda$ strategies in 25D. 
Furthermore, the performances of the adaptive strategies $\Lambda_a$ and $\tilde{\Lambda}_a$ and the non-adaptive strategy $\Lambda$ are similar, with overall minimal MSE comparable across all three staging strategies. 
In addition, a comparison of the trajectory of $\lambda$ over the course of training via the three staged regularization strategies is displayed in Figure~\ref{fig:adaptive_staging}(\textbf{C}), indicating that the adaptive staging strategies exhibit a roughly exponential annealing of $\lambda$ over the course of training. 
This result justifies the rationale of the heuristic, nonadaptive staging approach proposed in Section~\ref{subsec:staged}, which is computationally less expensive than the adaptive staging which relies on the repetitive computation of the monitor validation MSE.

\end{document}